%% file: main.tex
\documentclass{article} 
\usepackage{iclr2027_conference,times}

\input{math_commands.tex}

\usepackage{hyperref}
\usepackage{url}
\usepackage[utf8]{inputenc}
\usepackage[T1]{fontenc}
\usepackage{xcolor}
\usepackage{booktabs}
\usepackage{amsfonts,amssymb}
\usepackage{microtype}
\usepackage{graphicx}
\usepackage{subcaption}
\usepackage{float}
\usepackage{amsthm,mathtools,bm}
\usepackage{array,multirow}
\usepackage{colortbl}
\usepackage{enumitem}
\usepackage{tcolorbox}
\usepackage{makecell}
\usepackage{placeins}
\usepackage{titletoc}
\definecolor{LinkBlue}{RGB}{20,83,145}
\hypersetup{colorlinks=true,allcolors=LinkBlue}

\theoremstyle{plain}
\newtheorem{theorem}{Theorem}[section]
\newtheorem{proposition}{Proposition}[section]
\newtheorem{lemma}{Lemma}[section]

\theoremstyle{definition}
\newtheorem{assumption}{Assumption}[section]
\newtheorem{definition}{Definition}[section]
\theoremstyle{remark}

\providecommand{\E}{\mathbb{E}}
\newcommand{\tr}{\operatorname{Tr}}
\providecommand{\Cov}{\operatorname{Cov}}
\newcommand{\LNWA}{\textnormal{\textsc{LNWA}}}
\newcommand{\LAWA}{\textnormal{\textsc{LAWA}}}
\newcommand{\WMA}{\textnormal{\textsc{WMA}}}
\newcommand{\EMA}{\textnormal{\textsc{EMA}}}

\title{Understanding the Weight Averaging Mechanism in LLM Training for Post-Training Quantization}

\author{%
  {\bfseries
    Hanzhang Wang\textsuperscript{1}\quad
    Tianqi Shen\textsuperscript{1}\quad
    Zonglin Liu\textsuperscript{1}\quad
    Junze He\textsuperscript{1}\quad
    Difan Zou\textsuperscript{2}\quad
    Ziye Ma\textsuperscript{1}%
    \thanks{\normalfont Corresponding author.
    Difan Zou and Ziye Ma supervised this work.}
  }\\[5pt]
  {\normalfont
    \textsuperscript{1}City University of Hong Kong \quad \textsuperscript{2}The University of Hong Kong
    }\\[4pt]
  {\normalfont\footnotesize
    \texttt{hanzwang6-c@my.cityu.edu.hk}\quad
    \texttt{dzou@cs.hku.hk}\quad
    \texttt{ziyema@cityu.edu.hk}
  }
}

\iclrfinalcopy 
\begin{document}

\maketitle
\lhead{Preprint}

\begin{abstract}
Large language models (LLMs) are typically pretrained in high precision but increasingly deployed with low-precision post-training quantization (PTQ). Recent studies have shown that using weight averaging during pretraining can improve PTQ performance compared with learning-rate decay, suggesting that it might provide a simple way to improve the pretraining-to-quantization transition. But the mechanism behind weight averaging remains insufficiently explained. This leads to inconsistent and fragile performance gains, thereby preventing practitioners from applying such a technique confidently. As a response, we formulate weight averaging as a trade-off between retaining training progress and improving robustness under perturbation. We further derive a continuous family of averaging kernels that unifies conventional strategies and achieves the Pareto frontier between the two competing goals. Critically, a theoretical framework for performing weight averaging under PTQ is developed. It can be shown that coarser quantization is more susceptible to perturbations, whereas finer quantization could be less affected. Thus, our results could provide unified theoretical guidance for performing weight averaging under different PTQ conditions. Experiments validate both the predicted behavior and the proposed averaging strategy. Code is available at
\href{https://github.com/MOFA-LAB/weight-averaging-for-ptq}{https://github.com/MOFA-LAB/weight-averaging-for-ptq}.
\end{abstract}

\section{Introduction}
\label{sec:introduction}
Large language models (LLMs) are typically pretrained at relatively high
numerical precision \citep{workshop2022bloom, team2025gemma, team2025kimi}. As model scale continues to grow, post-training
quantization (PTQ) has become increasingly important for reducing
deployment costs by converting pretrained models to lower precision \citep{xiao2023smoothquant, shao2026dartquant,li2025gptaq}.
This mismatch between training and deployment raises a key challenge at the training-to-quantization
transition: better performance before quantization does not necessarily
translate into better performance afterward \citep{ouyang2025low, kumar2025scaling, chen2026coverage, catalan2026training}.

Prior work has shown that training dynamics affect both training performance
and sensitivity to quantization \citep{ahmadian2023intriguing, wang2025learning, catalan2026training, luo2026learning}. Late-stage learning-rate
(LR) decay can further reduce training loss, but may move the model
into sharper regions that are more sensitive \citep{watts2026sharpness, microsoft-ai-team-mai-thinking-1-building-a-hill-climbing-machine-2024, yano2026pre}. By contrast, under a stable LR, after completing $T$ training steps,
one can retain the last $M$ weights,
$\bm\theta_T,\ldots,\bm\theta_{T-M+1}$, where $M$ denotes the number
of retained weights. We index them by their age
$d\in\{0,\ldots,M-1\}$, with $d=0$ corresponding to the most recent
weight, and average them as
\begin{equation}
    \bm\theta_T^{\bm w}
    =
    \sum_{d=0}^{M-1}w_d\bm\theta_{T-d},
    \qquad
    w_d\geq0,\quad
    \sum_{d=0}^{M-1}w_d=1,
    \label{eq:intro1}
\end{equation}
where $\bm w=(w_0,\ldots,w_{M-1})$ is the averaging kernel, with different
choices corresponding to different averaging strategies.
Recent work shows that such weight averaging can achieve high-precision
performance comparable to LR decay while improving PTQ performance
\citep{sanyal2023early, sandler2023training, tian2026wsm}. However, existing evidence for the benefits of weight averaging
at the pretraining-to-quantization transition remains largely empirical
\citep{ajroldi2025and,catalan2026training,li2026model},
and a theoretical understanding of how the averaging strategy
affects PTQ is still lacking. This gap leaves unclear why weight averaging benefits PTQ
and when a particular averaging kernel is effective. As a consequence, blindly applying any weight averaging strategy now cannot guarantee improved downstream PTQ performance; and in some cases, it might be even worse. This is exemplified by our discovery that the same
averaging strategy could outperform another at higher precision
but instead fall behind at lower precision. Specifically, Fig.~\ref{fig:motivation}(a) compares linearly weighted \WMA{} with uniformly weighted \LAWA{} on OPT-125M, showing that their performance ranking reverses between 4-bit and 2-bit PTQ. In short, favorable empirical findings may not generalize across settings, and extensive trial-and-error is usually required. This strongly motivates a unified theoretical framework to understand
how weight averaging in pretraining affects later quantization, thereby providing
principled guidance for averaging strategy selection across settings.

\begin{figure}[t]
  \centering
  \includegraphics[width=\linewidth]{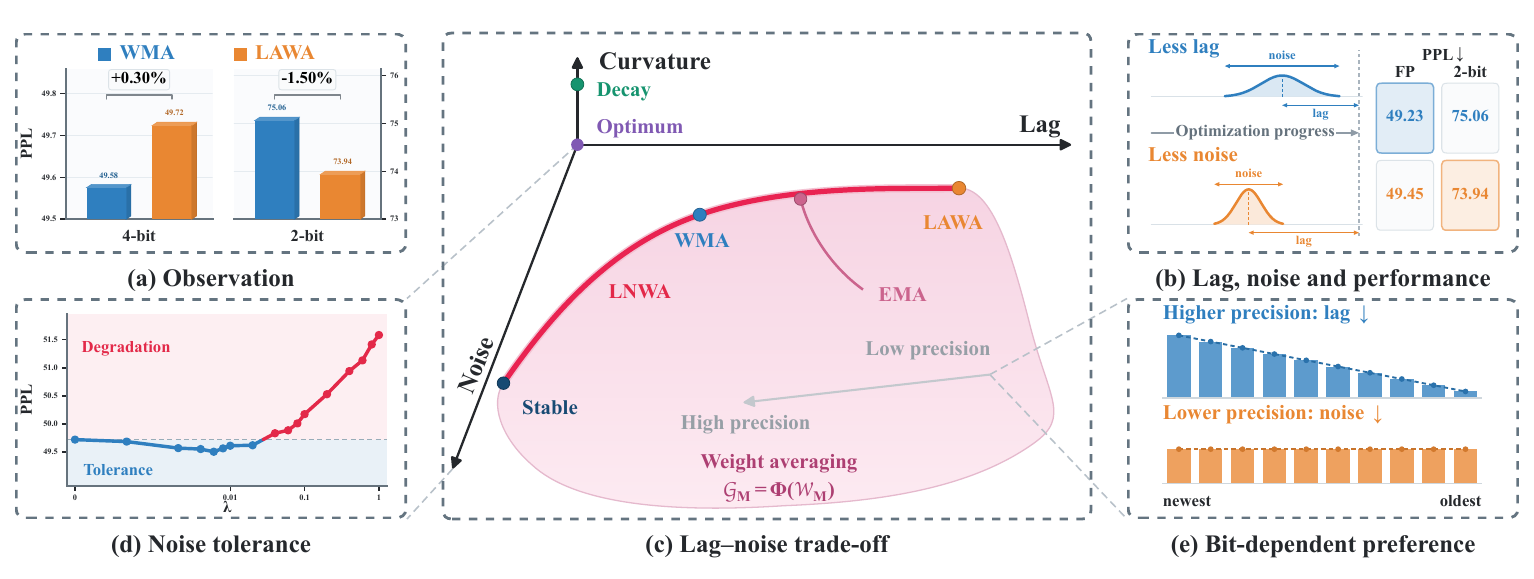}
  \caption{Overview of our work.
  (a) Different averaging strategies' rankings reverse between 4-bit and 2-bit PTQ.
  (b) Less lag favors better current performance (FP), while less noise favors greater robustness to quantization (2-bit).
  (c) We formulate weight averaging as a lag--noise trade-off and derive its
  Pareto frontier, \LNWA{}. 
  (d) Reduced lag can offset additional noise within a finite tolerance interval; beyond it, further noise leads to degradation.
  (e) The preferred balance reverses with precision, with lower precision
  favoring less noise and higher precision favoring less lag.}
  \label{fig:motivation}
\end{figure}

Building upon prior art, we believe it is useful to evaluate a training trajectory through its genuine progress towards improved model abilities and its robustness to different kinds of noise and perturbations \citep{morales2024exponential, block2025ema}.
Specifically,
we can decompose each retained weight as $\bm\theta_{t}=\bm\mu_t+\bm\epsilon_t$ at training step $t$, where the latent mean state $\bm\mu_t$ denotes the expectation of
$\bm\theta_t$ across repeated training runs over training randomness,
with $\bm\mu_T$ serving as an idealized noise-free final model, and
$\bm\epsilon_t$ denotes the corresponding zero-mean fluctuation arising
from mini-batch sampling, data ordering, and other sources of training
randomness.
Applying weight averaging in Eq.~(\ref{eq:intro1}) to the decomposition yields
\begin{equation}
\begin{aligned}
  \bm\theta_T^{\bm w}-\bm\mu_T
  &= -\underbrace{\left(\bm\mu_T-
      \sum_{d=0}^{M-1}w_d\bm\mu_{T-d}\right)}_{\text{temporal lag}}
     +\underbrace{\sum_{d=0}^{M-1}w_d\bm\epsilon_{T-d}}_{\text{spatial noise}}.
\end{aligned}
\label{eq:intro2}
\end{equation} 
Thus, the deviation of the averaged model from the idealized final model \(\bm\mu_T\) decomposes into temporal lag and spatial noise. Temporal lag captures the genuine progress sacrificed by weight averaging, which can impair the model's performance at the transition to PTQ. At the same time, weight averaging suppresses spatial noise by reducing the variance of stochastic fluctuations, thereby improving robustness to PTQ, as illustrated
in Fig.~\ref{fig:motivation}(b). The averaging
kernel thus determines the balance between these two effects.
We then derive
Lag--Noise Trade-off Weight Averaging (\LNWA{}), a closed-form kernel
family that parameterizes the Pareto frontier of the
lag--noise trade-off and includes classical averaging kernels as special cases, as shown in Fig.~\ref{fig:motivation}(c).

Building on this formulation, we develop a theoretical framework for weight averaging strategy selection under PTQ. We first study how moving along the lag--noise frontier affects PTQ performance at a fixed precision, quantifying how strongly the temporal lag could be reduced before PTQ performance degrades (Fig.~\ref{fig:motivation}(d)). Within this tolerance region, we then vary the quantization precision and show that lower precision favors less spatial noise, whereas higher precision favors less temporal lag (Fig.~\ref{fig:motivation}(e)); we further characterize the quantization precision threshold at which this preference reverses. Together, we extend conventional method-specific weight averaging selection into a unified theoretical framework under different PTQ conditions. Finally, experiments on LLMs from 125M to 1.3B parameters validate our theory.
Our contributions can be summarized as follows:
\begin{enumerate}

\item We establish a theoretical framework for understanding weight averaging through a two-way decomposition in Eq.~(\ref{eq:intro2}). Specifically, we separate its effects into temporal lag, which captures inherent training progress, and spatial noise, which captures the solution's robustness to perturbations. We show that these two objectives are in competitive tension, and a careful trade-off is necessary. 

\item A new class of weight averaging strategies called \LNWA{} is proposed based on the aforementioned framework. It operates on the Pareto frontier between temporal lag and spatial noise (Theorem~\ref{thm:main-frontier}) and smoothly interpolates between them. This class generalizes existing strategies (see Table~\ref{tab:app-classical-kernels}) and can achieve any desired Pareto-optimal trade-off between the two objectives.

\item We further provide a theoretical analysis of how the lag--noise balance affects PTQ performance (Sections~\ref{sec:noise-tolerance} and~\ref{sec:bit-preference}). We show that this balance depends critically on quantization precision (Theorems~\ref{thm:ptq-turnover} and~\ref{thm:bit-ordering}) and therefore should not be selected universally. These results provide principled guidance for performing weight averaging across PTQ settings.

\end{enumerate}

\section{Related Work}
\label{sec:related-work}

\textbf{Learning-rate decay.}
Late-stage LR decay is widely used in LLM pretraining to further improve final performance
\citep{hu2024minicpm,hagele2024scaling}.
Prior work has mainly focused on understanding why LR decay benefits pretraining \citep{wen2025understanding,li2026functional,li2026optimal,liu2026effective}. However, recent studies show that LR decay can also affect how well the resulting model transfers to subsequent stages
\citep{chen2026coverage}.
Similar effects have been observed across various downstream tasks
\citep{wang2025learning,yano2026pre,microsoft-ai-team-mai-thinking-1-building-a-hill-climbing-machine-2024,catalan2026training},
suggesting that lower final pretraining loss does not necessarily imply a better transition to downstream adaptation.

\textbf{Post-training quantization. }
PTQ converts a model to lower precision, reducing
storage and inference costs \citep{chee2023quip}. Various PTQ methods have been proposed \citep{nagel2020up,lin2024awq, ashkboos2024quarot,dettmers2022gpt3,egiazarian2024extreme}, and error compensation is one of the most well-established approaches \citep{yao2024exploring, zhang2025qera}, with GPTQ as a widely used representative \citep{frantar2022gptq}. GPTQ reduces layerwise reconstruction error through second-order updates to the remaining unquantized weights \citep{frantar2022gptq}. Recent work has further improved PTQ through a range of transformation, clipping, and reconstruction techniques \citep{liu2025spinquant,sun2024flatquant,shao2024omniquant,tseng2024quip,zhang2024lqer}.

\textbf{Weight averaging. }
Classical Polyak--Ruppert averaging established the variance-reduction benefit of averaging iterates \citep{polyak1992acceleration}. Motivated by it, stochastic weight averaging (SWA) showed that averaging weights along the training trajectory can lead to wider optima and better generalization \citep{izmailov2018averaging}. Moreover, \LAWA{} \citep{sanyal2023early} uniformly averages recent weights sparsely, while \EMA{} \citep{morales2024exponential} and \WMA{} \citep{li2026model} assign larger coefficients to more recent weights through exponential and linear weighting, respectively. Recent work has applied \LAWA{} to the pretraining-to-quantization transition, showing that the averaged models can match or outperform LR decay after PTQ \citep{catalan2026training}.

\section{Preliminaries}
\label{sec:preliminaries}

\paragraph{Notation.}
Scalars are denoted by italic letters (e.g., $f$), vectors by bold lowercase
letters (e.g., $\bm{f}$), matrices by bold uppercase letters
(e.g., $\bm{F}$), and sets by calligraphic uppercase letters
(e.g., $\mathcal F$).

Consider a fixed window of retained model weights
$\bm\theta_T,\ldots,\bm\theta_{T-M+1}\in\mathbb R^P$ ($M\geq2$), where $T$ is the training steps, $P$ is the
number of parameters and age $d$ counts backward from the newest retained weight at $d=0$. Let
$\bm w=(w_0,\ldots,w_{M-1})$ denote the averaging kernel, where $w_d$ is
the averaging coefficient assigned to the weight
$\bm\theta_{T-d}$. Its feasible set $\mathcal W_M$ and the resulting
averaged model $\bm\theta_T^{\bm w}$ are defined as
\begin{equation}
  \mathcal W_M=\left\{\bm w\in\mathbb R^M:
    w_d\geq0,\ \sum_{d=0}^{M-1}w_d=1\right\},
  \qquad
  \bm\theta_T^{\bm w}=\sum_{d=0}^{M-1}w_d\bm\theta_{T-d}.
  \label{eq:main-averaging}
\end{equation}
Thus $\bm\theta_T^{\bm w}$ lies in the convex hull of the retained weights. 
We first define the latent mean $\bm\mu_t$, which provides an idealized noise-free reference to separate training progress from spatial noise.

\begin{definition}[Latent-mean Trajectory]
\label{def:main-latent-mean}
Let $\mathcal F_0$ be the sigma-field generated by the information shared
across stochastic training runs, such as the model configuration, common
initialization, and training setup. Let $\xi$ index the run-specific
stochasticity, such as data ordering and dropout. For the aligned state
$\bm{\theta}_{T-d}(\xi)$, assume $
\E[
\|\bm{\theta}_{T-d}(\xi)\|_2^2
\mid\mathcal F_0
]
<\infty
$. The conditional mean state $\bm\mu_t$ and the zero-mean spatial noise
$\bm\epsilon_t$ at the training step $t$ of an individual run are
\begin{equation}
  \bm\mu_t:=\E_\xi[\bm\theta_{t}(\xi)\mid\mathcal F_0],
  \qquad
  \bm\epsilon_t(\xi):=\bm\theta_{t}(\xi)-\bm\mu_t.
 \label{eq:main-latent-mean}
\end{equation}
For an averaging kernel $\bm w$, write
$\bm\mu_T^{\bm w}:=\sum_{d=0}^{M-1}w_d\bm\mu_{T-d}$ for the mean of
$\bm\theta_T^{\bm w}$, and write $\bm\epsilon_T^{\bm w}:=\sum_d w_d\bm\epsilon_{T-d}$ for the
averaged spatial noise. 
\end{definition}

Besides, we model how PTQ reduces the reconstruction error. Let $\bm\theta^{(\ell,i)}\in\mathbb R^{d_{\mathrm{in}}}$ denote the parameter of output channel $i$ in layer $\ell$, extracted from $\bm\theta$. Quantization changes $\bm\theta^{(\ell,i)}$ by
$\Delta\bm\theta^{(\ell,i)}$, so the output error is $\bigl(\Delta\bm\theta^{(\ell,i)}\bigr)^\top\bm X$ for input activations $\bm X\in\mathbb R^{d_{\mathrm{in}}\times n}$. Its mean squared reconstruction error is
\begin{equation}
 \frac1n\|\bigl(\Delta\bm\theta^{(\ell,i)}\bigr)^\top\bm X\|_2^2
 =\bigl(\Delta\bm\theta^{(\ell,i)}\bigr)^\top\bm G^{(\ell,i)}\Delta\bm\theta^{(\ell,i)}
 =\|\Delta\bm\theta^{(\ell,i)}\|_{\bm G^{(\ell,i)}}^2,
 \label{eq:main-layer-reconstruction}
\end{equation}
where $\bm G^{(\ell,i)}:=\bm X\bm X^\top/n$. GPTQ \citep{frantar2022gptq} reduces errors by adjusting parameters that
remain unquantized. We represent this procedure with two parameter groups:
quantize the first group, use the second group to reduce the resulting
output error, and then quantize the second group. For a parameter vector $\bm z$, $\operatorname{quant}_b(\bm z)$ denotes
scale-and-zero-point quantization at precision $b$ using
$\operatorname{round}$. 

\begin{definition}[Simplified GPTQ Model]
\label{def:main-calibration}
Partition $\bm\theta^{(\ell,i)}$ into
$\bm\theta_{\mathrm{f}}^{(\ell,i)}\in\mathbb R^{d_{\mathrm{f}}}$ and
$\bm\theta_{\mathrm{r}}^{(\ell,i)}\in\mathbb R^{d_{\mathrm{r}}}$,
where $d_{\mathrm{f}}+d_{\mathrm{r}}=d_{\mathrm{in}}$ and
$d_{\mathrm{f}},d_{\mathrm{r}}\geq1$.
Let $\bm X_{\mathrm{f}},\bm X_{\mathrm{r}}$ be the corresponding row
blocks of $\bm X$, and partition $\bm G^{(\ell,i)}$ accordingly.
First quantize
$\bm q_{\mathrm{f}}^{(\ell,i)}
=\operatorname{quant}_b(\bm\theta_{\mathrm{f}}^{(\ell,i)})$,
with quantization error $\bm e_{\mathrm{f}}^{(\ell,i)}
:=\bm q_{\mathrm{f}}^{(\ell,i)}-\bm\theta_{\mathrm{f}}^{(\ell,i)}$.
Then use the remaining group to minimize the reconstruction error caused by $\bm e_{\mathrm{f}}^{(\ell, i)}$:
\begin{equation}
 \bm c_{\mathrm{r}}^{(\ell,i),\star}
 =\operatorname*{argmin}_{\bm c_{\mathrm r}\in\mathbb R^{d_{\mathrm{r}}}}
  \frac1n\left\|
   \bigl(\bm e_{\mathrm{f}}^{(\ell,i)}\bigr)^\top\bm X_{\mathrm{f}}
       +\bm c_{\mathrm r}^\top\bm X_{\mathrm{r}}\right\|_2^2.
 \label{eq:main-calibration-objective}
\end{equation}
Then quantize the corrected remaining group:
\begin{equation*}
 \bm q_{\mathrm{r}}^{(\ell,i)}
 =\operatorname{quant}_b\!\left(\bm\theta_{\mathrm{r}}^{(\ell,i)}
                         +\bm c_{\mathrm{r}}^{(\ell,i),\star}\right).
\end{equation*}
The quantized $\bm\theta^{(\ell,i)}$ is
$\bm q^{(\ell,i)}=
 ((\bm q_{\mathrm{f}}^{(\ell,i)})^\top,
  (\bm q_{\mathrm{r}}^{(\ell,i)})^\top)^\top$.
At precision $b$, applying these steps to every output channel of every linear
layer gives the full-model simplified GPTQ $Q_b(\bm\theta)$ (see Appendix~\ref{app:ptq}). 
\end{definition}
\begin{figure}[t]
  \centering
  \includegraphics[width=\textwidth]{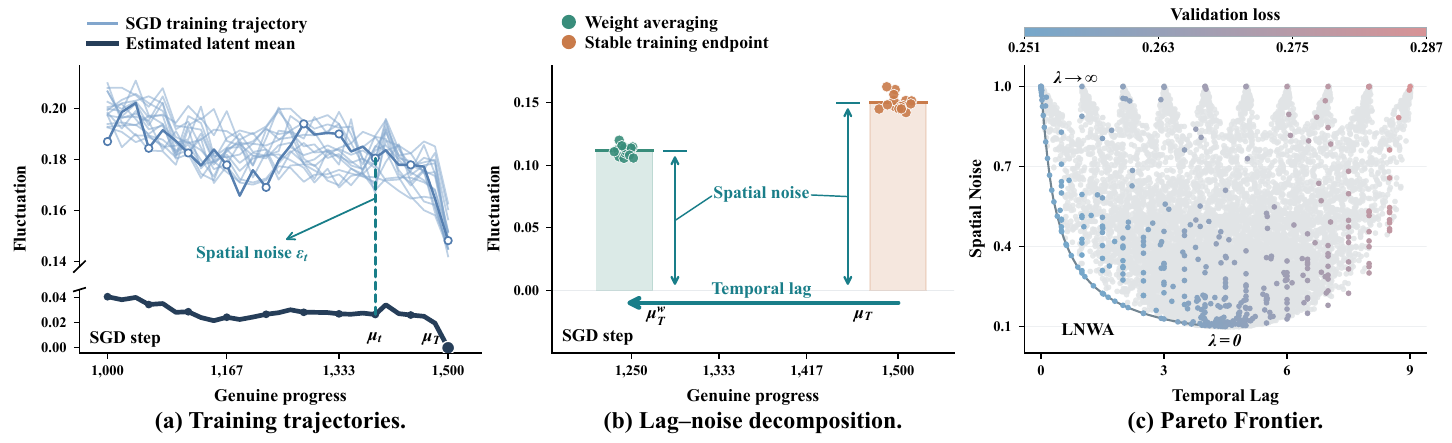}
  \caption{Illustrative visualizations with a two-layer MLP on MNIST.
  (a) SGD training trajectories with different data orders and their estimated latent mean.
  (b) Weight averaging reduces spatial noise but introduces temporal lag
  compared with the stable training endpoint.
  (c) \LNWA{} characterizes the Pareto frontier between temporal lag and
  spatial noise. Colors indicate validation loss.}
  \label{fig:lag-noise-mechanism}
\end{figure}

\section{Lag--Noise Trade-off}
\label{sec:method}

In this section, we formulate weight averaging as a trade-off between temporal lag and spatial noise and derive its Pareto frontier (Section~\ref{sec:lag-noise-geometry}). We then connect lag and noise to PTQ and further characterize the tolerance-to-degradation transition at a fixed precision (Section~\ref{sec:noise-tolerance}) and the precision-dependent shift in the preferred lag--noise balance (Section~\ref{sec:bit-preference}).

\subsection{Lag--noise decomposition and the \LNWA{} frontier}
\label{sec:lag-noise-geometry}

Using the latent-mean trajectory in
Definition~\ref{def:main-latent-mean}, a retained weight can be written as
\(
\bm\theta_{T-d}=\bm\mu_{T-d}+\bm\epsilon_{T-d}
\). As illustrated in Fig.~\ref{fig:lag-noise-mechanism}(a), across different data orders, the individual training trajectories fluctuate around a shared latent-mean trajectory. The latent mean captures the progress that is shared across runs, while the difference between an individual trajectory and the latent mean captures the fluctuation caused by run-specific randomness.
Weight averaging therefore acts on both components. Let
\(
\bm v_t:=\bm\mu_t-\bm\mu_{t-1},
\quad T-M+2\leq t\leq T,
\)
denote the latent progress between adjacent retained weights.
To quantify the effects of temporal lag  and  spatial noise in Eq.~(\ref{eq:intro2}), define
\[
    m(\bm w):=\sum_{d=0}^{M-1} d w_d,
    \qquad
    R(\bm w):=\sum_{d=0}^{M-1} w_d^2.
\]
Here, $m(\bm w)$ is the mean age of the averaged weights, which measures temporal lag; $R(\bm w)$ is the squared $\ell_2$ norm of the averaging kernel, which determines the amount of spatial-noise variance retained after averaging.
Lemma~\ref{lem:app-scalar} then gives
\begin{equation}
    \bm\theta_T^{\bm w}-\bm\mu_T
    =
    -m(\bm w)\bm v_T
    +\bm\epsilon_T^{\bm w}
    +\bm e_{\bm w}^{\mathrm{traj}},
    \label{eq:main-decomposition}
\end{equation}
where $\bm e_{\bm w}^{\mathrm{traj}}$ is a bounded trajectory remainder, and
\(
    \Cov(\bm\epsilon_T^{\bm w}\mid\mathcal F_0)
    =R(\bm w)\bm\Sigma.
\)
Fig.~\ref{fig:lag-noise-mechanism}(b) visualizes how averaging acts on
spatial noise and temporal lag. Weight averaging reduces the fluctuation across
training runs (spatial noise), while shifting the averaged latent mean
$\bm\mu_T^{\bm w}$ away from the latest latent mean $\bm\mu_T$ along the
direction of genuine progress (temporal lag). These two effects therefore
form a fundamental trade-off in weight averaging. We formalize this trade-off
by minimizing retained spatial noise under a fixed temporal-lag budget
(Lemma~\ref{lem:app-budget}). As the budget varies, the resulting solutions
trace the lag--noise Pareto frontier. Thus, we can parameterize it with a one-parameter
family of averaging kernels termed Lag--Noise Trade-off Weight Averaging
(\LNWA{}).

\begin{theorem}[\LNWA{} Kernel]
\label{thm:main-frontier}
For $0\leq\lambda<\infty$, let
\begin{equation}
    \bm w_\lambda
    :=
    \arg\min_{\bm w\in\mathcal W_M}
    \left\{
        R(\bm w)+\lambda m(\bm w)^2
    \right\}.
    \label{eq:main-lnwa-objective}
\end{equation}
For each finite $\lambda$, the unique minimizer has an active prefix
$S_q=\{0,\ldots,q-1\}$ with $2\leq q\leq M$, and its coefficients are
\begin{equation}
  (\bm w_\lambda)_d=
  \begin{cases}
    q^{-1}+\lambda m_q(\lambda)(\bar d_q-d),&0\leq d<q,\\
    0,&q\leq d<M,
  \end{cases}
  \label{eq:main-lnwa-kernel}
\end{equation}
where $m_q(\lambda)=\bar d_q/(1+\lambda D_q)$,
$\bar d_q=(q-1)/2$, and $D_q=q(q^2-1)/12$.
Define $\bm w_\infty:=\lim_{\lambda\to\infty}\bm w_\lambda
=\bm e_0=(1,0,\ldots,0)^\top$, the Stable endpoint.
The extended family $\{\bm w_\lambda\}_{0\leq\lambda\leq\infty}$
parameterizes the complete Pareto frontier between temporal lag and
spatial noise.
\end{theorem}
\begin{table}[t]
  \centering
  \caption{Relation between classical averaging kernels and \LNWA{}.
  \LAWA{} and Stable are the two endpoints, and \WMA{} is an exact interior
  point. \EMA{} is exact when $M=2$. For $M>2$, $\lambda_{\EMA}$ gives a local approximation with
  error $\mathcal O(\eta^2M)$, where $\eta=-\log\beta$.}
  \label{tab:app-classical-kernels}
  \scriptsize
  \setlength{\tabcolsep}{3.5pt}
  \renewcommand{\arraystretch}{1.22}
  \resizebox{\textwidth}{!}{%
  \begin{tabular}{lcccc}
    \toprule
    Methods & Kernel $w_d$ & Lag $m(w)$ & Noise $R(w)$ & Exact match \\
    \midrule
    \LAWA{}
      & $\displaystyle 1/M$
      & $\displaystyle (M-1)/2$
      & $\displaystyle 1/M$
      & $\checkmark$ \\[2mm]
    \WMA{}
      & $\displaystyle 2(M-d)/[M(M+1)]$
      & $\displaystyle (M-1)/3$
      & $\displaystyle 2(2M+1)/[3M(M+1)]$
      & $\checkmark$ \\[2mm]
    \EMA{}
      & $\displaystyle (1-\beta)\beta^d/(1-\beta^M)$
      & $\displaystyle
        \frac{\beta-M\beta^M+(M-1)\beta^{M+1}}
             {(1-\beta)(1-\beta^M)}$
      & $\displaystyle
        \frac{(1-\beta)(1+\beta^M)}
             {(1+\beta)(1-\beta^M)}$
      & \makecell{$M=2$: $\checkmark$\\$M>2$: $\mathcal O(\eta^2M)$} \\[2mm]
    Stable
      & $\displaystyle \bm e_0=(1,0,\ldots,0)$
      & $0$
      & $1$
      & $\checkmark$ \\
    \bottomrule
  \end{tabular}%
  }
\end{table}

In particular, for $0\leq\lambda_1<\lambda_2\leq\infty$,
\begin{equation}
    m(\bm w_{\lambda_1})
    >
    m(\bm w_{\lambda_2}),
    \qquad
    R(\bm w_{\lambda_1})
    <
    R(\bm w_{\lambda_2}).
    \label{eq:main-lnwa-ordering}
\end{equation}
Thus, increasing $\lambda$ moves the averaging kernel along the frontier
toward larger spatial noise, while decreasing
$\lambda$ moves it toward larger temporal lag. Lemma~\ref{lem:app-endpoints-support} characterizes how the number of
retained weights assigned nonzero averaging coefficients changes with
$\lambda$, while Lemma~\ref{lem:app-frontier-coordinates} gives explicit
formulas for the corresponding temporal lag and spatial noise. Fig.~\ref{fig:lag-noise-mechanism}(c) visualizes the feasible lag--noise
region induced by different averaging kernels, where \LNWA{} traces its
Pareto boundary. The distribution of validation loss across this region
shows that the optimal averaging strategy depends on
how these two factors are balanced. Several classical averaging kernels can also be viewed as special cases of
the \LNWA{} family: \LAWA{} and Stable form its two endpoints
(Proposition~\ref{prop:app-endpoints}); \WMA{} is an exact interior kernel
(Proposition~\ref{prop:app-wma}); and \EMA{} admits an exact representation
for $M=2$ and a local approximation for $M>2$
(Proposition~\ref{prop:app-ema}). Table~\ref{tab:app-classical-kernels} summarizes the lag--noise coordinates of
these classical kernels.


\subsection{Noise tolerance and quantitative degradation}
\label{sec:noise-tolerance}

As discussed above, increasing $\lambda$ along the \LNWA{} frontier reduces temporal lag while increasing spatial noise. In this section, we quantify how much lag reduction can compensate for the accompanying increase in noise and thereby improve PTQ performance at a fixed precision.

We measure the effect of an averaging kernel by the expected output reconstruction error (defined in Eq.~(\ref{eq:main-layer-reconstruction})) of the quantized model (defined in Definition~\ref{def:main-calibration}) relative to the latest noise-free state $\bm\mu_T$:
\begin{equation}
\mathcal E_b(\bm w)
:=
\frac12\E\!\left[
\bigl\|Q_b(\bm\theta_T^{\bm w})-\bm\mu_T\bigr\|_{\bm G}^2
\right]
=\mathcal P_0(\bm w)+\gamma_b\mathcal Q(\bm w).
\label{eq:main-loss-definition}
\end{equation}
Here $\mathcal P_0$ is the error before quantization,
$\gamma_b\mathcal Q$ is the reconstruction error introduced by quantization,
and $\gamma_b$ is a precision-dependent coefficient. Lemma~\ref{lem:app-quantization-expansion} proves this decomposition. Next, we need to show that both $\mathcal P_0$ and $\mathcal Q$ depend on temporal lag and spatial noise. Before quantization, these two effects are readily separated:
\begin{equation}
\mathcal P_0(\bm w)
:=
\frac12\E\!\left[
\|\bm\theta_T^{\bm w}-\bm\mu_T\|_{\bm G}^2
\right]
=
\frac12\|\bm\mu_T^{\bm w}-\bm\mu_T\|_{\bm G}^2
+
N_G R(\bm w),
\qquad
N_G:=\frac{\sigma^2}{2}\tr(\bm G).
\label{eq:main-unquantized-costs}
\end{equation}
The first term measures the error associated with temporal lag, while the second is proportional to the spatial-noise variance retained by the averaging kernel. However, the contributions of temporal lag and spatial noise are less direct after quantization. To make these contributions explicit, we decompose the quantization error change into separate terms and evaluate their expected contributions using Lemma~\ref{lem:app-grid-moment-bound}. This yields the quantization term in Lemma~\ref{lem:app-quantization-expansion}, which depends on spatial noise and indirectly on temporal lag through the averaged latent mean:
\begin{equation}
\mathcal Q(\bm w)
=
\|\bm\mu_T^{\bm w}\|_2^2
+
P\sigma^2R(\bm w),
\qquad
\gamma_b
=
\frac{c_G}{(2^b-1)^2},
\label{eq:main-ptq-decomposition}
\end{equation}
where $c_G>0$ collects the constants associated with the simplified GPTQ model. Thus, we can analyze how the total reconstruction error $\mathcal E_b(\bm w)$ changes along the \LNWA{} frontier as temporal lag is reduced and more spatial noise is retained. Let $m_0:=(M-1)/2$ be the mean age of \LAWA{}, and set
$x:=1-m(\bm w_\lambda)/m_0\in[0,1]$ to measure the fraction of this age removed. Thus $x=0$ gives \LAWA{}, while $x=1$ selects only the most recent saved weight, corresponding to $\lambda=\infty$. Thus, increasing $x$ reduces temporal lag and retains more spatial noise. Each $x$ specifies a frontier kernel $\bm w_\lambda$; we write its averaged latent mean and retained noise factor as $\bm\mu(x):=\bm\mu_T^{\bm w_\lambda}$ and $R(x):=R(\bm w_\lambda)$. To identify the tolerance region, we compare $\mathcal E_b(x)$ with its value at \LAWA{}, $\mathcal E_b(0)$. We first examine the error on a fixed-support segment, where the frontier formulas in Lemma~\ref{lem:app-frontier-coordinates} show that the averaged latent mean is affine in $x$ and $R(x)$ is quadratic. For $q\geq2$ active weights, differentiating the reconstruction error gives
\begin{equation}
\frac{\mathrm d^2}{\mathrm dx^2}\mathcal E_b(x)
=
m_0^2
\left[
\bm a_q^\top\bm G\bm a_q
+
2\gamma_b\|\bm a_q\|_2^2
+
\frac{2N_G+2\gamma_bP\sigma^2}{D_q}
\right]
>0,
\label{eq:main-segment-curvature}
\end{equation}
where $\bm a_q:=m_0^{-1}\mathrm d\bm\mu(x)/\mathrm dx$ is constant on the segment and $D_q:=q(q^2-1)/12$. The positive curvature increases the error slope within each fixed-support segment. Although the slope may jump at support boundaries, Lemma~\ref{lem:app-turnover-support-jumps} controls these jumps. The comparison with the starting error $\mathcal E_b(0)$ therefore depends on their cumulative effect along $[0,x]$. To track the sign of this comparison, we normalize the error difference by $x$:
\begin{equation}
F_b(x)
:=
\frac{\mathcal E_b(x)-\mathcal E_b(0)}{x},
\qquad
x>0,
\qquad
F_b(0):=\mathcal E_b'(0^+).
\label{eq:main-average-change}
\end{equation}
For $x>0$, $F_b(x)$ has the same sign as $\mathcal E_b(x)-\mathcal E_b(0)$. Its derivative accounts for the curvature accumulated over $[0,x]$ and the slope jumps at boundaries crossed along the way. Lemma~\ref{lem:app-short-window-main} shows that, for sufficiently large $T$, the positive curvature dominates any adverse cumulative contribution from these jumps. Eq.~(\ref{eq:app-turnover-quotient-bounds}) therefore gives constants $0<k_b\leq K_b$ such that
\begin{equation}
\frac{k_b}{2}
\leq
F_b'(x)
\leq
\frac{K_b}{2}.
\label{eq:main-average-change-slope}
\end{equation}
Since $F_b$ is continuous, the positive lower bound makes it strictly increasing over the entire frontier. Lemma~\ref{lem:app-short-window-main} also gives $F_b(1)=\mathcal E_b(1)-\mathcal E_b(0)>0$. Integrating the derivative bounds gives the quantitative comparison in the following theorem.

\begin{theorem}[Finite Noise Tolerance and Degradation]
\label{thm:ptq-turnover}
Under Assumptions~\ref{ass:app-bounded-noise}--\ref{ass:app-short-window},
for all sufficiently large $T$ and every evaluated precision
$b\geq2$, there exist constants $0<k_b\leq K_b$ and a unique
threshold $\tau_b\in[0,1)$. With $\ell=(M-1)/T$,
$g=\|\bm G\|_2$, and $A_*=\|\bm\mu_T\|_2+V\ell$, where $V$ is the
increment bound in Assumption~\ref{ass:app-window-regularity},
the threshold satisfies
\[
 0\leq\tau_b\leq\frac{D_MV}{m_0T}
 \left(\frac{gV\ell}{N_G}+\frac{2A_*}{P\sigma^2}\right).
\]
\textbf{(i) Tolerance:} for $0<x\leq\tau_b$,
\begin{equation}
-\frac{K_b}{2}x(\tau_b-x)
\leq
\mathcal E_b(x)-\mathcal E_b(0)
\leq
-\frac{k_b}{2}x(\tau_b-x),
\label{eq:main-tolerance-bound}
\end{equation}
\textbf{(ii) Degradation:} for $\tau_b<x\leq1$,
\begin{equation}
\mathcal E_b(x)-\mathcal E_b(0)
\geq
\frac{k_b}{2}x(x-\tau_b)
>0.
\label{eq:main-degradation-bound}
\end{equation}
\end{theorem}

Theorem~\ref{thm:ptq-turnover} characterizes noise tolerance and degradation at a fixed precision. When $\tau_b>0$, the recovery of recent progress compensates for the additional spatial noise and quantization error over $0<x<\tau_b$; beyond $\tau_b$, the reconstruction error exceeds its value at $x=0$ and increases strictly. If $\tau_b=0$, any positive movement causes degradation. This is because the model may be unable to tolerate any additional spatial noise at sufficiently low precision. The remaining question is: at what precision does the $\tau_b=0$ regime arise and favor less spatial noise?

\subsection{Precision determines the preferred balance}
\label{sec:bit-preference}

By Eqs.~(\ref{eq:main-loss-definition}) and~(\ref{eq:main-ptq-decomposition}),
lower precision increases the weight on the quantization term $\mathcal Q$.
Lemma~\ref{lem:app-signal-accumulation}
shows that the retained latent-mean norm cannot decrease as $x$ increases.
The retained noise factor also increases, so on every open support segment,
\begin{equation}
\mathcal Q'(x)
=2\bm\mu(x)^\top\bm\mu'(x)+P\sigma^2R'(x)>0,
\qquad x>0.
\label{eq:main-quantization-monotonicity}
\end{equation}
Thus, $\mathcal Q$ increases as the kernel shifts toward more
recent weights. Lemma~\ref{lem:app-local-progress} gives
$\mathcal P_0'(0^+)<0$, so $\mathcal P_0$ initially decreases.
As precision varies, whether this decrease offsets the increase in
$\gamma_b\mathcal Q$ determines which tolerance regime in
Theorem~\ref{thm:ptq-turnover} applies: $\tau_b>0$ when $\mathcal E_b$
initially decreases, and $\tau_b=0$ otherwise. To characterize where
this change occurs, Lemma~\ref{lem:app-preference-interior} places the
unquantized and evaluated-precision optima in the first support
segment. Write $\bar{\bm\mu}=\bm\mu(0)$ and
$\bm u=\bm\mu'(0^+)=m_0\bm a_M$. On this segment, both error
components are quadratic:
\begin{equation}
\begin{aligned}
\mathcal P_0(x)-\mathcal P_0(0)
=\mathcal P_0'(0^+)x+\frac{\mathcal P_0''(0^+)}{2}x^2,\quad
\mathcal Q(x)-\mathcal Q(0)
=\mathcal Q'(0^+)x+\frac{\mathcal Q''(0^+)}{2}x^2.
\end{aligned}
\label{eq:main-first-segment-quadratics}
\end{equation}
Appendix~\ref{app:bit-selection} gives
$\mathcal Q'(0^+)\geq0$, $a:=\mathcal P_0''(0^+)>0$, and
$c:=\mathcal Q''(0^+)>0$. The unquantized optimum and the optimum
at precision $b$ are
\begin{equation}
x_0=\frac{\bm u^\top\bm G(\bm\mu_T-\bar{\bm\mu})}{a},
\label{eq:main-unquantized-optimum}
\end{equation}
and
\begin{equation}
x_b=\max\left\{0,
\frac{\bm u^\top\bm G(\bm\mu_T-\bar{\bm\mu})
-2\gamma_b\bar{\bm\mu}^\top\bm u}
{a+c\gamma_b}\right\}\in[0,x_0).
\label{eq:main-precision-minimizer}
\end{equation}
Eq.~(\ref{eq:main-precision-minimizer}) shows that quantization
can make $x_b=0$, corresponding to $\tau_b=0$ in
Theorem~\ref{thm:ptq-turnover}. Lemma~\ref{lem:app-preference-interior}
establishes that this $x_b$ is the global optimum across all support
segments, so we can quantify the precision threshold in the next
theorem.

\begin{theorem}[Precision Threshold for Lag--Noise Preference]
\label{thm:bit-ordering}
Under Assumptions~\ref{ass:app-bounded-noise}--\ref{ass:app-nonvanishing-progress},
for sufficiently large $T$, define the precision threshold
\begin{equation}
b_{\mathrm{crit}}
:=\log_2\!\left(
1+\sqrt{\frac{2c_G\bar{\bm\mu}^\top\bm u}
{\bm u^\top\bm G(\bm\mu_T-\bar{\bm\mu})}}
\right).
\label{eq:main-critical-precision}
\end{equation}
For every evaluated precision $b\geq b_{\min}$,
\begin{equation}
\begin{aligned}
b\leq b_{\mathrm{crit}}
&\quad\Longrightarrow\quad x_b=0,\quad\tau_b=0
 \qquad\text{\textbf{(less spatial noise)}},\\
b>b_{\mathrm{crit}}
&\quad\Longrightarrow\quad x_b>0,\quad\tau_b>0
 \qquad\text{\textbf{(less temporal lag)}}.
\end{aligned}
\label{eq:main-preference-reversal}
\end{equation}
Moreover, $x_b$ is nondecreasing in $b$ and strictly increasing whenever
$b>b_{\mathrm{crit}}$.
If $b_{\mathrm{crit}}<b_{\min}$, all evaluated precisions lie in the
second regime.
\end{theorem}

Theorem~\ref{thm:bit-ordering} answers the question left at the end of
Section~\ref{sec:noise-tolerance}: $b_{\mathrm{crit}}$ marks where zero
tolerance gives way to a positive tolerance interval. This
threshold characterizes the reversal in the preferred lag--noise
balance, from denoising at lower precision to reducing temporal lag at
higher precision.

\section{Experiments}
\label{sec:experiments}
\textbf{Experimental setup.}
We pretrain OPT-series models (125M, 350M, and 1.3B) \citep{zhang2022opt} from scratch on FineWeb-Edu \citep{penedo2024fineweb} with AdamW \citep{loshchilov2017decoupled}. Following prior work \citep{catalan2026training}, we save 10 weights at 500-step intervals near the end of training and average them. We then apply GPTQ \citep{frantar2022gptq} and AWQ \citep{lin2024awq} at 4, 3, and 2 bits (denoted by W4, W3, W2) using the same C4 calibration set \citep{raffel2020exploring}, then evaluate the PTQ performance by perplexity (PPL; lower is better) on a held-out RefinedWeb set \citep{penedo2023refinedweb}.
Complete settings and results are in Appendix~\ref{app:experiments}.

\textbf{From noise tolerance to degradation.}
We evaluate \LNWA{} across coefficients ranging from $\lambda=0$ (\LAWA{}) to $\lambda=1$. We report the performance change \(\Delta\mathrm{NLL}(\lambda)\) relative to \LAWA{}, where \(\Delta\mathrm{NLL}(\lambda)
=\mathrm{NLL}(\lambda)-\mathrm{NLL}(0)
=\log\frac{\mathrm{PPL}(\lambda)}{\mathrm{PPL}(0)}\), and negative values indicate an improvement.

\begin{figure}
  \centering
  \includegraphics[width=\linewidth]{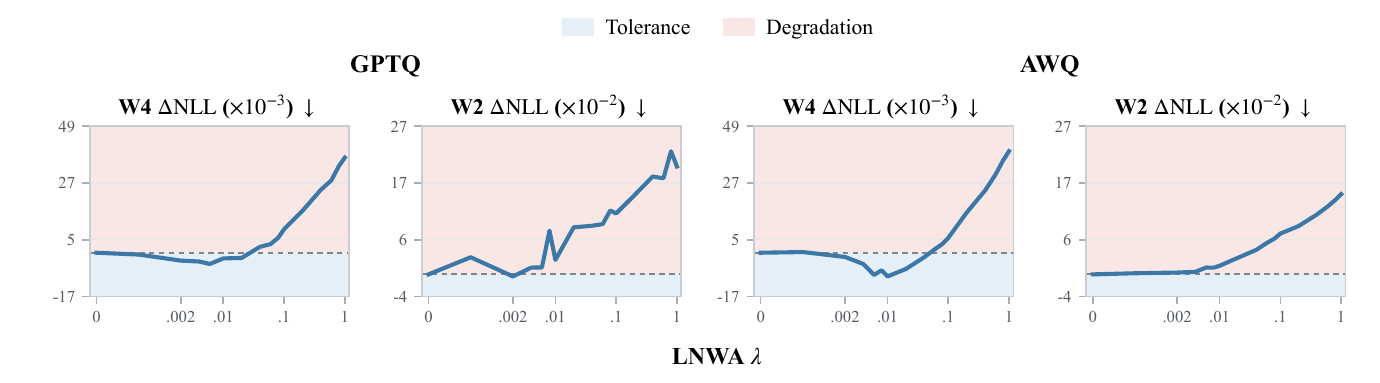}
  \caption{Comparison of $\Delta$NLL across \LNWA{} coefficients for OPT-125M under GPTQ and AWQ at 4-bit (W4) and 2-bit (W2).}
  \label{fig:lnwa-turnover-main}
\end{figure}

Figure~\ref{fig:lnwa-turnover-main} shows that NLL can fall below the
\LAWA{} baseline at small $\lambda$, then rises above it at larger
$\lambda$. This indicates that the benefit of reducing temporal lag
can offset or outweigh the cost of retaining additional spatial noise.
Moving further along the frontier retains more noise and leads to
performance degradation, consistent with the tolerance-to-degradation pattern characterized in Theorem~\ref{thm:ptq-turnover}.

\textbf{Precision changes the preferred lag-noise balance.}
We compare \LAWA{}, \WMA{}, and \LNWA{} in the small-$\lambda$
range by fitting PPL against the \LNWA{}-equivalent averaging coefficient.
Kernel choices and regression details are given in
Appendix~\ref{app:averaging-settings}.

\begin{figure}
  \centering
  \includegraphics[width=\linewidth]{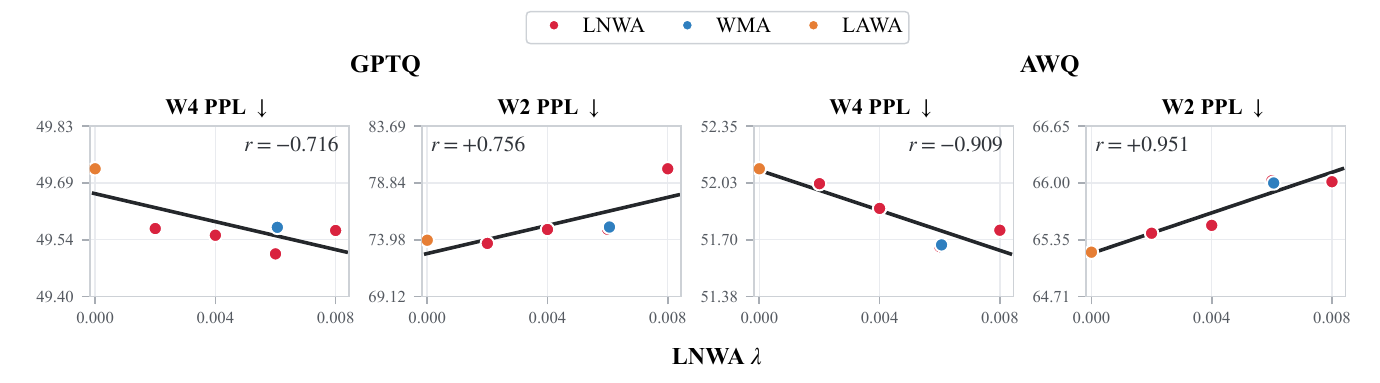}
  \caption{Comparison of PPL across \LNWA{}-equivalent averaging coefficients for OPT-125M under GPTQ and AWQ at 4-bit (W4) and 2-bit (W2).}
  \label{fig:lambda-precision-reversal}
\end{figure}

Figure~\ref{fig:lambda-precision-reversal}
shows a positive correlation between PPL and the averaging coefficient
at W2 and a negative correlation at W4. Smaller coefficients retain
less spatial noise, whereas larger coefficients reduce temporal lag.
These correlations therefore indicate that lower precision favors less
noise and higher precision favors less lag, supporting a precision-dependent reversal in preference in Theorem~\ref{thm:bit-ordering}.

\textbf{Comparison with learning-rate decay.}
We finally compare \LNWA{} with two
LR decay schedules on OPT-1.3B. As shown in Figure~\ref{fig:lnwa-vs-lr-decay}, \LNWA{} achieves lower average PPL (Avg.) than both decay schedules under GPTQ and AWQ, with the largest improvement at W2.
These results support weight averaging as an effective strategy
for the pretraining-to-quantization handoff. 

\begin{figure}
  \centering
  \includegraphics[width=\linewidth]{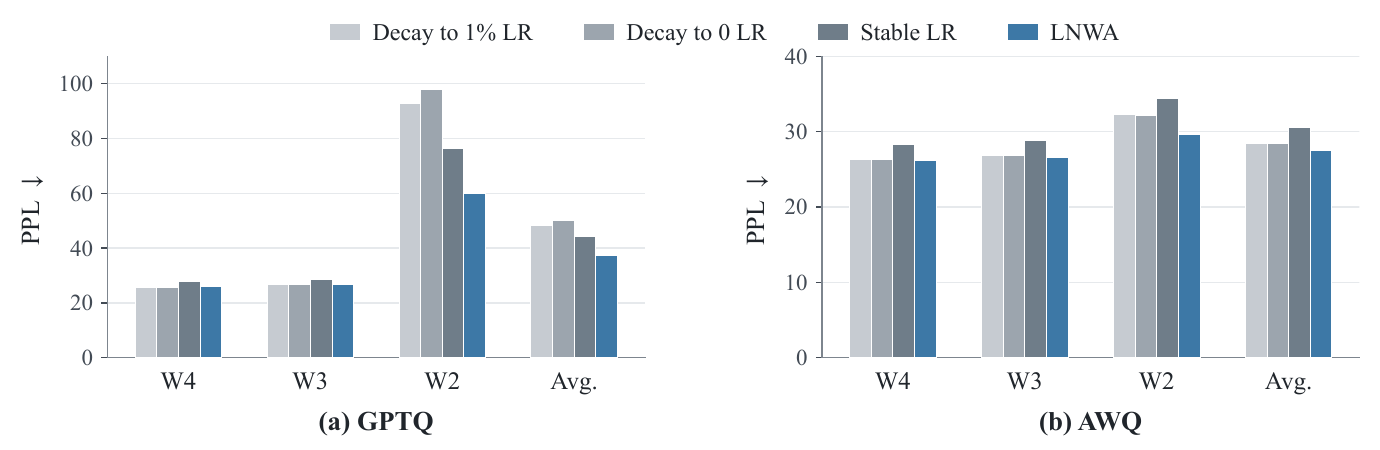}
  \caption{Comparison of PPL between LR decay, Stable LR, and \LNWA{} for OPT-1.3B under GPTQ and AWQ at 4-bit (W4), 3-bit (W3), 2-bit (W2), and their average.}
  \label{fig:lnwa-vs-lr-decay}
\end{figure}

\section{Conclusion}
\label{sec:conclusion}

We developed a theoretical framework for weight averaging at the
pretraining-to-quantization transition. Averaging suppresses spatial
noise but introduces temporal lag by sacrificing recent optimization
progress. We quantify this trade-off and derive \LNWA{}, a closed-form
family of averaging kernels that parameterizes its Pareto frontier. Through a simplified GPTQ model, we connect this trade-off to output
reconstruction error and establish two main results. At a fixed
precision, reduced lag can compensate for additional noise within a
finite tolerance interval; beyond it, the error increases. Within this
interval, we further show that lower precision favors less spatial
noise, whereas higher precision favors less temporal lag.
Experiments on LLMs support
both results.

\subsection*{AI use statement}
Generative AI tools were used only to assist with literature search, coding, writing refinement, and consistency checking. The research motivation, theoretical development, experiments, and conclusions were independently developed and completed by the authors. All AI-assisted outputs were carefully reviewed and verified, and the authors take full responsibility for the contents of this paper.


\bibliography{iclr2027_conference}
\bibliographystyle{iclr2027_conference}

\newpage
\appendix
\startcontents[appendix]

\pdfbookmark[1]{Appendix}{appendix-contents}
\section*{Appendix}
\begingroup
\small
\setcounter{tocdepth}{2} 
\printcontents[appendix]{}{1}{}
\endgroup

\section{Theoretical Proofs}
\label{app:proofs}

This section presents the detailed proofs of all theoretical results in Section~\ref{sec:method}. 
\subsection{Setup}
\label{app:setup}
\begin{definition}[Finite Convex Weight Averaging]
\label{def:app-window}

Let
$\bm{\theta}_T,\bm{\theta}_{T-1},\ldots,
\bm{\theta}_{T-M+1}\in\mathbb R^P$
be a window of $M\geq2$ retained model weights.
Age $d=0$ denotes the newest weights; $T$ denotes the training steps. Define
\begin{equation*}
\mathcal W_M
=
\left\{
\bm{w}\in\mathbb R^M:
w_d\geq0,\ 
\sum_{d=0}^{M-1}w_d=1
\right\},
\qquad
\bm{\theta}_T^{\bm{w}}
=
\sum_{d=0}^{M-1}w_d\bm{\theta}_{T-d}.
\end{equation*}

Here, $\mathcal W_M$ is the set of feasible averaging kernels,
$\bm{w}\in\mathcal W_M$ is an averaging kernel, and
$\bm{\theta}_T^{\bm{w}}$ lies in the finite convex hull of the retained weights.
\end{definition}

\begin{definition}[Latent-mean Trajectory]
\label{def:app-latent-mean}
Let $\mathcal F_0$ be the sigma-field generated by the information shared
across stochastic training runs, such as the model configuration, common
initialization, and training setup. Let $\xi$ index the run-specific
stochasticity, such as data ordering and dropout. For the aligned state
$\bm{\theta}_{T-d}(\xi)$, assume $
\E[
\|\bm{\theta}_{T-d}(\xi)\|_2^2
\mid\mathcal F_0
]
<\infty
$. The conditional mean state $\bm\mu_t$ and the zero-mean spatial noise
$\bm\epsilon_t$ at the training step $t$ of an individual run are
\begin{equation*}
\bm{\mu}_{T-d}
:=
\E_{\xi}
[
\bm{\theta}_{T-d}(\xi)
\mid\mathcal F_0
],
\qquad
\bm{\epsilon}_{T-d}(\xi)
:=
\bm{\theta}_{T-d}(\xi)-\bm{\mu}_{T-d}.
\end{equation*}

A stochastic training run therefore satisfies
$\bm{\theta}_{T-d}=\bm{\mu}_{T-d}+\bm{\epsilon}_{T-d}$ and
$\E[\bm{\epsilon}_{T-d}\mid\mathcal F_0]=\bm{0}$.
\end{definition}

Throughout Appendix~\ref{app:proofs}, all expectations and covariances
condition on $\mathcal F_0$, and averaging kernels are fixed given
$\mathcal F_0$.

For ages $d,e\in\{0,\ldots,M-1\}$, define the conditional covariance
blocks of the fluctuations around the latent-mean trajectory:
\begin{equation}
\bm{\Gamma}_{de}
:=
\Cov(
\bm{\epsilon}_{T-d},
\bm{\epsilon}_{T-e}
\mid\mathcal F_0
).
\label{eq:app-covariance-blocks}
\end{equation}

Within the retained stable-LR window, we use a common covariance for
spatial noise and assume that fluctuations at different retained positions
are uncorrelated.

\begin{assumption}[Common Covariance and Uncorrelated Spatial Noise]
\label{ass:app-noise}
Conditional on $\mathcal F_0$, the spatial-noise vectors have finite
second moments and satisfy, for $d,e\in\{0,\ldots,M-1\}$,
\begin{equation*}
 \E[\bm\epsilon_{T-d}\mid\mathcal F_0]=\bm0,\qquad
 \E[\bm\epsilon_{T-d}\bm\epsilon_{T-e}^{\top}\mid\mathcal F_0]
 =\begin{cases}
   \bm\Sigma,&d=e,\\
   \bm0,&d\ne e,
  \end{cases}
 \qquad \bm\Sigma\succeq0.
\end{equation*}
\end{assumption}

Thus $\bm\Gamma_{dd}=\bm\Sigma$ and
$\bm\Gamma_{de}=\bm0$ for $d\ne e$.

Define the forward latent increment $
\bm{v}_{T-j+1}
=
\bm{\mu}_{T-j+1}-\bm{\mu}_{T-j}, 1\leq j<M,
$ and the cumulative old-weight mass $
s_j(\bm{w})
=
\sum_{d=j}^{M-1}w_d$, 
we have
\begin{equation*}
\bm{b}(\bm{w})
=
\sum_{j=1}^{M-1}s_j(\bm{w})\bm{v}_{T-j+1},
\qquad
\bm{\epsilon}_T^{\bm{w}}
=
\sum_{d=0}^{M-1}w_d\bm{\epsilon}_{T-d},
\qquad
R(\bm{w})
=
\sum_{d=0}^{M-1}w_d^2.
\end{equation*}

Generally, we have
\begin{equation*}
\Cov(
\bm{\epsilon}_T^{\bm{w}}
\mid\mathcal F_0
)
=
\sum_{d,e}
w_dw_e\bm{\Gamma}_{de}.
\end{equation*}

Under Assumption~\ref{ass:app-noise}, this reduces to
\begin{equation*}
\Cov(
\bm{\epsilon}_T^{\bm{w}}
\mid\mathcal F_0
)
=
R(\bm{w})\bm{\Sigma}.
\end{equation*}

Accordingly, $R(\bm{w})^{-1}$ is a variance-equivalent
effective sample size.

\subsection{Temporal lag and spatial noise}
\label{app:decomposition}

\begin{lemma}[Lag-noise Decomposition]
\label{lem:app-exact}
Every $\bm{w}\in\mathcal W_M$ satisfies
\begin{equation}
  \bm{\theta}_T^{\bm{w}}
  =
  \bm{\mu}_T-\bm{b}(\bm{w})+\bm{\epsilon}_T^{\bm{w}}.
  \label{eq:app-exact}
\end{equation}
Under Assumption~\ref{ass:app-noise}, its residual covariance is
\begin{equation*}
  \Cov(\bm{\epsilon}_T^{\bm{w}}\mid\mathcal F_0)
  =
  R(\bm{w})\bm{\Sigma}.
\end{equation*}
\end{lemma}

\begin{proof}
The increments can be written as
$\bm{\mu}_{T-d}=\bm{\mu}_T-\sum_{j=1}^{d}\bm{v}_{T-j+1}$.  Substituting this identity into
$\sum_dw_d(\bm{\mu}_{T-d}+\bm{\epsilon}_{T-d})$ gives
\begin{equation*}
  \sum_{d=0}^{M-1}w_d\sum_{j=1}^{d}\bm{v}_{T-j+1}
  =
  \sum_{j=1}^{M-1}
  \left(\sum_{d=j}^{M-1}w_d\right)\bm{v}_{T-j+1}.
\end{equation*}
This proves Eq.~(\ref{eq:app-exact}).  The conditional
cross-covariances of distinct $\bm{\epsilon}_{T-d}$ vanish by
Assumption~\ref{ass:app-noise}, hence
$\Cov(\bm{\epsilon}_T^{\bm{w}}\mid\mathcal F_0)
=\sum_dw_d^2\bm{\Sigma}
=R(\bm{w})\bm{\Sigma}$.
\end{proof}

\begin{assumption}[Bounded Adjacent Progress Variation]
\label{ass:app-progress}
Adjacent latent increments in the retained discrete window obey
\begin{equation*}
  \|\bm{v}_{T-j+1}-\bm{v}_{T-j+2}\|_2\leq\kappa,\qquad 2\leq j<M,
\end{equation*}
for a finite $\kappa\geq0$.
\end{assumption}


Under Assumption~\ref{ass:app-progress}, the next lemma reduces the trajectory displacement to a scalar lag coordinate with a controlled remainder.

\begin{lemma}[Scalar Lag with Controlled Remainder]
\label{lem:app-scalar}
Under Assumption~\ref{ass:app-progress}, define
$m(\bm{w})=\sum_{d=0}^{M-1}dw_d$.  Then
\begin{equation}
  \bm{\theta}_T^{\bm{w}}
  =
  \bm{\mu}_T-m(\bm{w})\bm{v}_T
  +\bm{\epsilon}_T^{\bm{w}}
  +\bm{e}_{\bm{w}}^{\mathrm{traj}},
  \qquad
  \|\bm{e}_{\bm{w}}^{\mathrm{traj}}\|_2\leq
  \frac{\kappa}{2}\sum_{d=0}^{M-1}w_dd(d-1).
  \label{eq:app-remainder}
\end{equation}
When $\bm{v}_{T-j+1}=\bm{v}_T$ throughout the sampled window,
$\bm{e}_{\bm{w}}^{\mathrm{traj}}=\bm{0}$.
\end{lemma}

\begin{proof}
The cumulative weights satisfy
\begin{equation*}
  \sum_{j=1}^{M-1}s_j(\bm{w})
  =
  \sum_{d=0}^{M-1}d\,w_d
  =
  m(\bm{w}).
\end{equation*}
Define the trajectory deviation
$\bm{\delta}_j^{\mathrm{traj}}:=\bm{v}_{T-j+1}-\bm{v}_T$ for
$1\leq j<M$.  Then
\begin{equation*}
  \bm{b}(\bm{w})
  =
  m(\bm{w})\bm{v}_T
  +
  \sum_{j=1}^{M-1}
  s_j(\bm{w})\bm{\delta}_j^{\mathrm{traj}}.
\end{equation*}
Adjacent variation implies
$\|\bm{\delta}_j^{\mathrm{traj}}\|_2\leq(j-1)\kappa$.
Non-negativity of the weights and
\begin{equation*}
  \sum_{j=1}^{M-1}(j-1)s_j(\bm{w})
  =\frac12\sum_{d=0}^{M-1}w_dd(d-1)
\end{equation*}
give Eq.~(\ref{eq:app-remainder}) after setting
\begin{equation*}
  \bm{e}_{\bm{w}}^{\mathrm{traj}}
  =
  -\sum_{j=1}^{M-1}
  s_j(\bm{w})\bm{\delta}_j^{\mathrm{traj}}.
\end{equation*}
\end{proof}

\subsection{Lag--noise Pareto frontier and closed-form \LNWA{} kernel}
\label{app:pareto}

The decomposition above motivates the kernel coordinates.
\begin{definition}[Lag--noise Coordinates]
\label{def:main-lag-noise-coordinates2}
For any $\bm w\in\mathcal W_M$, define
\begin{equation*}
\Phi(\bm w)
=
\bigl(m(\bm w),R(\bm w)\bigr),
\qquad
\mathcal G_M
=
\left\{
\Phi(\bm w):\bm w\in\mathcal W_M
\right\}.
\label{eq:main-lag-noise-image2}
\end{equation*}
Here, $\Phi(\bm w)$ denotes the lag--noise coordinates of the averaging
kernel $\bm w$, and $\mathcal G_M$ is the set of all feasible lag--noise
coordinates.
\end{definition}
Write $\operatorname{Eff}(\mathcal G_M)$ for the set of Pareto-efficient points
in $\mathcal G_M$: a point belongs to this set if no feasible kernel has
both a smaller or equal temporal lag and a smaller or equal spatial noise,
with at least one strict inequality.

To identify the efficient set, first fix a lag budget $\ell$:
\begin{equation}
  \bm{w}^\star(\ell)
  =\arg\min_{\bm{w}\in\mathcal W_M}R(\bm{w})
  \quad\text{s.t.}\quad m(\bm{w})\leq\ell,
  \qquad 0\leq\ell\leq\frac{M-1}{2}.
  \label{eq:app-lag-budget}
\end{equation}

\begin{lemma}[Lag-noise Pareto Frontier]
\label{lem:app-budget}
For $0<\ell<(M-1)/2$, Eq.~(\ref{eq:app-lag-budget}) has a unique
solution supported on the $q$ most recent saved weights, where
\begin{equation}
  q=q(\ell)=\min\{M,\lceil3\ell+1\rceil\}\in\{2,\ldots,M\}.
  \label{eq:app-budget-support}
\end{equation}
Its coordinates are
\begin{equation}
  \bigl(\bm{w}^\star(\ell)\bigr)_d=
  \begin{cases}
    q^{-1}-\gamma_q(d-\bar d_q),&0\leq d<q,\\
    0,&q\leq d<M,
  \end{cases}
  \label{eq:app-budget-kernel}
\end{equation}
where
\begin{equation}
  \bar d_q=\frac{q-1}{2},\quad
  D_q=\sum_{d=0}^{q-1}(d-\bar d_q)^2=\frac{q(q^2-1)}{12},
  \quad
  \gamma_q=\frac{\bar d_q-\ell}{D_q},
  \label{eq:app-support-moments}
\end{equation}
and
\begin{equation*}
  R^\star(\ell)=\frac1q+
  \frac{(\bar d_q-\ell)^2}{D_q}.
\end{equation*}
\end{lemma}

\begin{proof}
The feasible set is compact, and the objective is strictly convex.  Hence the
minimizer is unique.  With $\bm d=(0,1,\ldots,M-1)^\top$, the Lagrangian is
\begin{equation*}
  \mathcal L_{\mathrm{KKT}}(\bm w,\tau,\nu,\bm\alpha)
  =\|\bm w\|_2^2+\tau(\bm d^\top\bm w-\ell)
  +\nu(\bm 1^\top\bm w-1)-\bm\alpha^\top\bm w.
\end{equation*}
Slater's condition holds for $\ell>0$, so the KKT conditions are
\begin{equation}
  2\bm w+\tau\bm d+\nu\bm 1-\bm\alpha=0,\qquad
  \tau(\bm d^\top\bm w-\ell)=0,\qquad
  \alpha_dw_d=0,
\label{eq:app-budget-kkt}
\end{equation}
with $\tau\geq0$ and $\bm\alpha\geq0$.  If $\tau=0$, they give the uniform
simplex minimizer, which is infeasible because $(M-1)/2>\ell$.  Thus
$\tau>0$ and the lag constraint is active.  Stationarity and complementary
slackness in Eq.~(\ref{eq:app-budget-kkt}) give
\begin{align}
  w_d>0
  &\ \Longrightarrow\ \alpha_d=0
  \ \Longrightarrow\ w_d=-\frac{\nu+\tau d}{2}>0,
  \label{eq:app-budget-active-coordinate}\\
  w_d=0
  &\ \Longrightarrow\ -\frac{\nu+\tau d}{2}
  =-\frac{\alpha_d}{2}\leq0. \notag
\end{align}
Combining the two cases yields
\begin{equation}
  w_d=\left[-\frac{\nu+\tau d}{2}\right]_+.
  \label{eq:app-budget-positive-part}
\end{equation}
Because $\tau>0$, the expression inside the positive part decreases with
$d$.  Hence the support is a prefix $S_q=\{0,\ldots,q-1\}$.  For
$d\in S_q$, the positive-part operator is inactive and
\begin{equation}
  w_d=-\frac{\nu}{2}-\frac{\tau}{2}d.
\label{eq:app-budget-active-affine}
\end{equation}
Using $\sum_{d=0}^{q-1}w_d=1$ and
$\sum_{d=0}^{q-1}d=q\bar d_q$ in
Eq.~(\ref{eq:app-budget-active-affine}) gives
\begin{equation*}
  1=-\frac{q\nu}{2}-\frac{\tau q\bar d_q}{2},
  \qquad
  -\frac{\nu}{2}=\frac1q+\frac{\tau}{2}\bar d_q.
\end{equation*}
Substitution back into Eq.~(\ref{eq:app-budget-active-affine}) yields
\begin{equation*}
  w_d=\frac1q-\frac{\tau}{2}(d-\bar d_q),
  \qquad d\in S_q.
\end{equation*}
Therefore, with
$\bm c_q=(0,1,\ldots,q-1)^\top-\bar d_q\bm 1_q$,
\begin{equation*}
  \bm w^\star_{S_q}
  =q^{-1}\bm 1_q-\frac{\tau}{2}\bm c_q.
\end{equation*}
Finally, the lag constraint is active, so
\begin{equation*}
\begin{aligned}
\ell
&=\sum_{d=0}^{q-1} d\,w_d \\
&=\bar d_q-\frac{\tau}{2}
  \sum_{d=0}^{q-1} d(d-\bar d_q) \\
&=\bar d_q-\frac{\tau}{2}D_q .
\end{aligned}
\end{equation*}
Hence $\tau/2=(\bar d_q-\ell)/D_q=\gamma_q$, which proves
Eqs.~(\ref{eq:app-budget-kernel})--(\ref{eq:app-support-moments}).
The last-active and first-inactive conditions reduce to
\begin{equation*}
  \frac{q-2}{3}<\ell\leq\frac{q-1}{3}\quad(q<M),
  \qquad
  \ell>\frac{M-2}{3}\quad(q=M),
\end{equation*}
which gives Eq.~(\ref{eq:app-budget-support}).  Finally,
$\bm 1_q^\top\bm c_q=0$ and $\|\bm c_q\|_2^2=D_q$, so
\begin{equation*}
  R(\bm w^\star(\ell))
  =\left\|q^{-1}\bm 1_q-\frac{\tau}{2}\bm c_q\right\|_2^2
  =\frac1q+\frac{(\bar d_q-\ell)^2}{D_q}.
\end{equation*}
\end{proof}

We next parameterize the budget-optimal kernels by a single coefficient.

\begin{theorem}[Lag-noise Trade-off Weight Averaging (\LNWA{}) Kernel]
\label{thm:app-frontier}
For $0\leq\lambda<\infty$, define the Lag-noise Trade-off Weight Averaging (\LNWA{}) kernel by
\begin{equation}
  \bm{w}_\lambda
  =\arg\min_{\bm{w}\in\mathcal W_M}
  \{R(\bm{w})+\lambda m(\bm{w})^2\}.
  \label{eq:app-lnwa-objective}
\end{equation}
The finite-$\lambda$ family satisfies
\begin{equation}
  \{\Phi(\bm w_\lambda):0\leq\lambda<\infty\}
  =\operatorname{Eff}(\mathcal G_M)\setminus\{\Phi(\bm e_0)\}.
  \label{eq:app-lnwa-finite-frontier}
\end{equation}
Its unique solution has an active prefix of size $q\geq2$ and
\begin{equation}
  (\bm{w}_\lambda)_d=
  \begin{cases}
    q^{-1}+\lambda m_q(\lambda)(\bar d_q-d),&0\leq d<q,\\
    0,&q\leq d<M,
  \end{cases}
  \label{eq:app-lnwa-kernel}
\end{equation}
where
\begin{equation}
  m_q(\lambda)=m(\bm w_\lambda)
      =\frac{\bar d_q}{1+\lambda D_q},\qquad
  \bar d_q=\frac{q-1}{2},\qquad
  D_q=\sum_{d=0}^{q-1}(d-\bar d_q)^2
      =\frac{q(q^2-1)}{12}.
  \label{eq:app-lnwa-moment}
\end{equation}
\end{theorem}

\begin{proof}
Any kernel with $m(\bm w)>(M-1)/2$ is dominated by the uniform kernel.  For
$0\leq\ell\leq(M-1)/2$, uniqueness in
Eq.~(\ref{eq:app-lag-budget}) implies
\begin{equation}
  \operatorname{Eff}(\mathcal G_M)
  =\{\Phi(\bm w^\star(\ell)):0\leq\ell\leq(M-1)/2\}.
\label{eq:app-frontier-budget-characterization}
\end{equation}
Indeed, a dominator of $\bm w^\star(\ell)$ contradicts its optimality or
uniqueness; conversely, $\bm w^\star(m(\bm w))$ dominates every feasible
$\bm w$ that is not the budget minimizer.

For finite $\lambda$, Eq.~(\ref{eq:app-lnwa-objective}) is strictly convex.
By the same stationarity and complementary-slackness argument as
Eqs.~(\ref{eq:app-budget-active-coordinate})--(\ref{eq:app-budget-positive-part}),
its KKT conditions give
\begin{equation}
  w_d=\left[-\frac{\nu}{2}-\lambda m d\right]_+,
  \qquad m=m(\bm w_\lambda),
\label{eq:app-lnwa-kkt-form}
\end{equation}
so the active set is a prefix.  Moreover, for sufficiently small $t>0$,
\begin{equation}
\begin{aligned}
  &R\!\left((1-t)\bm e_0+t\bm e_1\right)
  +\lambda m\!\left((1-t)\bm e_0+t\bm e_1\right)^2 \\
  &\qquad =1-2t+(\lambda+2)t^2
  <1
  =R(\bm e_0)+\lambda m(\bm e_0)^2,
\end{aligned}
\label{eq:app-lnwa-not-latest-only}
\end{equation}
so $\bm e_0$ cannot solve Eq.~(\ref{eq:app-lnwa-objective}); hence $q\geq2$.
On the active prefix $S_q=\{0,\ldots,q-1\}$,
Eq.~(\ref{eq:app-lnwa-kkt-form}) becomes
\begin{equation}
  w_d=-\frac{\nu}{2}-\lambda m d,
  \qquad d\in S_q.
\label{eq:app-lnwa-active-affine}
\end{equation}
Using $\sum_{d=0}^{q-1}w_d=1$ and
$\sum_{d=0}^{q-1}d=q\bar d_q$ gives
\begin{equation*}
  1=-\frac{q\nu}{2}-\lambda m q\bar d_q,
  \qquad
  -\frac{\nu}{2}
  =\frac1q+\lambda m\bar d_q.
\end{equation*}
Substituting this identity into
Eq.~(\ref{eq:app-lnwa-active-affine}) yields
\begin{equation*}
  w_d=\frac1q+\lambda m(\bar d_q-d),
  \qquad d\in S_q.
\end{equation*}
Therefore,
\begin{align*}
  m
  &=\sum_{d=0}^{q-1}d\,w_d\\
  &=\bar d_q
    +\lambda m\sum_{d=0}^{q-1}d(\bar d_q-d)\\
  &=\bar d_q-\lambda mD_q,
\end{align*}
where
\[
\sum_{d=0}^{q-1}d(\bar d_q-d)
=q\bar d_q^2-\sum_{d=0}^{q-1}d^2
=-D_q.
\]
Consequently,
\[
m=\frac{\bar d_q}{1+\lambda D_q},
\]
which proves Eqs.~(\ref{eq:app-lnwa-kernel}) and
(\ref{eq:app-lnwa-moment}).

For $0<\ell<(M-1)/2$, set
\begin{equation*}
  \lambda=\frac{\gamma_q}{\ell}
  =\frac{\bar d_q-\ell}{D_q\ell}.
\end{equation*}
Then Eqs.~(\ref{eq:app-budget-kernel}) and
(\ref{eq:app-lnwa-kernel}) coincide, so every interior budget solution lies
on the \LNWA{} path. The budget endpoint
$\ell=(M-1)/2$ is the uniform kernel $\bm w_0$, whereas $\ell=0$ has the
unique solution $\bm e_0$.  Eq.~(\ref{eq:app-lnwa-not-latest-only})
excludes $\bm e_0$ for every finite $\lambda$.  Together with
Eq.~(\ref{eq:app-frontier-budget-characterization}), these facts prove
Eq.~(\ref{eq:app-lnwa-finite-frontier}).
\end{proof}

\begin{lemma}[Global Structure of the \LNWA{}]
\label{lem:app-endpoints-support}
The \LNWA{} path has endpoints
\begin{equation}
  \bm w_0=\frac1M\bm 1,\qquad
  \lim_{\lambda\to\infty}\bm w_\lambda=\bm e_0.
  \label{eq:app-frontier-endpoints}
\end{equation}
For $q\geq3$, define
\begin{equation*}
  \lambda_q^{\mathrm{knot}}=\frac{6}{q(q-1)(q-2)}.
\end{equation*}
If $M>2$, all $M$ coordinates are active on
$[0,\lambda_M^{\mathrm{knot}})$; for each
$q\in\{3,\ldots,M-1\}$, exactly $q$ coordinates are active on
$[\lambda_{q+1}^{\mathrm{knot}},\lambda_q^{\mathrm{knot}})$; and exactly two
coordinates are active on $[\lambda_3^{\mathrm{knot}},\infty)$.  If $M=2$,
both coordinates remain active for every finite $\lambda$.  Defining
$\bm w_\infty:=\bm e_0$ therefore completes the frontier:
\begin{equation*}
  \{\Phi(\bm w_\lambda):0\leq\lambda\leq\infty\}
  =\operatorname{Eff}(\mathcal G_M).
\end{equation*}
Moreover, for $0\leq\lambda_1<\lambda_2\leq\infty$,
\begin{equation}
  m(\bm{w}_{\lambda_2})<m(\bm{w}_{\lambda_1}),
  \qquad
  R(\bm{w}_{\lambda_2})>R(\bm{w}_{\lambda_1}).
  \label{eq:app-monotonicity}
\end{equation}
\end{lemma}

\begin{proof}
At $\lambda=0$, Eq.~(\ref{eq:app-lnwa-objective}) reduces to minimizing
$R(\bm w)$ over $\mathcal W_M$, whose unique solution is
$\bm w_0=M^{-1}\bm 1$.  This proves the first identity in
Eq.~(\ref{eq:app-frontier-endpoints}).

Suppose the active prefix is $S_q=\{0,\ldots,q-1\}$.  To verify that its size
is exactly $q$, we check its last active coefficient and the affine continuation
to the next coordinate, denoted by $\widetilde w_q$.
Eqs.~(\ref{eq:app-lnwa-kernel})--(\ref{eq:app-lnwa-moment}) give
\begin{equation*}
\begin{aligned}
  \widetilde w_q
  &:=q^{-1}+\lambda m_q(\lambda)(\bar d_q-q),\\
  (\bm w_\lambda)_{q-1}>0
  &\iff \lambda<\lambda_q^{\mathrm{knot}} &&(q\geq3),\\
  \widetilde w_q\leq0
  &\iff \lambda\geq\lambda_{q+1}^{\mathrm{knot}} &&(q<M).
\end{aligned}
\end{equation*}
These two boundary conditions give $[\lambda_{q+1}^{\mathrm{knot}},\lambda_q^{\mathrm{knot}})$; at each knot, the oldest active coefficient reaches zero.

When $M>2$, the two-coordinate segment is
$\lambda\in[\lambda_3^{\mathrm{knot}},\infty)$; when $M=2$, it is the entire
finite path $\lambda\in[0,\infty)$.  On this segment,
Theorem~\ref{thm:app-frontier} gives
\begin{equation}
  \bm w_\lambda
  =\left(\frac{\lambda+1}{\lambda+2},
          \frac1{\lambda+2},0,\ldots,0\right).
  \label{eq:app-lnwa-two-coordinate}
\end{equation}
Taking
$\lambda\to\infty$ in Eq.~(\ref{eq:app-lnwa-two-coordinate}) gives
\begin{equation*}
  \lim_{\lambda\to\infty}\bm w_\lambda=\bm e_0.
\end{equation*}
This proves the second identity in Eq.~(\ref{eq:app-frontier-endpoints}).

To prove Eq.~(\ref{eq:app-monotonicity}), first suppose
$0\leq\lambda_1<\lambda_2<\infty$.  Optimality at the two coefficients gives
\begin{align}
  R(\bm w_{\lambda_1})+\lambda_1m(\bm w_{\lambda_1})^2
  &\leq R(\bm w_{\lambda_2})+\lambda_1m(\bm w_{\lambda_2})^2,
  \label{eq:app-monotonicity-opt-1}\\
  R(\bm w_{\lambda_2})+\lambda_2m(\bm w_{\lambda_2})^2
  &\leq R(\bm w_{\lambda_1})+\lambda_2m(\bm w_{\lambda_1})^2.
  \label{eq:app-monotonicity-opt-2}
\end{align}
Adding Eqs.~(\ref{eq:app-monotonicity-opt-1}) and
(\ref{eq:app-monotonicity-opt-2}) yields
\begin{equation*}
  (\lambda_2-\lambda_1)
  \left[m(\bm w_{\lambda_2})^2-m(\bm w_{\lambda_1})^2\right]\leq0.
\end{equation*}
Thus,
\begin{equation*}
  m(\bm w_{\lambda_2})\leq m(\bm w_{\lambda_1}).
\end{equation*}
Equality would force $R(\bm w_{\lambda_2})=R(\bm w_{\lambda_1})$ by
Eqs.~(\ref{eq:app-monotonicity-opt-1})--(\ref{eq:app-monotonicity-opt-2}),
so strict convexity would give $\bm w_{\lambda_2}=\bm w_{\lambda_1}$.
But Eq.~(\ref{eq:app-lnwa-active-affine}) gives
$(\bm w_\lambda)_0-(\bm w_\lambda)_1=\lambda m(\bm w_\lambda)$, with
$m(\bm w_\lambda)>0$ for every finite $\lambda$. Thus the same kernel
cannot occur at distinct coefficients, and the inequality is strict.
Eq.~(\ref{eq:app-monotonicity-opt-1}) then gives
\begin{equation*}
  R(\bm w_{\lambda_1})-R(\bm w_{\lambda_2})
  \leq\lambda_1
  \left[m(\bm w_{\lambda_2})^2-m(\bm w_{\lambda_1})^2\right]\leq0.
\end{equation*}
This is strict when $\lambda_1>0$. For $\lambda_1=0$, it is strict because
$\bm w_0$ uniquely minimizes $R$ and
$\bm w_{\lambda_2}\ne\bm w_0$. Thus,
\begin{equation*}
  R(\bm w_{\lambda_2})>R(\bm w_{\lambda_1}).
\end{equation*}
Finally, Eq.~(\ref{eq:app-frontier-endpoints}) gives
$m(\bm w_\infty)=0$ and $R(\bm w_\infty)=1$. Every finite-$\lambda_1$
kernel has at least two positive coordinates, so
\begin{equation*}
\begin{aligned}
  m(\bm w_\infty)&=0<m(\bm w_{\lambda_1}),\\
  R(\bm w_\infty)&=1>R(\bm w_{\lambda_1}).
\end{aligned}
\end{equation*}
\end{proof}

For the PTQ analysis below, we record the closed-form coordinates of each
fixed-support segment of the \LNWA{} frontier.

\begin{lemma}[Closed-Form Coordinates of the \LNWA{} Frontier]
\label{lem:app-frontier-coordinates}
On an active prefix of size $q\geq2$,
\begin{equation}
  m(\bm w_\lambda)^2=\frac{\bar d_q^2}{(1+D_q\lambda)^2},\qquad
  R(\bm w_\lambda)=\frac1q+
  \frac{D_q\bar d_q^2\lambda^2}{(1+D_q\lambda)^2}.
  \label{eq:app-frontier-coordinates}
\end{equation}
\end{lemma}

\begin{proof}
Eq.~(\ref{eq:app-lnwa-moment}) gives the first identity.  Using
$\sum_{d=0}^{q-1}(\bar d_q-d)=0$ and
$\sum_{d=0}^{q-1}(\bar d_q-d)^2=D_q$ in
Eq.~(\ref{eq:app-lnwa-kernel}) gives
\begin{align*}
  R(\bm w_\lambda)
  &=\sum_{d=0}^{q-1}
  \left[q^{-1}+\lambda m_q(\lambda)(\bar d_q-d)\right]^2\\
  &=\frac1q+\frac{2\lambda m_q(\lambda)}q
  \sum_{d=0}^{q-1}(\bar d_q-d)
  +\lambda^2m_q(\lambda)^2
  \sum_{d=0}^{q-1}(\bar d_q-d)^2\\
  &=\frac1q+D_q\lambda^2m_q(\lambda)^2
  =\frac1q+\frac{D_q\bar d_q^2\lambda^2}
  {(1+D_q\lambda)^2}.
\end{align*}
\end{proof}

Theorem~\ref{thm:app-frontier} derives the \LNWA{} kernel and its
finite-$\lambda$ Pareto path.  Lemma~\ref{lem:app-endpoints-support} then
characterizes the global structure of this path, including its endpoints,
support transitions, and monotonicity in $\lambda$.
Lemma~\ref{lem:app-frontier-coordinates} gives its closed-form lag--noise
coordinates on each fixed-support segment.

\subsection{Classical averaging kernels as \LNWA{} special cases}
\label{app:classical}

This subsection shows that classical averaging kernels can be viewed as
special cases of the \LNWA{} framework.  We first present its two endpoints,
\LAWA{} (uniform averaging) and Stable (the latest weight at the end of the
stable-LR training trajectory), followed by \WMA{} (linear weighting) and
\EMA{} (exponential weighting).

\begin{proposition}[The two endpoints]
\label{prop:app-endpoints}
The $\lambda=0$ endpoint of the extended \LNWA{} path is \LAWA{}, whereas its
$\lambda=\infty$ endpoint is Stable. Their lag--noise coordinates are therefore
\begin{equation*}
  \Phi(\bm{w}^{\LAWA})=\left(\frac{M-1}{2},\frac1M\right),
  \qquad
  \Phi(\bm{w}^{\mathrm{Stable}})=(0,1).
\end{equation*}
\end{proposition}

This proposition is established by Lemma~\ref{lem:app-endpoints-support}. We next locate the linearly weighted rule on the interior of the
frontier.

\begin{proposition}[WMA on the \LNWA{} Frontier]
\label{prop:app-wma}
The \WMA{} kernel $\bm{w}^{\WMA}\in\mathcal W_M$ has coordinates
$(\bm{w}^{\WMA})_d=2(M-d)/[M(M+1)]$.  Its lag--noise coordinates and
corresponding \LNWA{} coefficient are
\begin{equation*}
  m_{\WMA}=\frac{M-1}{3},
  \qquad
  R_{\WMA}=\frac{2(2M+1)}{3M(M+1)},
  \qquad
  \lambda_{\WMA}=\frac{6}{M(M^2-1)}.
\end{equation*}
Consequently, $\bm w^{\WMA}=\bm w_{\lambda_{\WMA}}$.
\end{proposition}

\begin{proof}
Substituting the \WMA{} coefficients into the definitions of $m$ and $R$ gives
\begin{align*}
  m_{\WMA}
  &=\sum_{d=0}^{M-1}d\,(\bm w^{\WMA})_d
   =\frac{2}{M(M+1)}
    \left(M\sum_{d=0}^{M-1}d-\sum_{d=0}^{M-1}d^2\right)
   =\frac{M-1}{3},
\\
  R_{\WMA}
  &=\sum_{d=0}^{M-1}\bigl((\bm w^{\WMA})_d\bigr)^2
   =\frac{4}{M^2(M+1)^2}\sum_{d=0}^{M-1}(M-d)^2
   =\frac{2(2M+1)}{3M(M+1)}.
\end{align*}
On the full-support face, both kernels are affine in $d$:
\begin{align}
  (\bm w_\lambda)_d
  &=\frac1M+\lambda m(\bm w_\lambda)(\bar d_M-d), \notag\\
  (\bm w^{\WMA})_d
  &=\frac{2}{M+1}-\frac{2}{M(M+1)}d.
  \label{eq:app-wma-affine}
\end{align}
Matching the coefficients of $d$ gives
\begin{equation*}
  \lambda_{\WMA}m_{\WMA}=\frac{2}{M(M+1)}
  \qquad\Longrightarrow\qquad
  \lambda_{\WMA}=\frac{6}{M(M^2-1)}.
\end{equation*}
For this coefficient, the intercept of the \LNWA{} kernel is
\begin{equation*}
  \frac1M+\lambda_{\WMA}m_{\WMA}\bar d_M
  =\frac1M+\frac{2}{M(M+1)}\frac{M-1}{2}
  =\frac{2}{M+1},
\end{equation*}
which equals the \WMA{} intercept in
Eq.~(\ref{eq:app-wma-affine}).  Hence
$\bm w^{\WMA}=\bm w_{\lambda_{\WMA}}$.
\end{proof}

The next two propositions give the exact coordinates of \EMA{} and its
\LNWA{} approximation.

\begin{proposition}[Exact Lag-Noise Coordinates of \EMA{}]
\label{prop:app-ema-coordinates}
For $M\geq2$ and $0<\beta<1$, the \EMA{} kernel is
\begin{equation*}
  (\bm{w}^{\EMA})_d=\frac{(1-\beta)\beta^d}{1-\beta^M},
\end{equation*}
with exact coordinates
\begin{align*}
  m_{\EMA}
  &=\frac{\beta-M\beta^M+(M-1)\beta^{M+1}}
  {(1-\beta)(1-\beta^M)}, \\
  R_{\EMA}
  &=\frac{(1-\beta)(1+\beta^M)}
  {(1+\beta)(1-\beta^M)}.
\end{align*}
\end{proposition}

\begin{proof}
The finite geometric series gives
\begin{equation*}
  \sum_{d=0}^{M-1}(\bm w^{\EMA})_d
  =\frac{1-\beta}{1-\beta^M}
   \sum_{d=0}^{M-1}\beta^d
  =1.
\end{equation*}
Substitution into the definitions of $m$ and $R$ yields
\begin{align*}
  m_{\EMA}
  &=\frac{1-\beta}{1-\beta^M}
    \beta\frac{\mathrm d}{\mathrm d\beta}
    \left(\frac{1-\beta^M}{1-\beta}\right)
   =\frac{\beta-M\beta^M+(M-1)\beta^{M+1}}
    {(1-\beta)(1-\beta^M)},
\\
  R_{\EMA}
  &=\left(\frac{1-\beta}{1-\beta^M}\right)^2
    \sum_{d=0}^{M-1}\beta^{2d}
   =\left(\frac{1-\beta}{1-\beta^M}\right)^2
    \frac{1-\beta^{2M}}{1-\beta^2}
   =\frac{(1-\beta)(1+\beta^M)}
    {(1+\beta)(1-\beta^M)}.
\end{align*}
\end{proof}

\begin{proposition}[EMA on the \LNWA{} Frontier]
\label{prop:app-ema}
For $M=2$, \EMA{} is exactly an \LNWA{} kernel for every
$0<\beta<1$, with $\lambda=(1-\beta)/\beta$. For $M>2$, let
$\eta=-\log\beta$. For any $0<\rho\leq1$ such that
$\eta(M-1)\leq\rho$, define
\begin{equation*}
  \lambda_{\EMA}
  :=\frac{\eta}{M\bar d_M-\eta D_M}.
\end{equation*}
Then the approximation residual satisfies
\begin{equation}
  \varepsilon_{\EMA}
  :=\left\|\bm w^{\EMA}-\bm w_{\lambda_{\EMA}}\right\|_\infty
  \leq
  \frac{e^{\rho/2}\eta^2}{2M}
  \left[
    \frac{(M-1)^2}{4}
    +\left(1+\frac{\rho}{2}\right)\frac{M^2-1}{12}
  \right]
  =\mathcal O(\eta^2M).
  \label{eq:app-ema-local-error}
\end{equation}
Thus, \EMA{} is an exact special case of \LNWA{} when $M=2$ and an
approximate special case with the above residual bound when $M>2$.
\end{proposition}

\begin{proof}
For $M=2$, direct substitution gives
\begin{equation*}
  \bm w^{\EMA}
  =\left(\frac1{1+\beta},\frac{\beta}{1+\beta}\right)
  =\left(\frac{\lambda+1}{\lambda+2},\frac1{\lambda+2}\right)
  =\bm w_\lambda,
  \qquad
  \lambda=\frac{1-\beta}{\beta}.
\end{equation*}
For $M>2$, $\eta(M-1)\leq\rho\leq1$ gives
\begin{equation*}
  M\bar d_M-\eta D_M
  =\frac{M(M-1)}{2}
   \left(1-\frac{\eta(M+1)}{6}\right)
  \geq\frac{M(M-1)}{2}
   \left(1-\frac{\rho(M+1)}{6(M-1)}\right)>0,
\end{equation*}
so $\lambda_{\EMA}$ is well defined and positive.  It remains to verify that
the corresponding \LNWA{} kernel has full support.  By
Lemma~\ref{lem:app-endpoints-support}, this is equivalent to
$\lambda_{\EMA}<\lambda_M^{\mathrm{knot}}$.  Substituting the definitions of
$\lambda_{\EMA}$ and $\lambda_M^{\mathrm{knot}}$ gives
\begin{equation*}
  \lambda_{\EMA}<\lambda_M^{\mathrm{knot}}
  \iff \eta(M-1)<2.
\end{equation*}
Because $\eta(M-1)\leq\rho\leq1$, this condition holds.  Substituting
$\lambda_{\EMA}$ into Eq.~(\ref{eq:app-lnwa-moment}) with $q=M$ gives
\begin{equation}
  \lambda_{\EMA}m_M(\lambda_{\EMA})
  =\frac{\lambda_{\EMA}\bar d_M}
  {1+\lambda_{\EMA}D_M}
  =\frac{\eta}{M}.
  \label{eq:app-ema-slope-proof}
\end{equation}
Substituting Eq.~(\ref{eq:app-ema-slope-proof}) into
Eq.~(\ref{eq:app-lnwa-kernel}) with $q=M$ gives
\begin{equation*}
  (\bm w_{\lambda_{\EMA}})_d
  =\frac1M-\frac{\eta}{M}(d-\bar d_M).
\end{equation*}
Define
\begin{equation*}
  r_d:=e^{-\eta(d-\bar d_M)}-
  \left[1-\eta(d-\bar d_M)\right],
  \qquad
  \bar r:=\frac1M\sum_{j=0}^{M-1}r_j.
\end{equation*}
Since $|\eta(d-\bar d_M)|\leq\rho/2$, we can bound each
$r_d$; averaging these bounds then gives the bound for $\bar r$:
\begin{equation*}
  0\leq r_d
  \leq\frac{e^{\rho/2}\eta^2(d-\bar d_M)^2}{2},
  \qquad
  0\leq\bar r
  \leq\frac{e^{\rho/2}\eta^2D_M}{2M}.
\end{equation*}
Using $\sum_{j=0}^{M-1}(j-\bar d_M)=0$, the normalized \EMA{} coefficient is
exactly
\begin{equation*}
  (\bm w^{\EMA})_d
  =\frac{e^{-\eta(d-\bar d_M)}}
  {\sum_{j=0}^{M-1}e^{-\eta(j-\bar d_M)}}
  =\frac{1-\eta(d-\bar d_M)+r_d}{M(1+\bar r)}.
\end{equation*}
Therefore,
\begin{align*}
  \left|(\bm w^{\EMA})_d
  -(\bm w_{\lambda_{\EMA}})_d\right|
  &=\frac{\left|r_d-\bigl[1-\eta(d-\bar d_M)\bigr]\bar r\right|}
  {M(1+\bar r)}\\
  &\leq\frac{e^{\rho/2}\eta^2}{2M}
  \left[(d-\bar d_M)^2+
  \left(1+\frac\rho2\right)\frac{D_M}{M}\right].
\end{align*}
Taking the maximum over $d$ and using
$|d-\bar d_M|\leq(M-1)/2$ and $D_M/M=(M^2-1)/12$ proves
Eq.~(\ref{eq:app-ema-local-error}).
\end{proof}



Together, the preceding results establish the lag-noise decomposition, characterize
the \LNWA{} frontier, and relate classical averaging methods to \LNWA{}.  We
next analyze how weight averaging affects PTQ, leading to our two main
theorems.

\subsection{PTQ tolerance and degradation}
\label{app:ptq}
\label{app:turnover}

We now ask how far an averaging kernel can move toward the newest
weight before quantization performance begins to degrade. We measure
performance by the expected output reconstruction error relative to the
latest latent mean. We first compute the error introduced by the two quantization stages of simplified GPTQ model. This gives an exact expression for the expected reconstruction error. We then study this expression along the \LNWA{} frontier and locate the transition from noise tolerance to degradation.

We use the latent-mean
decomposition and write the mean of an averaged state as
\[
 \bm\mu_T^{\bm w}:=\sum_{d=0}^{M-1}w_d\bm\mu_{T-d}
 =\bm\mu_T-\bm b(\bm w),
 \qquad
 \bm\theta_T^{\bm w}=\bm\mu_T^{\bm w}+\bm\epsilon_T^{\bm w}.
\]
The term $-\bm b(\bm w)$ is the displacement due to temporal lag, and
$\bm\epsilon_T^{\bm w}$ is the retained spatial noise. Both will enter
the output reconstruction error through the channel matrices defined
below.

\begin{assumption}[Bounded Isotropic Spatial Noise]
\label{ass:app-bounded-noise}
In addition to Assumption~\ref{ass:app-noise}, assume
\[
 \bm\Sigma=\sigma^2\bm I_P,\qquad \sigma>0,
 \qquad |[\bm\epsilon_{T-d}]_j|\leq\sigma K_\epsilon
 \quad(0\leq d<M,\ 1\leq j\leq P)
\]
almost surely, for a fixed finite $K_\epsilon$.
\end{assumption}
Thus the averaged spatial noise has
\[
 \E\bm\epsilon_T^{\bm w}=\bm0,
 \qquad
 \Cov(\bm\epsilon_T^{\bm w})=\sigma^2R(\bm w)\bm I_P.
\]

We measure the effect of a parameter perturbation by the reconstruction error it induces on the fixed calibration inputs. For each linear layer $\ell$, let $i$ index its output channels and
$n^{(\ell)}$ denote the number of calibration samples. Let
$\bm X^{(\ell)}\in\mathbb R^{d_{\mathrm{in}}^{(\ell)}\times n^{(\ell)}}$
contain its fixed calibration inputs. The channel vector
$\bm\theta^{(\ell,i)}$ is the transpose of row $i$ of that layer's
parameter matrix. A change $\bm z^{(\ell,i)}$ causes mean squared
reconstruction error
\[
 \frac1{n^{(\ell)}}
 \| (\bm z^{(\ell,i)})^\top\bm X^{(\ell)}\|_2^2
 = (\bm z^{(\ell,i)})^\top\bm G^{(\ell,i)}\bm z^{(\ell,i)},
 \qquad
 \bm G^{(\ell,i)}:=\frac{\bm X^{(\ell)}(\bm X^{(\ell)})^\top}{n^{(\ell)}}.
\]
Let $\mathcal I$ be the set of layer--channel pairs included in the
model, and stack their parameter vectors in a fixed order.
The sum of their reconstruction errors is $\bm z^\top\bm G\bm z$, where
\[
 \bm G=\operatorname{diag}_{(\ell,i)\in\mathcal I}
             \bm G^{(\ell,i)}
 =\begin{pmatrix}
   \bm G^{(\ell_1,i_1)}&\bm0&\cdots\\
   \bm0&\bm G^{(\ell_2,i_2)}&\cdots\\
   \vdots&\vdots&\ddots
  \end{pmatrix},
 \qquad \|\bm z\|_{\bm G}^2:=\bm z^\top\bm G\bm z.
\]
The global matrix is block diagonal across channels. The stacked parameter dimension is
$P=\sum_{(\ell,i)\in\mathcal I}d_{\mathrm{in}}^{(\ell)}$.

Write $Q_b(\bm\theta;s)$ for the model produced by the simplified
GPTQ model at precision $b$ and scale $s$, whose channel operations are
specified below. We evaluate each averaging kernel using the expected reconstruction error
\begin{equation}
 \mathcal E_b(\bm w)
 :=\frac12\E\left[
    \|Q_b(\bm\theta_T^{\bm w};s(\bm w))-\bm\mu_T\|_{\bm G}^2\right].
 \label{eq:app-total-output-error}
\end{equation}
This sums the squared channel-output differences between the quantized
averaged model and the latest latent mean. It includes the temporal lag and spatial noise
already present before quantization, together with the change caused
by quantization. We assume that every channel matrix $\bm G^{(\ell,i)}$ is positive definite.

The simplified GPTQ model quantizes one group of coordinates and uses
the other to compensate for its output error. Partition each channel into
two nonempty groups $\mathrm f$ and $\mathrm r$. Reorder the stacked
coordinates so that all $\mathrm f$ entries precede all $\mathrm r$
entries, and apply the same permutation to every global vector and
matrix. We retain the symbols $\bm\theta$ and $\bm G$ in this order. Then
\[
 \bm\theta=\begin{pmatrix}\bm\theta_{\mathrm f}\\\bm\theta_{\mathrm r}\end{pmatrix},
 \qquad
 \bm G=\begin{pmatrix}
    \bm G_{\mathrm f,\mathrm f}&\bm G_{\mathrm f,\mathrm r}\\
    \bm G_{\mathrm r,\mathrm f}&\bm G_{\mathrm r,\mathrm r}
  \end{pmatrix}.
\]
To simplify the theoretical analysis, we do not model group-wise quantization scales. Instead, we assign the averaged state a single scale, computed before rounding from its RMS magnitude.
\begin{assumption}[Fixed RMS Scale]
\label{ass:app-fixed-rms}
For each averaging kernel $\bm w$, the scale is
\[
 s(\bm w)^2:=\frac1P\E\|\bm\theta_T^{\bm w}\|_2^2
 =\frac{\|\bm\mu_T^{\bm w}\|_2^2}{P}+\sigma^2R(\bm w).
\]
\end{assumption}

For each group $\mathrm g\in{\mathrm f,\mathrm r}$, we use a fixed range multiplier $H_{\mathrm g}>0$, shared across kernels and precisions. Given the RMS scale $s(\bm w)$, the corresponding parameter-space and normalized quantization steps are
$$
\Delta_{b,\mathrm g}(\bm w)
:=\frac{2H_{\mathrm g}s(\bm w)}{2^b-1},
\qquad
\bar\Delta_{b,\mathrm g}
:=\frac{\Delta_{b,\mathrm g}(\bm w)}{s(\bm w)}
=\frac{2H_{\mathrm g}}{2^b-1}.
$$
Accordingly, the $2^b$ reconstruction levels for group $\mathrm g$ cover the range
$$
[-H_{\mathrm g}s(\bm w),\,H_{\mathrm g}s(\bm w)].
$$

Deterministic rounding generally leads to less tractable rounding-error statistics \citep{alvarez2025introducing, zhang2026spfq}.
To simplify the analysis, following \cite{823976}, we instead use
nonsubtractive triangular dither and define
\[
\operatorname{round}(t)
:=
\left\lfloor t+U+V+\frac12\right\rfloor,
\qquad
U,V\overset{\mathrm{i.i.d.}}{\sim}\operatorname{Unif}[-1/2,1/2].
\]
Here $\operatorname{Unif}[a,b]$ denotes the uniform distribution on $[a,b]$,
and $\lfloor\cdot\rfloor$ denotes the floor operator, so
$\operatorname{round}(t)$ returns an integer code.

For an input $\bm z$, a scale $s>0$, and group range $H_{\mathrm g}$,
write $\Delta=s\bar\Delta_{b,\mathrm g}$. Quantization assigns the
integer code and maps it back to parameter space as
\[
 k_j=\operatorname{clip}_{[0,2^b-1]}
       \left(\operatorname{round}
          \left(\frac{z_j}{\Delta}+\frac{2^b-1}{2}\right)\right),
 \qquad
 [\operatorname{quant}_b(\bm z;s,H_{\mathrm g})]_j
 =\Delta\left(k_j-\frac{2^b-1}{2}\right).
\]
Here $\operatorname{clip}$ restricts the integer code to the displayed
range. We next derive the compensation and write the  full-model simplified GPTQ. At the common scale $s$, the first group
has rounding error
\[
 \bm e_{\mathrm f}
 :=\operatorname{quant}_b(\bm\theta_{\mathrm f};s,H_{\mathrm f})-\bm\theta_{\mathrm f}.
\]
The remaining group is adjusted to minimize the output reconstruction
error caused by this first rounding:
\begin{equation}
 \bm c_{\mathrm r}^{\star}
 :=\operatorname*{argmin}_{\bm c_{\mathrm r}}
 \begin{pmatrix}\bm e_{\mathrm f}\\\bm c_{\mathrm r}\end{pmatrix}^{\!\top}
 \bm G
 \begin{pmatrix}\bm e_{\mathrm f}\\\bm c_{\mathrm r}\end{pmatrix}.
 \label{eq:app-gptq-operation}
\end{equation}
Differentiating this quadratic objective gives
\[
 \bm G_{\mathrm r,\mathrm r}\bm c_{\mathrm r}^{\star}
 +\bm G_{\mathrm r,\mathrm f}\bm e_{\mathrm f}=\bm0.
\]
Since $\bm G_{\mathrm r,\mathrm r}\succ0$, the compensation is
\[
 \bm c_{\mathrm r}^{\star}=-\bm K\bm e_{\mathrm f},
 \qquad
 \bm K:=\bm G_{\mathrm r,\mathrm r}^{-1}\bm G_{\mathrm r,\mathrm f}.
\]
The matrix $\bm K$ specifies how first-group rounding changes the
remaining group. Write $B_K:=\|\bm K\|_\infty$, where the matrix infinity
norm is the maximum absolute row sum. The second rounding, applied after
compensation, has error
\[
 \bm\eta_{\mathrm r}
 :=\operatorname{quant}_b(\bm\theta_{\mathrm r}-\bm K\bm e_{\mathrm f};s,H_{\mathrm r})
   -(\bm\theta_{\mathrm r}-\bm K\bm e_{\mathrm f}).
\]
Thus the complete model-level map is
\begin{equation}
 Q_b(\bm\theta;s)
 =\begin{pmatrix}
   \bm\theta_{\mathrm f}+\bm e_{\mathrm f}\\
   \bm\theta_{\mathrm r}-\bm K\bm e_{\mathrm f}+\bm\eta_{\mathrm r}
  \end{pmatrix},
 \qquad
 \bm e_{\mathrm f}
 =\operatorname{quant}_b(\bm\theta_{\mathrm f};s,H_{\mathrm f})-\bm\theta_{\mathrm f}.
 \label{eq:app-global-gptq-map}
\end{equation}

We assume that the quantizer codebook has a sufficiently wide range to represent the parameter values in both quantization stages.
\begin{assumption}[Rounding Inputs within the Quantizer Range]
\label{ass:app-no-clipping}
Let $b_{\min}=2$. For every $\bm w\in\mathcal W_M$ and integer
$b\geq b_{\min}$, the two rounding inputs satisfy
\[
 \frac{\|(\bm\theta_T^{\bm w})_{\mathrm f}\|_\infty}{s(\bm w)}
 <H_{\mathrm f}\left(1-\frac1{2^b-1}\right),
 \qquad
 \frac{\|(\bm\theta_T^{\bm w})_{\mathrm r}-\bm K\bm e_{\mathrm f}\|_\infty}
      {s(\bm w)}
 <H_{\mathrm r}\left(1-\frac1{2^b-1}\right)
\]
almost surely over the training distribution and the first-stage
rounding randomness.
\end{assumption}

\begin{lemma}[A Sufficient Quantization Range]
\label{lem:app-range-sufficient}
Under Assumptions~\ref{ass:app-bounded-noise}
and~\ref{ass:app-fixed-rms}, define
\[
 Z_{\max}:=\sqrt P+K_\epsilon\sqrt M.
\]
If
\[
 H_{\mathrm f}>\frac32Z_{\max},
 \qquad
 H_{\mathrm r}>\frac32\left(Z_{\max}+B_KH_{\mathrm f}\right),
\]
then Assumption~\ref{ass:app-no-clipping} holds.
\end{lemma}
\begin{proof}
For any $\bm w\in\mathcal W_M$, $R(\bm w)\geq1/M$. Hence
$s(\bm w)\geq\sigma/\sqrt M$ and
$\|\bm\mu_T^{\bm w}\|_2/s(\bm w)\leq\sqrt P$. For each coordinate,
\[
\begin{aligned}
 \frac{|[\bm\theta_T^{\bm w}]_j|}{s(\bm w)}
 &\leq\frac{\|\bm\mu_T^{\bm w}\|_2}{s(\bm w)}
       +\frac{\sum_d w_d|[\bm\epsilon_{T-d}]_j|}{s(\bm w)}\\
 &\leq\sqrt P+\frac{\sigma K_\epsilon}{s(\bm w)}\\
 &\leq\sqrt P+K_\epsilon\sqrt M\\
 &=Z_{\max}.
\end{aligned}
\]
Because $b\geq b_{\min}=2$,
$1-1/(2^b-1)\geq2/3$. The first-stage input therefore satisfies
\[
 \frac{\|(\bm\theta_T^{\bm w})_{\mathrm f}\|_\infty}{s(\bm w)}
 \leq Z_{\max}
 <\frac23H_{\mathrm f}
 \leq H_{\mathrm f}\left(1-\frac1{2^b-1}\right).
\]
For any scalar input $z$ satisfying the corresponding margin
\[
 \frac{|z|}{s(\bm w)}
 <H_{\mathrm g}\left(1-\frac1{2^b-1}\right),
\]
we have
\[
 \frac12<\frac{z}{\Delta_{b,\mathrm g}}+\frac{2^b-1}{2}
 <2^b-\frac32.
\]
Since $-1\leq U+V\leq1$, it follows that
\[
 0<\frac{z}{\Delta_{b,\mathrm g}}+\frac{2^b-1}{2}
      +U+V+\frac12<2^b.
\]
Its floor therefore belongs to $\{0,\ldots,2^b-1\}$, so clipping is
inactive. The first-stage margin established above thus ensures that
the perturbation has magnitude at most
$\Delta_{b,\mathrm f}$, and nearest-integer rounding adds at most
$\Delta_{b,\mathrm f}/2$. Thus
\[
 \|\bm e_{\mathrm f}\|_\infty
 \leq\frac32\Delta_{b,\mathrm f}(\bm w),
 \qquad
 \frac{\|\bm K\bm e_{\mathrm f}\|_\infty}{s(\bm w)}
 \leq\frac{3B_KH_{\mathrm f}}{2^b-1}.
\]
The second input consequently satisfies
\[
\begin{aligned}
 \frac{\|(\bm\theta_T^{\bm w})_{\mathrm r}
             -\bm K\bm e_{\mathrm f}\|_\infty}
      {s(\bm w)}
 &\leq Z_{\max}+\frac{3B_KH_{\mathrm f}}{2^b-1}\\
 &\leq Z_{\max}+B_KH_{\mathrm f}\\
 &<\frac23H_{\mathrm r}\\
 &\leq H_{\mathrm r}\left(1-\frac1{2^b-1}\right).
\end{aligned}
\]
This is the second inequality in
Assumption~\ref{ass:app-no-clipping}; the first was established above.
Hence that assumption holds for both quantization stages.
\end{proof}

We next derive the parameter error induced by the simplified GPTQ procedure and its corresponding output reconstruction error.
\[
 \bm C_{\mathrm f}
 :=\bm G_{\mathrm f,\mathrm f}-\bm G_{\mathrm f,\mathrm r}\bm K.
\]
\begin{lemma}[Parameter and Reconstruction Errors of the Simplified GPTQ Model]
\label{lem:app-calibration}
\label{lem:app-reconstruction-error}
The Schur complement $\bm C_{\mathrm f}$ is positive definite. The
quantization-induced global parameter error is
\[
 \bm\delta_b(\bm\theta;s)
 :=Q_b(\bm\theta;s)-\bm\theta
 =\begin{pmatrix}\bm e_{\mathrm f}\\-\bm K\bm e_{\mathrm f}+\bm\eta_{\mathrm r}\end{pmatrix}.
\]
Its output reconstruction error satisfies
\begin{equation}
 \bm\delta_b(\bm\theta;s)^\top\bm G\bm\delta_b(\bm\theta;s)
 =\bm e_{\mathrm f}^\top\bm C_{\mathrm f}\bm e_{\mathrm f}
  +\bm\eta_{\mathrm r}^\top\bm G_{\mathrm r,\mathrm r}\bm\eta_{\mathrm r}.
 \label{eq:app-global-reconstruction-identity}
\end{equation}
\end{lemma}
\begin{proof}
Because $\bm G\succ0$, its principal block
$\bm G_{\mathrm r,\mathrm r}$ is positive definite. Let $J$ be the
global reconstruction objective in Eq.~(\ref{eq:app-gptq-operation}).
For a fixed first-group error, expanding and differentiating it gives
\[
\begin{aligned}
 J(\bm c_{\mathrm r})
 &=\bm e_{\mathrm f}^\top\bm G_{\mathrm f,\mathrm f}\bm e_{\mathrm f}
   +2\bm c_{\mathrm r}^\top\bm G_{\mathrm r,\mathrm f}\bm e_{\mathrm f}
   +\bm c_{\mathrm r}^\top\bm G_{\mathrm r,\mathrm r}\bm c_{\mathrm r},\\
 \nabla_{\bm c_{\mathrm r}}J
 &=2\bm G_{\mathrm r,\mathrm f}\bm e_{\mathrm f}
   +2\bm G_{\mathrm r,\mathrm r}\bm c_{\mathrm r}.
\end{aligned}
\]
Setting this derivative to zero recovers
$\bm c_{\mathrm r}^{\star}=-\bm K\bm e_{\mathrm f}$.

Completing the square decomposes the reconstruction error into the
residual error after optimal compensation and the additional error from
a non-optimal compensation:
\begin{equation}
J(\bm c_{\mathrm r})
=
\underbrace{
\bm e_{\mathrm f}^\top
\bm C_{\mathrm f}
\bm e_{\mathrm f}
}_{\text{residual error}}
+
\underbrace{
(\bm c_{\mathrm r}+\bm K\bm e_{\mathrm f})^\top
\bm G_{\mathrm r,\mathrm r}
(\bm c_{\mathrm r}+\bm K\bm e_{\mathrm f})
}_{\text{compensation error}}.
\label{eq:app-completed-square}
\end{equation}

This identity holds for every first-group rounding error
$\bm e_{\mathrm f}$. At the optimum,
$\bm c_{\mathrm r}^{\star}=-\bm K\bm e_{\mathrm f}$, so the second term
in Eq.~(\ref{eq:app-completed-square}) vanishes and
\[
\bm e_{\mathrm f}^\top\bm C_{\mathrm f}\bm e_{\mathrm f}
=
\begin{pmatrix}
\bm e_{\mathrm f}\\
-\bm K\bm e_{\mathrm f}
\end{pmatrix}^{\!\top}
\bm G
\begin{pmatrix}
\bm e_{\mathrm f}\\
-\bm K\bm e_{\mathrm f}
\end{pmatrix}
>0
\qquad
(\bm e_{\mathrm f}\neq\bm0).
\]
Hence $\bm C_{\mathrm f}\succ0$. The second rounding introduces the additional error
$\bm\eta_{\mathrm r}$, so that
\[
\bm c_{\mathrm r}
=
\bm c_{\mathrm r}^{\star}+\bm\eta_{\mathrm r},
\]
therefore
\[
\bm c_{\mathrm r}+\bm K\bm e_{\mathrm f}
=
\bm\eta_{\mathrm r},
\]
Substituting this expression into Eq.~(\ref{eq:app-completed-square})
gives
\[
J
=
\bm e_{\mathrm f}^\top\bm C_{\mathrm f}\bm e_{\mathrm f}
+
\bm\eta_{\mathrm r}^\top
\bm G_{\mathrm r,\mathrm r}
\bm\eta_{\mathrm r},
\]
which proves Eq.~(\ref{eq:app-global-reconstruction-identity}).
\end{proof}

Lemma~\ref{lem:app-calibration} gives the parameter and reconstruction errors of the simplified GPTQ model.
We use this result to analyze the quantization-induced contribution to
$\mathcal E_b(\bm w)$. Let
\[
\bm a
:=
\bm\theta-\bm\mu_T
=
\begin{pmatrix}
\bm a_{\mathrm f}\\
\bm a_{\mathrm r}
\end{pmatrix}.
\]
we have
\[
Q_b(\bm\theta;s)-\bm\mu_T
=
\bm a+\bm\delta_b(\bm\theta;s),
\]
and the change in reconstruction error introduced by quantization is
\begin{equation}
\begin{aligned}
&\frac{1}{2}
\left\|
\bm a+\bm\delta_b(\bm\theta;s)
\right\|_{\bm G}^{2}
-
\frac{1}{2}
\|\bm a\|_{\bm G}^{2}
\\
&=
\bm a_{\mathrm f}^{\top}
\bm C_{\mathrm f}
\bm e_{\mathrm f}
+
\bm a_{\mathrm f}^{\top}
\bm G_{\mathrm f,\mathrm r}
\bm\eta_{\mathrm r}
+
\bm a_{\mathrm r}^{\top}
\bm G_{\mathrm r,\mathrm r}
\bm\eta_{\mathrm r}
+
\frac{1}{2}
\bm e_{\mathrm f}^{\top}
\bm C_{\mathrm f}
\bm e_{\mathrm f}
+
\frac{1}{2}
\bm\eta_{\mathrm r}^{\top}
\bm G_{\mathrm r,\mathrm r}
\bm\eta_{\mathrm r}.
\end{aligned}
\label{eq:app-gptq-error-change}
\end{equation}

Eq.~(\ref{eq:app-gptq-error-change}) separates the quantization-induced
change into linear and quadratic rounding terms. We first fix the input and take the expectation over the rounding randomness, then take the expectation over the spatial noise. The resulting moments give the exact expected contribution of each term. Partition
$\bm\mu_T=(\bm\mu_{T,\mathrm f}^{\top},\bm\mu_{T,\mathrm r}^{\top})^{\top}$
in the same way as $\bm\theta$. Substituting
$\bm a=\bm\theta-\bm\mu_T$ into Eq.~(\ref{eq:app-gptq-error-change}) and
expanding the quadratic forms gives
\begin{equation}
\begin{aligned}
&\E\!\left[\frac12\|\bm a+\bm\delta_b(\bm\theta;s)\|_{\bm G}^2
             -\frac12\|\bm a\|_{\bm G}^2\right]=\mathcal T_1+\mathcal T_2+\mathcal T_3+\mathcal T_4,
\end{aligned}
\label{eq:app-expected-rounding-terms}
\end{equation}
Here the four terms are
\[
\begin{aligned}
\mathcal T_1
&:=\E[\bm\theta_{\mathrm f}^{\top}\bm C_{\mathrm f}\bm e_{\mathrm f}]
 +\E[\bm\theta_{\mathrm f}^{\top}\bm G_{\mathrm f,\mathrm r}\bm\eta_{\mathrm r}]
 +\E[\bm\theta_{\mathrm r}^{\top}\bm G_{\mathrm r,\mathrm r}\bm\eta_{\mathrm r}],\\
\mathcal T_2
&:=-\bm\mu_{T,\mathrm f}^{\top}\bm C_{\mathrm f}\E[\bm e_{\mathrm f}]
 -\bm\mu_{T,\mathrm f}^{\top}\bm G_{\mathrm f,\mathrm r}\E[\bm\eta_{\mathrm r}]
 -\bm\mu_{T,\mathrm r}^{\top}\bm G_{\mathrm r,\mathrm r}\E[\bm\eta_{\mathrm r}],\\
\mathcal T_3
&:=\frac12\sum_j(\bm C_{\mathrm f})_{jj}\E[(\bm e_{\mathrm f})_j^2]
 +\frac12\sum_j(\bm G_{\mathrm r,\mathrm r})_{jj}
                \E[(\bm\eta_{\mathrm r})_j^2],\\
\mathcal T_4
&:=\frac12\sum_{j\ne k}(\bm C_{\mathrm f})_{jk}
                \E[(\bm e_{\mathrm f})_j(\bm e_{\mathrm f})_k]
 +\frac12\sum_{j\ne k}(\bm G_{\mathrm r,\mathrm r})_{jk}
                \E[(\bm\eta_{\mathrm r})_j(\bm\eta_{\mathrm r})_k].
\end{aligned}
\]
$\mathcal T_1$ captures the coupling between the input parameters and
rounding errors, while $\mathcal T_2$ contains rounding errors multiplied
by the fixed reference $\bm\mu_T$. The remaining terms are quadratic in
the rounding errors: $\mathcal T_3$ collects coordinatewise squares, and
$\mathcal T_4$ collects cross-coordinate products. The following lemma
gives the conditional moments needed to evaluate these terms.

\begin{lemma}[Moments of Triangular-Dither Rounding]
\label{lem:app-grid-moment-bound}
Under Assumption~\ref{ass:app-no-clipping}, let $\bm z$ be either
the first-stage input $\bm\theta_{\mathrm f}$ or the compensated
second-stage input $\bm\theta_{\mathrm r}-\bm K\bm e_{\mathrm f}$,
with its corresponding fixed step $\Delta>0$. Let $e_j$ denote the
parameter-space error introduced by $\operatorname{quant}_b(\cdot)$ in
coordinate $j$. For any integrable coefficient $\phi(\bm z)$, the
conditional moments corresponding to $\mathcal T_1,\ldots,\mathcal T_4$ are
\[
\begin{aligned}
 \E[\phi(\bm z)e_j\mid\bm z]&=0,\\
 \E[e_j\mid\bm z]&=0,\\
 \E[e_j^2\mid\bm z]&=\frac{\Delta^2}{4},\\
 \E[e_je_k\mid\bm z]&=0\qquad(j\ne k).
\end{aligned}
\]
Here the conditional expectation averages over the fresh rounding
draws, which are independent across coordinates.
\end{lemma}
\begin{proof}
Condition on the vector input $\bm z$ and fix a coordinate $j$. Write
$z=z_j$. In parameter space, let
$u,v$ be independent uniform variables on $[-\Delta/2,\Delta/2]$.
Fix $v$ and put $t=z+v$. Let $a$ and $a+\Delta$ be the neighboring
grid points with $a\leq t<a+\Delta$, and set $r=t-a$.
Nearest rounding of $t+u$ returns $a$ with probability
$1-r/\Delta$ and $a+\Delta$ with probability $r/\Delta$. If $q$
denotes the output, direct calculation gives
\begin{align}
 \E[q-t\mid z,v]
 &=(-r)\left(1-\frac r\Delta\right)
   +(\Delta-r)\frac r\Delta=0,
 \label{eq:app-dither-conditional-mean}\\
 \E[(q-t)^2\mid z,v]
 &=r^2\left(1-\frac r\Delta\right)
   +(\Delta-r)^2\frac r\Delta=r(\Delta-r). \notag
\end{align}
The full error is $e=q-z=(q-t)+v$. Taking the expectation over $v$ therefore yields
\[
 \E[e\mid z]=\E[v]=0.
\]
As $v$ ranges uniformly over one grid cell, $r$ also ranges uniformly
over $[0,\Delta)$, irrespective of $z$ and the grid translation.
By Eq.~(\ref{eq:app-dither-conditional-mean}), the mixed term satisfies
\[
 \E[v(q-t)\mid z]
 =\E\!\left[v\,\E[q-t\mid z,v]\mid z\right]=0.
\]
Consequently,
\[
\begin{aligned}
 \E[e^2\mid z]
 &=\E[(q-t)^2\mid z]+\E[v^2]+2\E[v(q-t)\mid z]\\
 &=\E[r(\Delta-r)]+\E[v^2]+0\\
 &=\frac1\Delta\int_0^\Delta r(\Delta-r)\,\mathrm dr
   +\frac1\Delta\int_{-\Delta/2}^{\Delta/2}v^2\,\mathrm dv\\
 &=\frac{\Delta^2}{6}+\frac{\Delta^2}{12}
 =\frac{\Delta^2}{4}.
\end{aligned}
\]
This calculation holds for every coordinate of $\bm z$.
Since $\phi(\bm z)$ is fixed conditional on $\bm z$,
\[
 \E[\phi(\bm z)e_j\mid\bm z]
 =\phi(\bm z)\E[e_j\mid\bm z]=0.
\]
Independent rounding draws across coordinates give
\[
 \E[e_je_k\mid\bm z]
 =\E[e_j\mid\bm z]\E[e_k\mid\bm z]=0
 \qquad(j\ne k).
\]
This proves all four identities in the lemma.
\end{proof}

For the first group, condition on the full input $\bm\theta$.
For the remaining group, condition on both $\bm\theta$ and the
first-stage error $\bm e_{\mathrm f}$; its compensated input is then
fixed and its rounding draws are fresh. With
$d_{\mathrm f}:=\dim\bm\theta_{\mathrm f}$ and
$d_{\mathrm r}:=\dim\bm\theta_{\mathrm r}$,
Lemma~\ref{lem:app-grid-moment-bound} gives
\begin{equation}
\begin{aligned}
 \E[\bm e_{\mathrm f}\mid\bm\theta]&=\bm0,
 &\E[\bm e_{\mathrm f}\bm e_{\mathrm f}^{\top}\mid\bm\theta]
 &=\frac{\Delta_{b,\mathrm f}(\bm w)^2}{4}\bm I_{d_{\mathrm f}},\\
 \E[\bm\eta_{\mathrm r}\mid\bm\theta,\bm e_{\mathrm f}]&=\bm0,
 &\E[\bm\eta_{\mathrm r}\bm\eta_{\mathrm r}^{\top}
            \mid\bm\theta,\bm e_{\mathrm f}]
 &=\frac{\Delta_{b,\mathrm r}(\bm w)^2}{4}\bm I_{d_{\mathrm r}}.
\end{aligned}
 \label{eq:app-rounding-moments}
\end{equation}
Eq.~(\ref{eq:app-rounding-moments}) gives
\[
 \E[\bm\delta_b(\bm\theta;s(\bm w))\mid\bm\theta]
 =\begin{pmatrix}
   \E[\bm e_{\mathrm f}\mid\bm\theta]\\
   -\bm K\E[\bm e_{\mathrm f}\mid\bm\theta]
        +\E[\E[\bm\eta_{\mathrm r}\mid\bm\theta,\bm e_{\mathrm f}]
                 \mid\bm\theta]
  \end{pmatrix}=\bm0.
\]
Thus the mixed contribution in Eq.~(\ref{eq:app-gptq-error-change})
has zero expectation. In the four-term expansion, $ \mathcal T_1=\mathcal T_2=\mathcal T_4=0,$ and
\[
\begin{aligned}
 \mathcal T_3
 &=\frac12\sum_j(\bm C_{\mathrm f})_{jj}
                       \frac{\Delta_{b,\mathrm f}(\bm w)^2}{4}
   +\frac12\sum_j(\bm G_{\mathrm r,\mathrm r})_{jj}
                       \frac{\Delta_{b,\mathrm r}(\bm w)^2}{4}\\
 &=\frac{\Delta_{b,\mathrm f}(\bm w)^2}{8}\tr(\bm C_{\mathrm f})
   +\frac{\Delta_{b,\mathrm r}(\bm w)^2}{8}
                  \tr(\bm G_{\mathrm r,\mathrm r}).
\end{aligned}
\]
This evaluates the quantization contribution. The next lemma combines
it with the temporal lag and spatial noise present before quantization.

\begin{lemma}[Exact Expected Reconstruction Error]
\label{lem:app-quantization-expansion}
Under Assumptions~\ref{ass:app-bounded-noise},
\ref{ass:app-fixed-rms}, and~\ref{ass:app-no-clipping}, the simplified GPTQ model satisfies, for every $\bm w\in\mathcal W_M$
and integer $b\geq b_{\min}$,
\begin{equation}
 \mathcal E_b(\bm w)
 =\mathcal P_0(\bm w)+\gamma_b\mathcal Q(\bm w),
 \label{eq:app-error-expansion}
\end{equation}
where the reconstruction error before quantization is
\[
 \mathcal P_0(\bm w)
 :=\frac12\|\bm\mu_T^{\bm w}-\bm\mu_T\|_{\bm G}^2
       +N_GR(\bm w),\qquad N_G:=\frac{\sigma^2}{2}\tr(\bm G).
\]
The quantization contribution separates into the parameter energy
\[
 \mathcal Q(\bm w):=\E\|\bm\theta_T^{\bm w}\|_2^2
 =Ps(\bm w)^2=\|\bm\mu_T^{\bm w}\|_2^2+P\sigma^2R(\bm w)
\]
and the precision-dependent coefficient
\[
 \gamma_b:=\frac{c_G}{(2^b-1)^2},\qquad
 c_G:=\frac{H_{\mathrm f}^2\tr(\bm C_{\mathrm f})
              +H_{\mathrm r}^2\tr(\bm G_{\mathrm r,\mathrm r})}{2P}>0.
\]
\end{lemma}
\begin{proof}
The latent-mean decomposition and zero-mean spatial noise give
\[
 \frac12\E\|\bm\theta_T^{\bm w}-\bm\mu_T\|_{\bm G}^2
 =\frac12\|\bm\mu_T^{\bm w}-\bm\mu_T\|_{\bm G}^2
       +\frac{\sigma^2}{2}\tr(\bm G)R(\bm w)
 =\mathcal P_0(\bm w).
\]
Set $\bm a=\bm\theta_T^{\bm w}-\bm\mu_T$ and use the conditional
zero mean of $\bm\delta_b$. Expanding the definition of $\mathcal E_b$
and then applying Eq.~(\ref{eq:app-global-reconstruction-identity}) gives
\[
\begin{aligned}
 \mathcal E_b(\bm w)
 &=\frac12\E\|Q_b(\bm\theta_T^{\bm w};s(\bm w))-\bm\mu_T\|_{\bm G}^2\\
 &=\frac12\E\|\bm a+\bm\delta_b(\bm\theta_T^{\bm w};s(\bm w))\|_{\bm G}^2\\
 &=\mathcal P_0(\bm w)
   +\E\!\left[\bm a^\top\bm G
       \E[\bm\delta_b(\bm\theta_T^{\bm w};s(\bm w))\mid\bm\theta_T^{\bm w}]\right]
   +\frac12\E\|\bm\delta_b(\bm\theta_T^{\bm w};s(\bm w))\|_{\bm G}^2\\
 &=\mathcal P_0(\bm w)
   +\frac18\left[\Delta_{b,\mathrm f}(\bm w)^2\tr(\bm C_{\mathrm f})
                +\Delta_{b,\mathrm r}(\bm w)^2\tr(\bm G_{\mathrm r,\mathrm r})\right]\\
 &=\mathcal P_0(\bm w)
   +\frac{H_{\mathrm f}^2\tr(\bm C_{\mathrm f})
              +H_{\mathrm r}^2\tr(\bm G_{\mathrm r,\mathrm r})}
           {2P(2^b-1)^2}\,Ps(\bm w)^2\\
 &=\mathcal P_0(\bm w)+\gamma_b\mathcal Q(\bm w).
\end{aligned}
\]
This proves the identity for every averaging kernel in the stated model.
\end{proof}

Eq.~(\ref{eq:app-error-expansion}) connects the expected reconstruction
error to the lag-induced displacement of the averaged mean and the
retained spatial noise. To make this dependence explicit along the
\LNWA{} frontier, we use mean age to parameterize temporal lag, derive
the corresponding retained noise factor, and track the averaged latent
mean. These relations determine the transition from tolerance to
degradation.

Recall that the \LNWA{} frontier is the family $\bm w_\lambda$.
The \LAWA{} kernel $\bm w^{\LAWA}$ has mean age
$m_0:=m(\bm w^{\LAWA})=(M-1)/2$. We measure the reduction from LAWA by
\begin{equation*}
 y:=m_0-m(\bm w_\lambda),\qquad 0\leq y\leq m_0.
\end{equation*}
Thus $y=0$ corresponds to \LAWA{}, while $y=m_0$ corresponds to the
most recent saved weight. The frontier coordinates in
Lemma~\ref{lem:app-frontier-coordinates} show that mean age decreases
strictly as $\lambda$ increases. Hence each $y<m_0$ determines a unique
$\lambda(y)$. To simplify notation, we write
$\bm w(y):=\bm w_{\lambda(y)}$ for that kernel.
We also write $m(y):=m(\bm w(y))=m_0-y$ and
$R(y):=R(\bm w(y))$ for its mean age and retained noise factor.
For notational consistency, primes below denote derivatives along this
parameterized frontier. In particular,
\[
 \mathcal E_b'(y):=\frac{\mathrm d}{\mathrm dy}\mathcal E_b(\bm w(y)),
 \qquad
 \mathcal E_b''(y):=\frac{\mathrm d^2}{\mathrm dy^2}\mathcal E_b(\bm w(y)),
\]
and analogously for $\mathcal P_0$ and $\mathcal Q$.
After the later change of variable to $x$, primes correspondingly denote
derivatives with respect to $x$.
Thus mean weight age decreases linearly with $y$; the accompanying change
in spatial noise is determined next from the frontier kernel.

On a segment with $q\geq2$ active saved weights, the positive coefficients
have ages $0,\ldots,q-1$. Uniform averaging over these weights has
mean age $\bar d_q=(q-1)/2$. The reduction relative to this local
baseline is
\[
 z_q:=\bar d_q-m(y)=y-\frac{M-q}{2}.
\]
Lemma~\ref{lem:app-frontier-coordinates} then gives the kernel coefficients
\begin{equation}
 w_d(y)=\begin{cases}
 \displaystyle\frac1q+\frac{z_q(\bar d_q-d)}{D_q},&0\leq d<q,\\[4pt]
 0,&d\geq q,
 \end{cases}
 \label{eq:app-turnover-age-frontier}
\end{equation}
where $D_q:=q(q^2-1)/12$. Squaring and summing the coefficients
therefore gives the spatial noise factor
\[
 R(y)=\sum_{d=0}^{q-1}w_d(y)^2=\frac1q+\frac{z_q^2}{D_q},
\]
whose derivatives are
\[
 R'(y)=\frac{2z_q}{D_q},\qquad R''(y)=\frac2{D_q}.
\]
Finally, the original frontier parameter is recovered from
$z_q=D_q\lambda m(y)$ as
\begin{equation}
 \lambda(y)=\frac{z_q}{D_q(m_0-y)}.
 \label{eq:app-turnover-age-parameter}
\end{equation}

The endpoint is $y=m_0$, with
$\lambda(m_0)=\infty$. Write $\bm\mu(y):=\bm\mu_T^{\bm w(y)}$ for the
latent mean of this averaged state. It is affine within each support
segment because the averaging coefficients are affine there.
Using the latent-mean increments defined in Appendix~\ref{app:setup},
its derivative is

\begin{equation*}
\begin{aligned}
 \bm a_q:=\bm\mu'(y)
 &=\frac1{D_q}\sum_{d=0}^{q-1}(\bar d_q-d)\bm\mu_{T-d}\\
 &=\frac1{D_q}\sum_{d=0}^{q-1}(\bar d_q-d)
   \left(\bm\mu_{T-q+1}+\sum_{j=d+1}^{q-1}\bm v_{T-j+1}\right)\\
 &=\frac1{D_q}\sum_{j=1}^{q-1}
   \left(\sum_{d=0}^{j-1}(\bar d_q-d)\right)\bm v_{T-j+1}\\
 &=\frac1{D_q}\sum_{j=1}^{q-1}
   \left(j\bar d_q-\frac{j(j-1)}2\right)\bm v_{T-j+1}\\
 &=\sum_{j=1}^{q-1}\frac{j(q-j)}{2D_q}\bm v_{T-j+1}.
\end{aligned}
\end{equation*}

The coefficients are nonnegative and satisfy
\[
 \sum_{j=1}^{q-1}\frac{j(q-j)}{2D_q}=1.
\]
Thus $\bm a_q$ is a convex combination of the retained mean increments;
its norm is at most
$\bar v:=\max_{1\leq j<M}\|\bm v_{T-j+1}\|_2$.
We use this increment bound below to control changes in the averaged
mean at support boundaries.

On a fixed-support segment, $\bm\mu'(y)=\bm a_q$ is constant and
$R''(y)=2/D_q$. The two components have second derivatives
\[
\begin{aligned}
 \mathcal P_0''(y)
 &=\bm a_q^\top\bm G\bm a_q+\frac{2N_G}{D_q}>0,\\
 \mathcal Q''(y)
 &=2\|\bm a_q\|_2^2+\frac{2P\sigma^2}{D_q}>0.
\end{aligned}
\]
Consequently,
\begin{equation}
\begin{aligned}
 \mathcal E_b''(y)
 &=\mathcal P_0''(y)+\gamma_b\mathcal Q''(y)\\
 &=\bm a_q^\top\bm G\bm a_q
   +2\gamma_b\|\bm a_q\|_2^2
   +\frac{2N_G+2\gamma_bP\sigma^2}{D_q}>0.
\end{aligned}
 \label{eq:app-turnover-segment-curvature}
\end{equation}
These results describe movement within one support segment.
When the oldest positive averaging coefficient reaches zero, the
active set changes and $\bm\mu'(y)$ can jump. We next relate these
changes to the evolution of the latent mean over the retained window.

Across nearby saved weights in a stable training phase, latent-mean
increments are small and change gradually with normalized training
time. The following regularity condition bounds their magnitudes and
adjacent variation.
\begin{assumption}[Local Regularity of the Latent Mean]
\label{ass:app-window-regularity}
Measure training time as $t/T$. The $M$ consecutive retained weights
span the normalized length $\ell:=(M-1)/T$, with spacing
$h:=1/T=\ell/(M-1)$. The forward latent increments satisfy
\[
 \|\bm v_{T-j+1}\|_2\leq Vh
 \quad(1\leq j<M),\qquad
 \|\bm v_{T-j+1}-\bm v_{T-j+2}\|_2\leq Ah^2
 \quad(2\leq j<M),
\]
where $V,A\geq0$ are uniform over the considered late-training windows
and independent of the spacing $h$.
\end{assumption}

Summing the $d$ increments between $\bm\mu_{T-d}$ and $\bm\mu_T$ gives
\[
 \|\bm\mu_{T-d}-\bm\mu_T\|_2\leq Vdh.
\]
Denote the mean under \LAWA{} by
$\bar{\bm\mu}:=M^{-1}\sum_{d=0}^{M-1}\bm\mu_{T-d}$, and let
$D_\mu:=\max_d\|\bm\mu_{T-d}-\bm\mu_T\|_2$.
Averaging the preceding bound yields
\[
 \|\bar{\bm\mu}-\bm\mu_T\|_2
 \leq\frac1M\sum_{d=0}^{M-1}Vdh
 =\frac{V\ell}{2}.
\]
The second bound in Assumption~\ref{ass:app-window-regularity}
verifies Assumption~\ref{ass:app-progress} with
$\kappa=Ah^2$. The first bound controls $D_\mu$ and the
\LAWA{} mean displacement. Collecting these consequences gives
\begin{equation}
 \kappa=\frac{A\ell^2}{(M-1)^2},
 \qquad
 D_\mu\leq V\ell,
 \qquad
 \|\bar{\bm\mu}-\bm\mu_T\|_2\leq\frac{V\ell}{2}.
 \label{eq:app-turnover-window-rates}
\end{equation}

At a support transition, both the averaging kernel and the averaged latent mean remain continuous, although the mean slope may change. The following lemma bounds that change and the cumulative effect of the transitions.
\begin{lemma}[Slope Changes at Support Boundaries]
\label{lem:app-turnover-support-jumps}
Under Assumption~\ref{ass:app-progress}, let $y_j$,
$1\leq j\leq M-2$, be the internal support boundaries in
increasing order, and write
$\Delta\bm\mu'(y_j):=\bm\mu'(y_j+)-\bm\mu'(y_j-)$.
At every such boundary,
\[
 \|\Delta\bm\mu'(y_j)\|_2\leq\frac\kappa2,
 \qquad R'(y_j-)=R'(y_j+).
\]
For $0<y\leq m_0$, their positions satisfy
\begin{equation}
 \frac{\sum_{y_j\leq y}y_j}{y^2}\leq H_M,
\label{eq:app-turnover-knot-bound}
\end{equation}
where
\[
 H_2=0,\qquad
 H_M=\frac{12(M-2)(M-1)}{(3M-5)^2}\quad(M\geq3).
\]
\end{lemma}
\begin{proof}
At a transition from $q+1$ to $q$ active weights, subtracting the two
coefficient derivatives in Eq.~(\ref{eq:app-turnover-age-frontier}) gives

\begin{equation*}
 \Delta w_d'
 =\frac{12(q-1-3d)}{q(q-1)(q+1)(q+2)}\quad(0\leq d<q),
 \qquad
 \Delta w_q'=\frac6{(q+1)(q+2)}.
\end{equation*}

Set
\[
 K_{qj}:=\frac{6j(j+1)(q-j)}{q(q-1)(q+1)(q+2)}
 \quad(1\leq j<q),\qquad K_{q0}=K_{qq}=0.
\]
The corresponding change in the mean slope is
\[
 \Delta\bm\mu'(y_j)
 =\sum_{r=1}^{q-1}K_{qr}
  (\bm\mu_{T-r-1}-2\bm\mu_{T-r}+\bm\mu_{T-r+1}).
\]
Expanding the second differences reproduces the coefficient changes:
the endpoint coefficients are $K_{q1}$ and $K_{q,q-1}$, and an interior
coefficient is $K_{q,j-1}-2K_{qj}+K_{q,j+1}$.
Summing the coefficients gives
\[
 \sum_{j=1}^{q-1}K_{qj}
 =\frac{6}{q(q-1)(q+1)(q+2)}
   \sum_{j=1}^{q-1}j(j+1)(q-j)
 =\frac12.
\]
Assumption~\ref{ass:app-progress} therefore gives
\[
 \|\Delta\bm\mu'(y_j)\|_2
 \leq\kappa\sum_{j=1}^{q-1}K_{qj}=\frac\kappa2.
\]

The oldest coefficient becomes zero at
$z_{q+1}=(q+2)/6$. Substituting this in the definition of $z_{q+1}$
gives $y_j=(M+2j-1)/6$ for $1\leq j\leq M-2$. At the same transition,
$z_q=(q-1)/6$, and hence
\[
 R'(y_j-)=\frac{2z_{q+1}}{D_{q+1}}
 =\frac{4}{q(q+1)}
 =\frac{2z_q}{D_q}=R'(y_j+).
\]

Between consecutive boundaries the numerator in
Eq.~(\ref{eq:app-turnover-knot-bound}) is constant, and its ratio to
$y^2$ decreases.  At the $j$th boundary the ratio is

\begin{equation*}
 \frac{\sum_{i=1}^j y_i}{y_j^2}
 =\frac{6j(M+j)}{(M+2j-1)^2}.
\end{equation*}

This expression increases for $1\leq j\leq M-2$, treating $j$ as a
real variable, the sign of its derivative is the sign of
$M(M-1)-2j>0$.  Its last value is $H_M$.
There are no internal boundaries when $M=2$.

\end{proof}

For each precision $b$, write the curvature on a segment with $q$
active weights as
\[
 a_{q,b}:=\bm a_q^\top\bm G\bm a_q
          +2\gamma_b\|\bm a_q\|_2^2
          +\frac{2N_G+2\gamma_bP\sigma^2}{D_q},
 \qquad 2\leq q\leq M,
\]
as given by Eq.~(\ref{eq:app-turnover-segment-curvature}), and define
\[
 a_{\min,b}:=\min_{2\leq q\leq M}a_{q,b},
 \qquad
 a_{\max,b}:=\max_{2\leq q\leq M}a_{q,b}.
\]
In particular,
\[
 a_{\min,b}\geq\frac{2N_G+2\gamma_bP\sigma^2}{D_M}>0.
\]
Write $A_\mu:=\max_{0\leq d<M}\|\bm\mu_{T-d}\|_2$ and
$g:=\|\bm G\|_2$. The gradient with respect to the averaged mean
satisfies
\[
 \|\nabla_{\bm\mu}\mathcal E_b\|_2
 =\|\bm G(\bm\mu-\bm\mu_T)+2\gamma_b\bm\mu\|_2
 \leq gD_\mu+2\gamma_bA_\mu
 =:G_{\mathcal E,b}.
\]
On each fixed-support segment, differentiation with respect to $y$ gives
\[
 \mathcal E_b'(y)
 =\bigl[\bm G(\bm\mu(y)-\bm\mu_T)
       +2\gamma_b\bm\mu(y)\bigr]^\top\bm\mu'(y)
  +(N_G+\gamma_bP\sigma^2)R'(y).
\]
At a transition point $y_j$, $\bm\mu(y)$ is continuous and
Lemma~\ref{lem:app-turnover-support-jumps} gives
$R'(y_j-)=R'(y_j+)$. Thus the $R'$ term has no jump, and
\begin{equation}
\begin{aligned}
 |\Delta\mathcal E_b'(y_j)|
 &=\left|\bigl[\bm G(\bm\mu(y_j)-\bm\mu_T)
                 +2\gamma_b\bm\mu(y_j)\bigr]^\top
                    \Delta\bm\mu'(y_j)\right|\\
 &\leq\frac{G_{\mathcal E,b}\kappa}{2}.
\end{aligned}
 \label{eq:app-turnover-risk-jump}
\end{equation}
Thus the adjacent-increment bound controls the effect of changing the
active set on the slope of the expected reconstruction error.

We now locate where the expected reconstruction error along the frontier changes
from tolerance to degradation. Define its average change from the
starting kernel by
\[
 F_b(y):=
 \frac{\mathcal E_b(\bm w(y))-\mathcal E_b(\bm w(0))}{y}
 \quad(y>0),
 \qquad
 F_b(0):=\mathcal E_b'(0+).
\]
Thus $F_b(y)<0$ marks a kernel with lower error than the starting
kernel, while $F_b(y)>0$ marks degradation. To determine whether this
sign can change more than once, we bound the slope of $F_b$ across
both support segments and their boundaries.

For a point $t$ away from the support boundaries, its slope equals the
initial right slope plus the curvature accumulated within the support
segments and the slope jumps at all preceding boundaries:
\[
 \mathcal E_b'(t)
 =
 \mathcal E_b'(0+)
 +\int_0^t
    \mathcal E_b''(s)\,\mathrm ds
 +\sum_{y_j<t}\Delta\mathcal E_b'(y_j).
\]
Integrating this slope over $t\in[0,y]$ and exchanging the order of
integration gives
\[
\begin{aligned}
 &\mathcal E_b(\bm w(y))-\mathcal E_b(\bm w(0))\\
 &=
 y\mathcal E_b'(0+)
 +\int_0^y\int_0^t
    \mathcal E_b''(s)
    \,\mathrm ds\,\mathrm dt
 +\int_0^y\sum_{y_j<t}\Delta\mathcal E_b'(y_j)\,\mathrm dt\\
 &=
 y\mathcal E_b'(0+)
 +\int_0^y(y-t)
    \mathcal E_b''(t)\,\mathrm dt
 +\sum_{y_j<y}(y-y_j)\Delta\mathcal E_b'(y_j).
\end{aligned}
\]
Each curvature value at $t$ contributes over the remaining interval
of length $y-t$, while a slope jump at $y_j$ contributes over the
remaining interval of length $y-y_j$.

For $y>0$ away from a support boundary, the quotient rule gives
\[
 y^2F_b'(y)
 =
 y\mathcal E_b'(y)
 -\bigl[\mathcal E_b(\bm w(y))-\mathcal E_b(\bm w(0))\bigr].
\]
Substituting the two preceding identities cancels the initial-slope
terms. The remaining terms simplify as
\[
\begin{aligned}
 y^2F_b'(y)
 &=
 y\int_0^y
    \mathcal E_b''(t)\,\mathrm dt
 -\int_0^y(y-t)
    \mathcal E_b''(t)\,\mathrm dt\\
 &\quad
 +y\sum_{y_j<y}\Delta\mathcal E_b'(y_j)
 -\sum_{y_j<y}(y-y_j)\Delta\mathcal E_b'(y_j)\\
 &=
 \int_0^y t\,
    \mathcal E_b''(t)\,\mathrm dt
 +\sum_{y_j<y}y_j\Delta\mathcal E_b'(y_j).
\end{aligned}
\]
The segment curvatures and Eqs.~(\ref{eq:app-turnover-knot-bound})
and~(\ref{eq:app-turnover-risk-jump}) bound its two terms by
\[
 \frac{a_{\min,b}y^2}{2}
 \leq\int_0^y t\,
      \mathcal E_b''(t)\,\mathrm dt
 \leq\frac{a_{\max,b}y^2}{2},
 \qquad
 \left|\sum_{y_j<y}y_j\Delta\mathcal E_b'(y_j)\right|
 \leq\frac{H_MG_{\mathcal E,b}\kappa y^2}{2}.
\]
Hence, with
\[
 c_b:=a_{\min,b}-H_MG_{\mathcal E,b}\kappa,
 \qquad
 C_b:=a_{\max,b}+H_MG_{\mathcal E,b}\kappa,
\]
we obtain
\begin{equation}
 \frac{c_b}{2}\leq F_b'(y)\leq\frac{C_b}{2}.
 \label{eq:app-turnover-quotient-bounds}
\end{equation}
The bounds extend to one-sided derivatives at support boundaries.
It remains to establish $c_b>0$ and show that the error eventually
exceeds its starting value. Eq.~(\ref{eq:app-turnover-window-rates})
gives $D_\mu\leq V\ell$; the triangle inequality then yields
\[
 A_\mu=\max_{0\leq d<M}\|\bm\mu_{T-d}\|_2
 \leq\|\bm\mu_T\|_2+D_\mu
 \leq A_*:=\|\bm\mu_T\|_2+V\ell.
\]
Substituting these bounds gives
\begin{equation}
\begin{aligned}
 c_b
 &\geq\frac{2N_G+2\gamma_bP\sigma^2}{D_M}
       -H_M(gD_\mu+2\gamma_bA_\mu)\kappa\\
 &\geq\left[\frac{2N_G}{D_M}
       -\frac{H_MgVA}{(M-1)^2}\ell^3\right]
   +2\gamma_b\left[\frac{P\sigma^2}{D_M}
       -\frac{H_MA_*A}{(M-1)^2}\ell^2\right].
\end{aligned}
 \label{eq:app-turnover-main-curvature-bound}
\end{equation}
The possible loss from boundary jumps decreases with the window
length, whereas both positive curvature terms remain bounded below.

The same trajectory bounds also determine the comparison between the
two endpoints. The corresponding means and noise factors are
\[
 (\bm\mu(0),R(0))=(\bar{\bm\mu},1/M),
 \qquad
 (\bm\mu(m_0),R(m_0))=(\bm\mu_T,1).
\]
Here $y=0$ gives \LAWA{}, while $y=m_0$ selects the most recent saved
weight.
Write $c_M:=1-1/M$. The mean-norm difference satisfies
\[
\begin{aligned}
 \left|\|\bm\mu_T\|_2^2-\|\bar{\bm\mu}\|_2^2\right|
 &=\left|(\bm\mu_T+\bar{\bm\mu})^\top
                  (\bm\mu_T-\bar{\bm\mu})\right|\\
 &\leq2A_*\|\bm\mu_T-\bar{\bm\mu}\|_2
 \leq A_*V\ell.
\end{aligned}
\]
Substituting these endpoint means and noise factors into $\mathcal E_b$
therefore yields
\begin{equation}
\begin{aligned}
 &\mathcal E_b(\bm w(m_0))-\mathcal E_b(\bm w(0))\\
 &\quad=N_Gc_M-\frac12\|\bar{\bm\mu}-\bm\mu_T\|_{\bm G}^2
   +\gamma_b\left(\|\bm\mu_T\|_2^2-\|\bar{\bm\mu}\|_2^2
                       +P\sigma^2c_M\right)\\
 &\quad\geq
 \left[N_Gc_M-\frac{gV^2\ell^2}{8}\right]
 +\gamma_b\left[P\sigma^2c_M-A_*V\ell\right].
\end{aligned}
 \label{eq:app-turnover-endpoint-rate}
\end{equation}
We retain a fixed number of consecutive saved weights as $T$ increases.
Their normalized time span therefore tends to zero.

\begin{assumption}[Fixed Retained Count under Time Refinement]
\label{ass:app-short-window}
Consider a family of training trajectories indexed by $T\to\infty$
on the normalized time interval $[0,1]$. The number $M\geq2$ of
consecutive retained weights is fixed across the family.
The model dimension $P$, calibration matrix $\bm G$, and spatial-noise
level $\sigma>0$ are fixed across the family. The constants $V,A$ in
Assumption~\ref{ass:app-window-regularity} are uniform in $T$, and
$\sup_T\|\bm\mu_T\|_2<\infty$. Thus
$\ell=(M-1)/T\to0$.
\end{assumption}

\begin{lemma}[Finite-Window Criterion and Its Consequences]
\label{lem:app-short-window-main}
Under Assumptions~\ref{ass:app-bounded-noise}--\ref{ass:app-window-regularity}, let
$A_*:=\|\bm\mu_T\|_2+V\ell$ and $c_M:=1-1/M$.
If the normalized window length satisfies
\[
\begin{aligned}
 0<\ell\leq\min\biggl\{
 \left(\frac{N_G(M-1)^2}{H_MgVA D_M}\right)^{1/3},
  \left(\frac{P\sigma^2(M-1)^2}{2H_MA_*A D_M}\right)^{1/2}, \frac{2\sqrt{N_Gc_M}}{\sqrt g\,V},
  \frac{P\sigma^2c_M}{2A_*V}
 \biggr\},
\end{aligned}
\]
where a term with zero denominator is $+\infty$, then the expected
reconstruction error obeys
\begin{equation}
\begin{aligned}
 c_b&\geq\frac{N_G+\gamma_bP\sigma^2}{D_M}>0,\\
 \mathcal E_b(\bm w(m_0))-\mathcal E_b(\bm w(0))
 &\geq\frac{(N_G+\gamma_bP\sigma^2)c_M}{2}>0.
\end{aligned}
 \label{eq:app-turnover-window-margins}
\end{equation}
Under Assumption~\ref{ass:app-short-window}, there is a $T_0$,
independent of $b$, such that
the displayed short-window criterion and hence
Eq.~(\ref{eq:app-turnover-window-margins}) hold for every $T\geq T_0$.
\end{lemma}
\begin{proof}
The displayed condition is equivalent to the four inequalities
\[
\begin{gathered}
 \frac{H_MgVA\ell^3}{(M-1)^2}\leq\frac{N_G}{D_M},
 \qquad
 \frac{H_MA_*A\ell^2}{(M-1)^2}\leq\frac{P\sigma^2}{2D_M},\\
 \frac{gV^2\ell^2}{8}\leq\frac{N_Gc_M}{2},
 \qquad
 A_*V\ell\leq\frac{P\sigma^2c_M}{2}.
\end{gathered}
\]
The first row inserted into Eq.~(\ref{eq:app-turnover-main-curvature-bound})
gives
\[
 c_b\geq\frac{N_G}{D_M}
          +\gamma_b\frac{P\sigma^2}{D_M}
       =\frac{N_G+\gamma_bP\sigma^2}{D_M}>0.
\]
The second row inserted into Eq.~(\ref{eq:app-turnover-endpoint-rate})
gives
\[
 \mathcal E_b(\bm w(m_0))-\mathcal E_b(\bm w(0))
 \geq\frac{N_Gc_M}{2}
       +\gamma_b\frac{P\sigma^2c_M}{2}>0.
\]
These are the two bounds in Eq.~(\ref{eq:app-turnover-window-margins}).

To show that the four inequalities eventually hold, set
$B_\mu:=\sup_T\|\bm\mu_T\|_2<\infty$. Since $M$ is fixed and
$\ell=(M-1)/T\to0$, we have $A_*\leq B_\mu+V$ for all sufficiently
large $T$. The left sides of the four inequalities therefore tend to
zero, while their right sides are fixed and positive because
$N_G=\sigma^2\tr(\bm G)/2>0$, $P\sigma^2>0$, and $c_M>0$.
Hence all four hold for the same $T_0$. They contain no $b$, so this
$T_0$ is independent of precision.
\end{proof}

Eq.~(\ref{eq:app-turnover-quotient-bounds}) bounds the change in
$F_b$ along the frontier. Lemma~\ref{lem:app-short-window-main} makes
its lower slope bound positive and places the final frontier point
above the starting error. These facts locate the transition from
tolerance to degradation. Normalize the age reduction as
$x:=y/m_0=1-m(y)/m_0$, with $x=0$ at the starting kernel and $x=1$
at the most recent saved weight.
\begin{theorem}[Finite Noise Tolerance and Degradation]
\label{thm:app-turnover}
Under Assumptions~\ref{ass:app-bounded-noise}--\ref{ass:app-short-window},
for sufficiently large $T$, consider the
expected reconstruction error $\mathcal E_b=\mathcal P_0+\gamma_b\mathcal Q$
at a fixed precision $b\geq b_{\min}$. Write
$\mathcal E_b(x):=\mathcal E_b(\bm w(m_0x))$ and set
$k_b=m_0^2c_b>0$, $K_b=m_0^2C_b$.
There is a unique threshold $\tau_b\in[0,1)$. It is zero if
$\mathcal E_b'(0+)\geq0$; otherwise, it is the unique positive solution
of $\mathcal E_b(\tau_b)=\mathcal E_b(0)$.
It satisfies the explicit upper bound
\[
 0\leq\tau_b\leq\frac{D_MV}{m_0T}
 \left(\frac{gV\ell}{N_G}+\frac{2A_*}{P\sigma^2}\right).
\]
\textbf{(i) Tolerance:} for $0<x\leq\tau_b$,
\begin{equation}
 -\frac{K_b}{2}x(\tau_b-x)
 \leq\mathcal E_b(x)-\mathcal E_b(0)
 \leq-\frac{k_b}{2}x(\tau_b-x).
 \label{eq:app-turnover-actual-tolerance}
\end{equation}
\textbf{(ii) Degradation:} for $\tau_b<x\leq1$,
\begin{equation}
 \mathcal E_b(x)-\mathcal E_b(0)
 \geq\frac{k_b}{2}x(x-\tau_b)>0.
 \label{eq:app-turnover-actual-degradation}
\end{equation}
Moreover, $\mathcal E_b$ is strictly increasing beyond $\tau_b$.
\end{theorem}
\begin{proof}
Recall that $F_b(y)$ is the error change from $\bm w(0)$ divided by $y$,
with $F_b(0)$ equal to the initial right slope in $y$. Since $y=m_0x$,
$\mathcal E_b'(0+)=m_0F_b(0)$, so the two initial slopes have the same
sign.

The frontier kernels
make $\mathcal E_b(\bm w(y))$ continuous and piecewise quadratic, so
$F_b$ is continuous. Eq.~(\ref{eq:app-turnover-quotient-bounds}) and
Lemma~\ref{lem:app-short-window-main} give $F_b'(y)\geq c_b/2>0$
on each support segment, including the one-sided bounds at its
boundaries. Thus $F_b$ is strictly increasing. The endpoint bound in
that lemma gives $F_b(m_0)>0$.

If $F_b(0)<0$, there is a unique zero
$y_b\in(0,m_0)$; set $\tau_b=y_b/m_0$. If $F_b(0)\geq0$, set
$y_b=\tau_b=0$. This establishes the threshold claimed in the theorem.

We next bound its location. At $y=0$, $R'(0)=0$ and
$\bm\mu'(0+)=\bm a_M$, so
\[
 F_b(0)=\bigl[\bm G(\bar{\bm\mu}-\bm\mu_T)
             +2\gamma_b\bar{\bm\mu}\bigr]^\top\bm a_M,
 \qquad
 |F_b(0)|\leq(gD_\mu+2\gamma_bA_\mu)\|\bm a_M\|_2.
\]
The convex-combination representation of $\bm a_M$ gives
$\|\bm a_M\|_2\leq V/T$; also $D_\mu\leq V\ell$ and
$A_\mu\leq A_*$. For $y_b>0$,
\[
 -F_b(0)=F_b(y_b)-F_b(0)
 =\int_0^{y_b}F_b'(t)\,\mathrm dt
 \geq\int_0^{y_b}\frac{c_b}{2}\,\mathrm dt
 =\frac{c_by_b}{2}.
\]
Using $c_b\geq(N_G+\gamma_bP\sigma^2)/D_M$ from
Lemma~\ref{lem:app-short-window-main} then gives
\[
\begin{aligned}
 \tau_b
 &\leq\frac{2|F_b(0)|}{m_0c_b}\\
 &\leq\frac{2D_M\|\bm a_M\|_2}{m_0}
       \frac{gD_\mu+2\gamma_bA_\mu}
            {N_G+\gamma_bP\sigma^2}\\
 &\leq\frac{2D_MV}{m_0T}
       \left(\frac{gV\ell}{N_G}
                    +\frac{2A_*}{P\sigma^2}\right).
\end{aligned}
\]
Assumption~\ref{ass:app-short-window} fixes $M$ and bounds $A_*$,
so for sufficiently large $T$ this preliminary bound is smaller than
$(M+1)/[3(M-1)]$, the end of the first support segment in $x$.
On that segment,
$\mathcal E_b(\bm w(y))-\mathcal E_b(\bm w(0))
=F_b(0)y+a_{M,b}y^2/2$, hence
\[
 0=F_b(y_b)=F_b(0)+\frac{a_{M,b}}{2}y_b.
\]
Using $a_{M,b}\geq2(N_G+\gamma_bP\sigma^2)/D_M$ now gives the upper bound:
\[
\begin{aligned}
 \tau_b
 &=\frac{-2F_b(0)}{m_0a_{M,b}}\\
 &\leq\frac{D_M\|\bm a_M\|_2}{m_0}
       \frac{gD_\mu+2\gamma_bA_\mu}
            {N_G+\gamma_bP\sigma^2}\\
 &\leq\frac{D_MV}{m_0T}
       \left(\frac{gV\ell}{N_G}
                    +\frac{2A_*}{P\sigma^2}\right).
\end{aligned}
\]
When $y_b=0$, the upper bound is immediate.

For $y_b>0$ and $0\leq y\leq y_b$, Eq.~(\ref{eq:app-turnover-quotient-bounds})
and $F_b(y_b)=0$ give
\[
 -\frac{C_b}{2}(y_b-y)
 \leq-\int_y^{y_b}F_b'(t)\,\mathrm dt
 =F_b(y)-F_b(y_b)=F_b(y)
 \leq-\frac{c_b}{2}(y_b-y).
\]
Multiplying by $y$ and setting $y=m_0x$, $y_b=m_0\tau_b$ proves
Eq.~(\ref{eq:app-turnover-actual-tolerance}). If $y_b=0$, that bound
concerns only $x=0$. For $y>y_b$,
\[
 F_b(y)=F_b(y_b)+\int_{y_b}^{y}F_b'(t)\,\mathrm dt
 \geq F_b(y_b)+\frac{c_b}{2}(y-y_b)
 \geq\frac{c_b}{2}(y-y_b);
\]
the last inequality also holds when $y_b=0$, since then $F_b(0)\geq0$.
Therefore
\[
 \mathcal E_b(\bm w(y))-\mathcal E_b(\bm w(0))
 =yF_b(y)\geq\frac{c_b}{2}y(y-y_b),
\]
which proves Eq.~(\ref{eq:app-turnover-actual-degradation}).

Finally, for $y>y_b$ away from support boundaries,
\[
 \frac{\mathrm d}{\mathrm dy}\mathcal E_b(\bm w(y))
 =F_b(y)+yF_b'(y)
 \geq\frac{c_b}{2}(2y-y_b)>0.
\]
The same inequality holds for both one-sided derivatives at a
boundary. Multiplying by $m_0$ yields
$\mathrm d\mathcal E_b(x)/\mathrm dx\geq k_b(2x-\tau_b)/2>0$.
Continuity across the finitely many boundaries then gives strict
increase throughout $(\tau_b,1]$.
\end{proof}

The corresponding threshold in the original frontier parameter is
$\lambda_{{\rm tol},b}=\lambda(m_0\tau_b)$, obtained from
Eq.~(\ref{eq:app-turnover-age-parameter}).
The next subsection characterizes precision-dependent averaging preferences.

\subsection{Precision-dependent averaging preference}
\label{app:bit-selection}

The first main theorem compares the expected reconstruction error
$\mathcal E_b$ with its value at \LAWA{} for a fixed precision.
We now determine how quantization precision shifts the preferred
kernel and identify the precision at which its lag--noise preference
changes.

Recall that $x=1-m/m_0$ increases from uniform averaging toward the
most recent saved weight. Increasing $x$ reduces mean weight age and retains more
spatial noise. Write $\bm\mu(x)=\bm\mu_T^{\bm w(m_0x)}$ and
$R(x)=R(\bm w(m_0x))$. For any frontier function in this subsection,
an argument $x$ means evaluation at $\bm w(m_0x)$; in particular,
$\mathcal E_b(x)=\mathcal E_b(\bm w(m_0x))$.
The exact identity in Lemma~\ref{lem:app-quantization-expansion} separates
the unquantized error and the quantization contribution:
\[
 \mathcal E_b(x)=\mathcal P_0(x)+\gamma_b\mathcal Q(x),
\]
where $\gamma_b:=c_G/(2^b-1)^2$. Its two precision-independent
components are
\[
 \mathcal P_0(x)
 =\frac12\|\bm\mu(x)-\bm\mu_T\|_{\bm G}^2+N_G R(x),
 \qquad
 \mathcal Q(x)
 =\|\bm\mu(x)\|_2^2+P\sigma^2R(x).
\]
Here $\mathcal P_0$ is the unquantized reconstruction error, and
$\gamma_b\mathcal Q$ is the quantization contribution under
the simplified GPTQ model. The coefficients $N_G$ and $c_G>0$ were
defined in the preceding subsection. Precision enters $\mathcal E_b$
through $\gamma_b$, which increases as precision decreases. The effect
of this increase depends on how $\mathcal Q$ varies with $x$.

We assume that, during a stable training phase, latent-mean increments
are aligned with one another and with the latent mean at the start
of the retained window.

\begin{assumption}[Pairwise Alignment of Latent-Mean Increments]
\label{ass:app-increment-alignment}
The latent increments in the retained window satisfy
\begin{equation*}
 \bm v_{T-i+1}^{\top}\bm v_{T-j+1}\geq0
 \qquad(1\leq i,j<M).
\end{equation*}
\end{assumption}

\begin{assumption}[Alignment with the Oldest Retained Latent Mean]
\label{ass:app-mean-alignment}
The latent increments also satisfy
\begin{equation*}
 \bm\mu_{T-M+1}^{\top}\bm v_{T-j+1}\geq0
 \qquad(1\leq j<M).
\end{equation*}
\end{assumption}

Assumption~\ref{ass:app-increment-alignment} says that the progress
increments do not oppose one another; Assumption~\ref{ass:app-mean-alignment}
says that they do not oppose the oldest retained latent mean. Their
magnitudes and directions may differ. The next lemma uses both sets of
inner products to show that favoring more recent saved weights cannot
decrease the latent-mean norm.

\begin{lemma}[Latent-Mean Norm along \LNWA{}]
\label{lem:app-signal-accumulation}
Under Assumptions~\ref{ass:app-increment-alignment} and
\ref{ass:app-mean-alignment}, every open
\LNWA{} support segment satisfies
\[
 \bm\mu(x)^{\top}\bm\mu'(x)\geq0.
\]
\end{lemma}
\begin{proof}
Express each retained latent mean using the oldest mean and the progress
made since that position:
\[
 \bm\mu_{T-d}
 =\bm\mu_{T-M+1}
   +\sum_{j=d+1}^{M-1}\bm v_{T-j+1}.
\]
Averaging these expressions, with
$s_j(\bm w)=\sum_{d=j}^{M-1}w_d$ as in Appendix~\ref{app:setup}, gives
\begin{equation}
 \bm\mu_T^{\bm w}
 =\bm\mu_{T-M+1}
   +\sum_{j=1}^{M-1}\bigl[1-s_j(\bm w)\bigr]\bm v_{T-j+1}.
\label{eq:app-signal-mean-expansion}
\end{equation}
The coefficient $1-s_j(\bm w)$ is the averaging mass assigned to
saved weights that include the $j$th increment. These coefficients are
nonnegative, so Assumptions~\ref{ass:app-increment-alignment} and
\ref{ass:app-mean-alignment} give
\[
 (\bm\mu_T^{\bm w})^{\top}\bm v_{T-j+1}
 =\bm\mu_{T-M+1}^{\top}\bm v_{T-j+1}
  +\sum_{i=1}^{M-1}\bigl[1-s_i(\bm w)\bigr]
      \bm v_{T-i+1}^{\top}\bm v_{T-j+1}
 \geq0.
\]
On a segment with an active prefix of size $q\geq2$, differentiating the
kernel in Eqs.~(\ref{eq:app-lnwa-kernel})--(\ref{eq:app-lnwa-moment})
with respect to $x=1-m/m_0$ yields
\begin{equation*}
\begin{aligned}
 \bm\mu'(x)
 &=\frac{m_0}{D_q}\sum_{d=0}^{q-1}
       (\bar d_q-d)\bm\mu_{T-d}\\
 &=\frac{m_0}{2D_q}\sum_{j=1}^{q-1}
       j(q-j)\bm v_{T-j+1}.
\end{aligned}
\end{equation*}
The second equality follows from
Eq.~(\ref{eq:app-signal-mean-expansion}) and
$\sum_{d=0}^{j-1}(\bar d_q-d)=j(q-j)/2$.
Each coefficient in this last sum is positive. Taking its inner product
with $\bm\mu(x)$ and applying the preceding inequality proves
$\bm\mu(x)^\top\bm\mu'(x)\geq0$.
\end{proof}

Alignment describes the direction of recent progress. We further
assume that the latent mean continues to change near the end of training.

\begin{assumption}[Nonvanishing Local Progress]
\label{ass:app-nonvanishing-progress}
Let $h=1/T$ be the spacing in Assumption~\ref{ass:app-window-regularity}
and $v_*:=\|\bm v_T\|_2$. For constant $\nu>0$
independent of $h$, the newest increment satisfies, for all
sufficiently small $h$,
\[
 v_*\geq\nu h.
\]
\end{assumption}

Together with Assumptions~\ref{ass:app-window-regularity} and
\ref{ass:app-short-window}, Assumption~\ref{ass:app-nonvanishing-progress}
yields the quantitative local-progress bound in
Lemma~\ref{lem:app-local-progress}.

\begin{lemma}[Local Progress Bound]
\label{lem:app-local-progress}
Under Assumptions~\ref{ass:app-window-regularity},
\ref{ass:app-short-window}, and~\ref{ass:app-nonvanishing-progress},
for sufficiently large $T$, let
$\rho_T:=((M-2)A/\nu)h$. Then
\begin{equation}
 \max_{1\leq j<M}\|\bm v_{T-j+1}-\bm v_T\|_2
 \leq\rho_Tv_*,
 \qquad \rho_T\longrightarrow0.
 \label{eq:app-preference-window}
\end{equation}
\end{lemma}
\begin{proof}
Eq.~(\ref{eq:app-turnover-window-rates}) gives $\kappa=Ah^2$.
By the adjacent-increment bound and Assumption~\ref{ass:app-nonvanishing-progress},
\[
 \|\bm v_{T-j+1}-\bm v_T\|_2
 \leq(j-1)\kappa
 \leq(M-2)Ah^2
 \leq\frac{(M-2)A}{\nu}h\,v_*
 =\rho_Tv_*.
\]
Since $M$, $A$, and $\nu$ are fixed and $h\to0$, we have
$\rho_T\to0$, proving Eq.~(\ref{eq:app-preference-window}).
\end{proof}

In the interior of the first support segment, all $M$ averaging
coefficients are positive. Its right endpoint is
\[
 x_{\rm F}:=\frac{M+1}{3(M-1)},\qquad
 Y:=m_0x_{\rm F}=\frac{M+1}{6}.
\]
For $M=2$, this segment extends to $x=1$.
Let $\bm u:=m_0\bm a_M$ denote the change in the averaged mean per
unit increase of $x$ on this segment. The frontier formulas give
\[
 \bm\mu(x)=\bar{\bm\mu}+x\bm u,
 \qquad
 R(x)=\frac1M+\frac{m_0^2x^2}{D_M},
 \qquad 0\leq x\leq x_{\rm F}.
\]
The first derivatives at $x=0$ are
\[
 \mathcal P_0'(0^+)=\bm u^\top\bm G(\bar{\bm\mu}-\bm\mu_T),
 \qquad \mathcal Q'(0^+)=2\bar{\bm\mu}^{\top}\bm u.
\]
Here $R'(0^+)=0$, so spatial noise has no first-order contribution.
To determine the signs of these slopes, write the two vectors in the
first expression as averages of the retained increments:
\[
\begin{aligned}
 \frac{\bm u}{m_0}
 &=\sum_{j=1}^{M-1}\frac{j(M-j)}{2D_M}\bm v_{T-j+1},\\
 \frac{\bm\mu_T-\bar{\bm\mu}}{m_0}
 &=\sum_{j=1}^{M-1}\frac{2(M-j)}{M(M-1)}\bm v_{T-j+1}.
\end{aligned}
\]
Both sets of nonnegative coefficients sum to one. By
Lemma~\ref{lem:app-local-progress}, each average is within
$\rho_Tv_*$ of $\bm v_T$ in Euclidean norm. Thus $\bm G\succ0$ and
$\rho_T\to0$ give, for sufficiently large $T$,
\[
\begin{aligned}
 -\frac{\mathcal P_0'(0^+)}{m_0^2}
 &=\left(\frac{\bm u}{m_0}\right)^\top\bm G
   \left(\frac{\bm\mu_T-\bar{\bm\mu}}{m_0}\right)\\
 &\geq\bigl[\lambda_{\min}(\bm G)
  -\|\bm G\|_2(2\rho_T+\rho_T^2)\bigr]v_*^2>0.
\end{aligned}
\]
For the other slope, Assumptions~\ref{ass:app-increment-alignment} and
\ref{ass:app-mean-alignment} give
$\bar{\bm\mu}^{\top}\bm v_{T-j+1}\geq0$ by
Eq.~(\ref{eq:app-signal-mean-expansion}). Since the first average above
has nonnegative coefficients,
\[
 \mathcal Q'(0^+)
 =\frac{m_0}{D_M}\sum_{j=1}^{M-1}j(M-j)
       \bar{\bm\mu}^{\top}\bm v_{T-j+1}\geq0.
\]
Hence increasing $x$ initially lowers the unquantized error and
changes the quantization contribution at the nonnegative rate
$\gamma_b\mathcal Q'(0^+)$.
Their second derivatives are constant on this segment:
\[
 a:=\mathcal P_0''(0^+)
 =\bm u^\top\bm G\bm u+\frac{2N_Gm_0^2}{D_M}>0,
 \qquad
 c:=\mathcal Q''(0^+)
 =2\|\bm u\|_2^2+\frac{2P\sigma^2m_0^2}{D_M}>0.
\]
Consequently, the following quadratic expansions are exact throughout
the first support segment:
\begin{equation}
\begin{aligned}
 \mathcal P_0(x)-\mathcal P_0(0)
 &=\mathcal P_0'(0^+)x+\frac{\mathcal P_0''(0^+)}{2}x^2,\\
 \mathcal Q(x)-\mathcal Q(0)
 &=\mathcal Q'(0^+)x+\frac{\mathcal Q''(0^+)}{2}x^2.
\end{aligned}
 \label{eq:app-preference-first-segment}
\end{equation}
Under the assumptions of Theorem~\ref{thm:app-turnover}, the upper
bound on $\tau_b$ lies below $x_{\rm F}$ for sufficiently large $T$.
Hence Eq.~(\ref{eq:app-preference-first-segment}) applies throughout
the tolerance interval. We next use these quadratics to locate the
minimizers of $\mathcal P_0$ and $\mathcal E_b$ and show that they are
global.

\begin{lemma}[Optima on the First Support Segment]
\label{lem:app-preference-interior}
Under Assumptions~\ref{ass:app-bounded-noise}--\ref{ass:app-nonvanishing-progress},
for sufficiently large $T$, the unique minimizer of $\mathcal P_0$
on $[0,1]$ is $x_0=\bm u^\top\bm G(\bm\mu_T-\bar{\bm\mu})/a\in(0,x_{\rm F})$.
For every evaluated precision $b\geq b_{\min}$, the unique minimizer
of $\mathcal E_b$ on $[0,1]$ is
\begin{equation}
 x_b=\max\left\{0,
 \frac{\bm u^\top\bm G(\bm\mu_T-\bar{\bm\mu})
 -2\gamma_b\bar{\bm\mu}^\top\bm u}
 {a+c\gamma_b}\right\}\in[0,x_0).
 \label{eq:app-preference-minimizer}
\end{equation}
\end{lemma}
\begin{proof}
The preceding estimates give $\mathcal P_0'(0^+)<0$ and
$\mathcal Q'(0^+)\geq0$. We next show that the unquantized error
increases beyond the first
segment. Set $\bar v:=\max_{1\leq j<M}\|\bm v_{T-j+1}\|_2$.
Assumptions~\ref{ass:app-window-regularity} and
\ref{ass:app-short-window} give $\bar v\leq Vh$ and, for sufficiently
large $T$,
\[
 \sqrt g\,(M-1)\bar v
 \leq\sqrt g\,V\ell
 \leq2\sqrt{N_G/M}.
\]
In the mean-age coordinate $y$, the increment bound and the convex coefficients in
$\bm a_q$ imply
\[
 \|\bm\mu(y)-\bm\mu_T\|_2
 \leq\bar v\,m(y)\leq m_0\bar v,
 \qquad \|\bm a_q\|_2\leq\bar v.
\]
The derivative $R'(y)$ is continuous and increasing across the support
segments. For $y\geq Y$, it is at least
$R'(Y)=2Y/D_M=4/[M(M-1)]$. Consequently, on every such segment,
\[
\begin{aligned}
 \frac{\mathrm d}{\mathrm dy}\mathcal P_0(\bm w(y))
 &=\bm a_q^\top\bm G(\bm\mu(y)-\bm\mu_T)+N_GR'(y)\\
 &\geq-gm_0\bar v^2
        +\frac{4N_G}{M(M-1)}\\
 &\geq\frac{2N_G}{M(M-1)}>0.
\end{aligned}
\]
The final inequality uses the displayed bound on $\bar v$ and
$m_0=(M-1)/2$. The same estimates apply to
the one-sided derivatives at support boundaries. For $M=2$, they
apply to the left derivative at $Y=m_0$.

On the first segment, $\mathcal P_0'(0^+)<0$ and
$\mathcal P_0''=a>0$. Its derivative is positive at $x_{\rm F}$
by the preceding bound, so it vanishes once at
$x_0=\bm u^\top\bm G(\bm\mu_T-\bar{\bm\mu})/a\in(0,x_{\rm F})$.
Thereafter $\mathcal P_0$ increases on every segment. Continuity at
their boundaries makes $x_0$ its unique global minimum.

For $\mathcal E_b$, the initial slope is
$\bm u^\top\bm G(\bar{\bm\mu}-\bm\mu_T)
+2\gamma_b\bar{\bm\mu}^\top\bm u$.
If it is nonnegative, Theorem~\ref{thm:app-turnover} gives
$\tau_b=0$ and strict increase for $x>0$, so $x_b=0$.
If it is negative, Lemma~\ref{lem:app-signal-accumulation} gives
$\mathcal Q'(x)=2\bm\mu(x)^\top\bm\mu'(x)+P\sigma^2R'(x)>0$
for $x>0$ within each segment. In particular,
$\mathcal E_b'(x_0)=\gamma_b\mathcal Q'(x_0)>0$.
Solving the affine derivative of
Eq.~(\ref{eq:app-preference-first-segment}) therefore gives the unique
zero in $(0,x_0)$ stated in Eq.~(\ref{eq:app-preference-minimizer}).
Beyond the first segment, both $\mathcal P_0'$ and $\mathcal Q'$
are positive. Hence this zero is the unique global minimizer of
$\mathcal E_b$.
\end{proof}

Lemma~\ref{lem:app-preference-interior} gives the global minimizer at
each evaluated precision. In Eq.~(\ref{eq:app-preference-minimizer}),
precision enters only through $\gamma_b$: lower precision increases
this coefficient and gives more weight to the quantization cost, while
higher precision reduces it. This changes the optimal trade-off between
denoising and reducing temporal lag. The following theorem identifies
the precision at which that preference changes and proves how the
optimum responds.

\begin{theorem}[Precision-Dependent Averaging Preference]
\label{thm:app-bit-ordering}
Under Assumptions~\ref{ass:app-bounded-noise}--\ref{ass:app-nonvanishing-progress},
for sufficiently large $T$, define
\begin{equation}
 b_{\mathrm{crit}}
 :=\log_2\!\left(
 1+\sqrt{\frac{2c_G\bar{\bm\mu}^\top\bm u}
{\bm u^\top\bm G(\bm\mu_T-\bar{\bm\mu})}}
 \right).
 \label{eq:app-critical-precision}
\end{equation}
For every evaluated precision $b\geq b_{\min}$,
\[
\begin{aligned}
 b\leq b_{\mathrm{crit}}
 &\quad\Longrightarrow\quad x_b=0,\quad\tau_b=0
 \qquad\text{\textbf{(less spatial noise)}},\\
 b>b_{\mathrm{crit}}
 &\quad\Longrightarrow\quad x_b>0,\quad\tau_b>0
 \qquad\text{\textbf{(less temporal lag)}}.
\end{aligned}
\]
The optimal position $x_b$ is nondecreasing in $b$ and strictly
increasing in the regime $b>b_{\mathrm{crit}}$.
If $b_{\mathrm{crit}}<b_{\min}$, all evaluated precisions lie in the
second regime.
\end{theorem}
\begin{proof}
By Lemma~\ref{lem:app-preference-interior}, $x_b$ is determined by
the sign of $\mathcal E_b'(0^+)$. Since
\[
 \mathcal E_b'(0^+)
 =\bm u^\top\bm G(\bar{\bm\mu}-\bm\mu_T)
  +\frac{2c_G\bar{\bm\mu}^\top\bm u}{(2^b-1)^2},
\]
this slope is nonnegative exactly when $b\leq b_{\mathrm{crit}}$.
The equivalence also holds when $\bar{\bm\mu}^\top\bm u=0$, in which case
$b_{\mathrm{crit}}=0$ and every evaluated precision has a negative
initial slope. Theorem~\ref{thm:app-turnover} then gives
$\tau_b=0$ in the nonnegative-slope regime and $\tau_b>0$ in the
negative-slope regime.

Where the maximum in Eq.~(\ref{eq:app-preference-minimizer}) is
positive, differentiation with respect to $\gamma_b$ gives
\[
 \frac{\mathrm d}{\mathrm d\gamma_b}
 \left(
 \frac{\bm u^\top\bm G(\bm\mu_T-\bar{\bm\mu})
 -2\gamma_b\bar{\bm\mu}^\top\bm u}
 {a+c\gamma_b}
 \right)
 =-\frac{2a\bar{\bm\mu}^\top\bm u
 +c\bm u^\top\bm G(\bm\mu_T-\bar{\bm\mu})}
 {(a+c\gamma_b)^2}<0.
\]
As $\gamma_b$ strictly decreases with $b$, $x_b$ is nondecreasing
overall and strictly increasing above the threshold.
\end{proof}

Together, these results characterize the optimal kernel's shift in
lag--noise preference: lower precision favors denoising, whereas higher
precision favors reducing temporal lag. Eq.~(\ref{eq:app-critical-precision})
gives the explicit precision threshold for this transition.

\section{Key Assumptions and Limitation}
\subsection{Empirical Assessment of Key Assumptions}
\label{app:empirical-assumptions}

Our assumptions concern the latent mean, which is defined as a conditional
expectation over training randomness. The law of large numbers gives, at each
fixed training step $t$,
\[
 \widehat{\bm\mu}_t^{(n)}
 :=\frac1n\sum_{r=1}^{n}\bm\theta_t^{(r)}
 \xrightarrow[n\to\infty]{\mathrm{a.s.}}
 \E[\bm\theta_t\mid\mathcal F_0]=\bm\mu_t.
\]
Repeated runs let us estimate the latent mean and spatial noise.
Since rerunning LLM pretraining many times is costly, we use a small
neural network to compare trajectories across runs and empirically assess
our key assumptions.

\paragraph{Experimental setup.}
We train a two-layer ReLU multilayer perceptron (MLP) on MNIST dataset. Images are
average-pooled to $14\times14$ and standardized, and the network has
dimensions $196\to64\to10$. We train the full network with SGD and
analyze its 650 output-layer weights and biases, denoted by
$\bm\theta_t$. To reflect the single-epoch data consumption of LLM pretraining, we sample
without replacement throughout training. We split MNIST into 58,000
training and 2,000 validation images, train a shared prefix on 16,000
images, and then branch into 16 runs that traverse the remaining data in
different orders. Across
the 16 branches, we estimate the latent mean and spatial noise by
\[
 \widehat{\bm\mu}_t^{(16)}=\frac1{16}\sum_{r=1}^{16}\bm\theta_t^{(r)},
 \qquad
 \widehat{\bm\epsilon}_t^{(r)}
 =\bm\theta_t^{(r)}-\widehat{\bm\mu}_t^{(16)}.
\]

The shared prefix uses 500 updates at batch size 32 and learning rate
$0.03$. Each branch then runs for 1,000 updates. We vary the continuation
learning rate over $\{0.01,0.02,0.03\}$ at batch size 32 and the batch
size over $\{24,32,40\}$ at learning rate $0.03$.
\begin{figure}
  \centering
  \includegraphics[width=\textwidth]{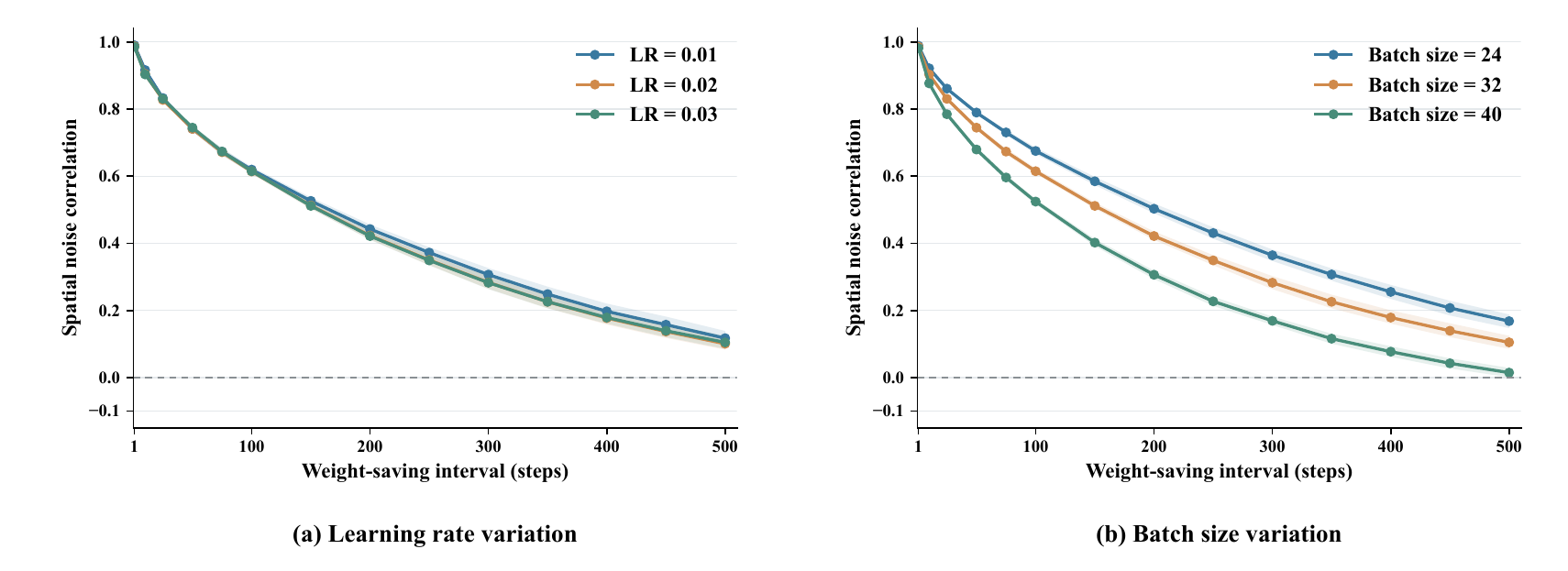}
  \caption{Spatial-noise correlation as the interval between saved
  weights increases. Left: learning-rate variation at batch size 32.
  Right: batch-size variation at learning rate $0.03$.}
  \label{fig:assumption-temporal}
\end{figure}

\paragraph{Spatial noise decorrelation.}
Figure~\ref{fig:assumption-temporal} measures the correlation between
spatial noise at different training positions. Let
$\mathcal T:=\{500,\ldots,1000\}$ denote the updates in the
analysis window. For a saving interval $h$, let
$\mathcal T_h:=\{t\in\mathcal T:t-h\geq0\}$ denote the endpoints for which both \(\widehat{\bm\epsilon}_{t-h}\) and \(\widehat{\bm\epsilon}_{t}\) are available. We evaluate the correlation as
\[
 \widehat\rho_{\epsilon}(h)
 :=
 \frac{\displaystyle
   \sum_{t\in\mathcal T_h}\sum_{r=1}^{16}
   \bigl\langle\widehat{\bm\epsilon}_{t-h}^{(r)},
                   \widehat{\bm\epsilon}_{t}^{(r)}\bigr\rangle}
 {\displaystyle
  \sqrt{\sum_{t\in\mathcal T_h}\sum_{r=1}^{16}
             \|\widehat{\bm\epsilon}_{t-h}^{(r)}\|_2^2}
  \sqrt{\sum_{t\in\mathcal T_h}\sum_{r=1}^{16}
             \|\widehat{\bm\epsilon}_{t}^{(r)}\|_2^2}},
\]
The correlation decreases steadily with $h$. At $h=500$, it is
$0.101$--$0.117$ across learning rates and
decreases from $0.168$ to $0.015$ as the batch size increases from 24 to
40. This supports the weak temporal dependence used in
Assumption~\ref{ass:app-noise} for sparsely saved weights.

\begin{figure}
  \centering
  \includegraphics[width=\textwidth]{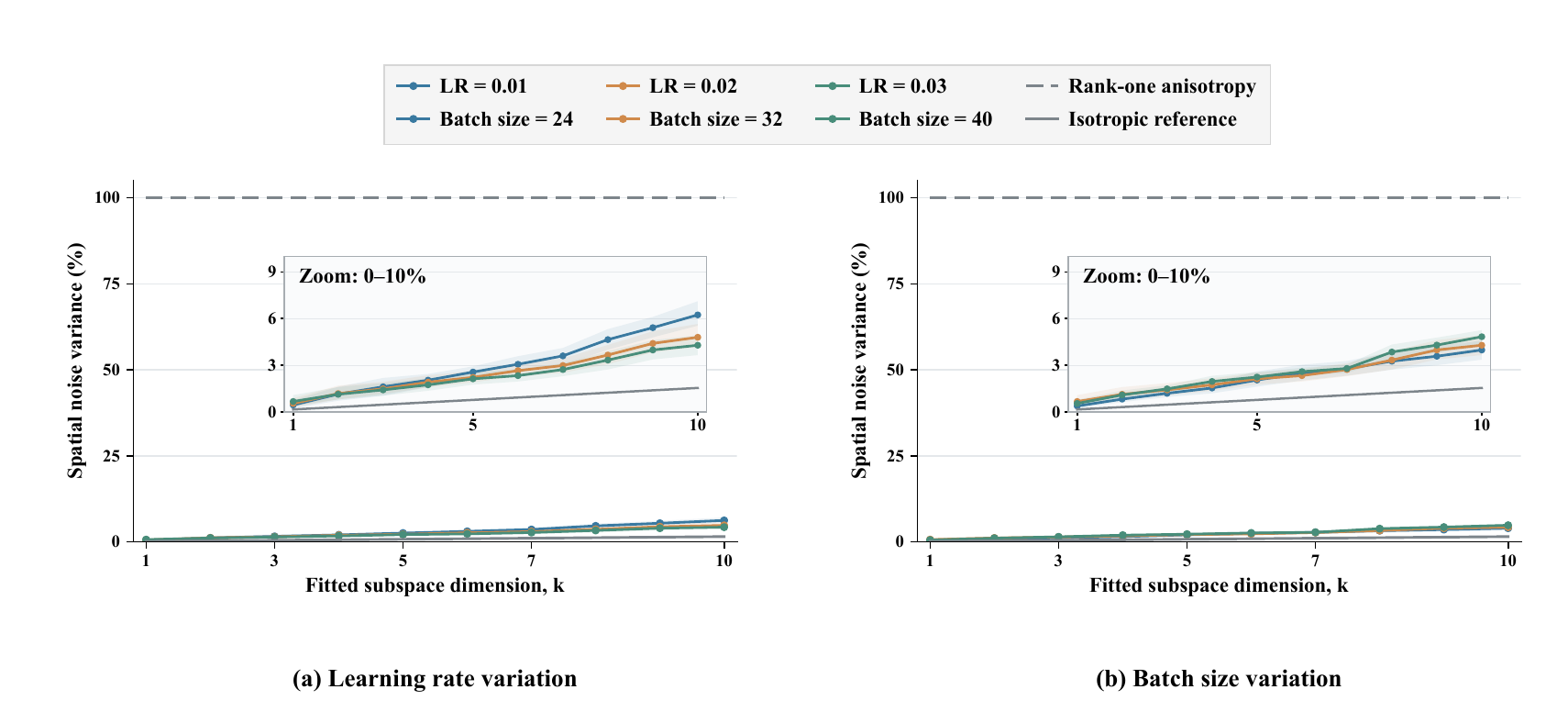}
  \caption{Spatial-noise variance captured by the first
  $k$ fitted directions. Left: learning-rate variation at batch size 32.
  Right: batch-size variation at learning rate $0.03$.}
  \label{fig:assumption-spatial}
\end{figure}

\paragraph{Spatial noise dispersion.}
Figure~\ref{fig:assumption-spatial} uses eight branches to identify the
principal directions of spatial noise and the remaining eight to evaluate
them. Let $\mathcal R_{\mathrm{fit}}$ and $\mathcal R_{\mathrm{eval}}$
denote these two sets. We recompute the latent-mean estimate within each
set and obtain the corresponding residuals
$\widehat{\bm\epsilon}_{t,\mathrm{fit}}^{(r)}$ and
$\widehat{\bm\epsilon}_{t,\mathrm{eval}}^{(r)}$. From the fitting set, we
form the empirical noise covariance
\[
 \widehat{\bm\Sigma}_{\mathrm{fit}}
 :=\frac{1}{|\mathcal T|\,|\mathcal R_{\mathrm{fit}}|}
 \sum_{t\in\mathcal T}\sum_{r\in\mathcal R_{\mathrm{fit}}}
 \widehat{\bm\epsilon}_{t,\mathrm{fit}}^{(r)}
 \bigl(\widehat{\bm\epsilon}_{t,\mathrm{fit}}^{(r)}\bigr)^\top.
\]
Let $\bm U_k\in\mathbb R^{650\times k}$ contain the orthonormal
eigenvectors associated with its $k$ largest eigenvalues. We then measure
how much spatial-noise variance these directions capture on the evaluation
branches:
\[
 \widehat V_{\epsilon}(k)
 :=100\times
 \frac{\displaystyle
   \sum_{t\in\mathcal T}\sum_{r\in\mathcal R_{\mathrm{eval}}}
   \|\bm U_k^\top\widehat{\bm\epsilon}_{t,\mathrm{eval}}^{(r)}\|_2^2}
 {\displaystyle
   \sum_{t\in\mathcal T}\sum_{r\in\mathcal R_{\mathrm{eval}}}
   \|\widehat{\bm\epsilon}_{t,\mathrm{eval}}^{(r)}\|_2^2}\,\%.
\]
The parameter space has 650 dimensions. If the noise were concentrated in
only a few directions, the first ten directions would explain a large
fraction of its variance; rank-one noise, for example, would be captured
entirely by the first direction. Instead, the first ten directions explain
only $3.98\%$--$6.24\%$ of the variance. Thus the spatial noise is not
concentrated in a small number of directions but dispersed across many
parameter directions. For analytical simplicity, Assumption~\ref{ass:app-bounded-noise}
models this dispersed spatial noise with an isotropic covariance.

\begin{figure}
  \centering
  \includegraphics[width=\textwidth]{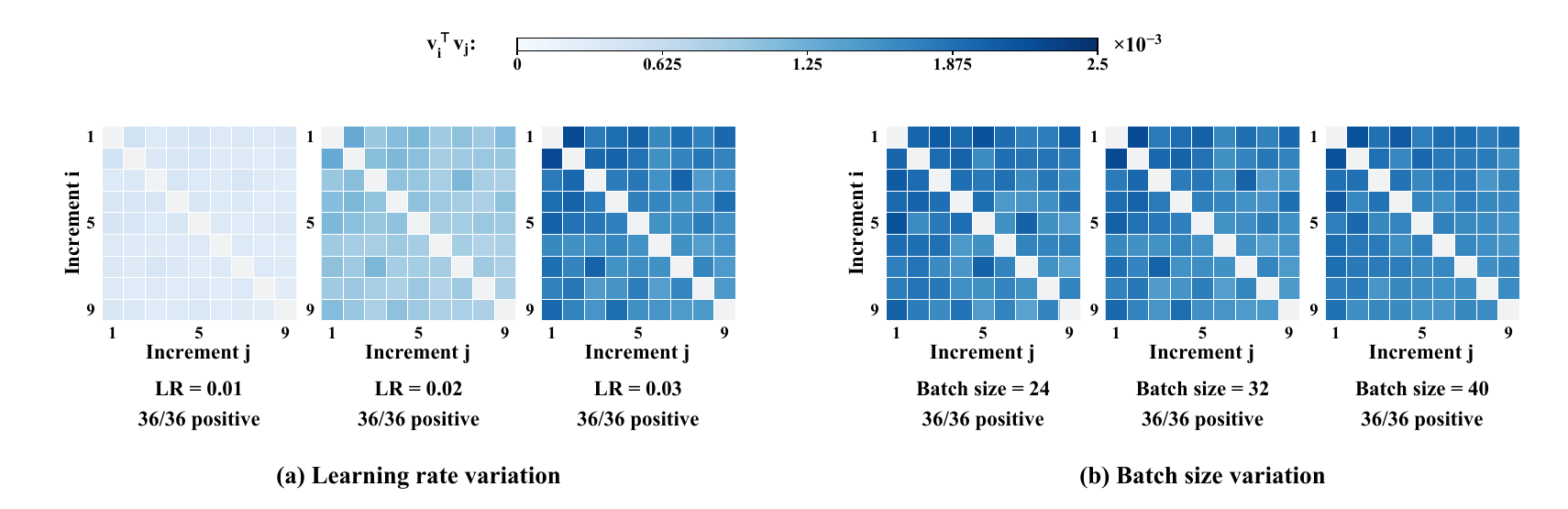}
  \caption{Cross-run estimates of the pairwise inner products between
  nine consecutive latent-mean progress increments. Left: learning-rate
  variation. Right: batch-size variation.}
  \label{fig:assumption-increments}
\end{figure}

\paragraph{Alignment between progress increments.}
Let $t_0<\cdots<t_9$ be the ten selected positions and define the
run-specific increments
$\Delta\bm\theta_j^{(r)}:=\bm\theta_{t_j}^{(r)}-
\bm\theta_{t_{j-1}}^{(r)}$. Figure~\ref{fig:assumption-increments}
estimates the inner product between latent-mean increments using distinct
runs,
\[
 \widehat A_{ij}
 :=\frac{1}{16\cdot15}\sum_{r\ne s}
 \bigl\langle\Delta\bm\theta_i^{(r)},
                 \Delta\bm\theta_j^{(s)}\bigr\rangle,
 \qquad 1\leq i<j\leq9.
\]
Each increment is first computed within one branch. In
$\widehat A_{ij}$, the increment at position $i$ from branch
$r$ is paired with the increment at position $j$ from a different branch
$s\ne r$, and the result is averaged over all $16\cdot15$ ordered branch
pairs. Because the two branches have independent zero-mean fluctuations,
their noise terms vanish in expectation. Hence $\widehat A_{ij}$ estimates
the inner product between the corresponding latent-mean increments.
All 36 pairs are positive in every setting, showing that recent progress
tends to remain aligned and supporting
Assumption~\ref{ass:app-increment-alignment}.

\begin{figure}
  \centering
  \includegraphics[width=\textwidth]{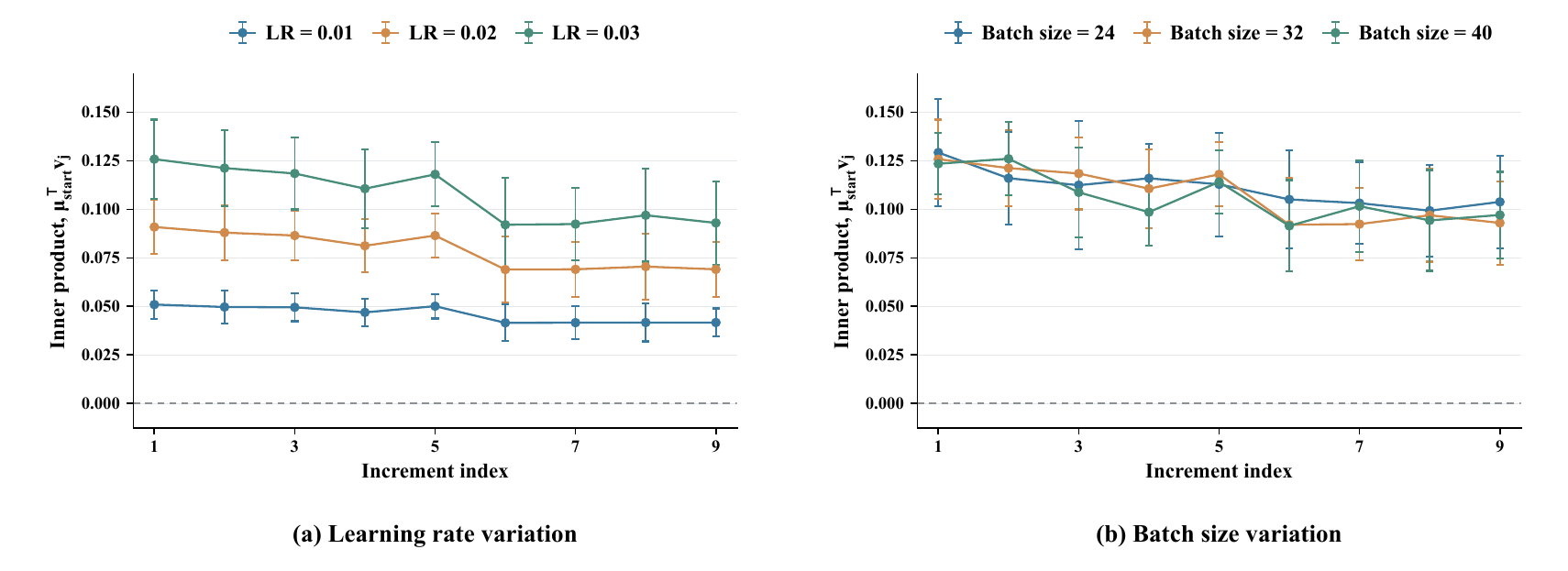}
  \caption{Cross-run estimates of the inner products between the
  window-start latent mean and nine subsequent progress increments.
  Left: learning-rate variation. Right: batch-size variation.}
  \label{fig:assumption-window-start}
\end{figure}

\paragraph{Alignment with the window-start mean.}
Figure~\ref{fig:assumption-window-start} measures the inner products
between the latent mean at the start of the analysis window and the nine
subsequent progress increments. It uses the similar cross-run estimate
\[
 \widehat B_j
 :=\frac{1}{16\cdot15}\sum_{r\ne s}
 \bigl\langle\bm\theta_{t_0}^{(r)},
                 \Delta\bm\theta_j^{(s)}\bigr\rangle,
 \qquad 1\leq j\leq9.
\]
All nine values are positive in every setting, supporting
Assumption~\ref{ass:app-mean-alignment}. 

\subsection{Limitation}
\label{app:limitations}

The main limitation of our work is that the theoretical guarantees are
established under an idealized setting. However, bridging tractable theory and practical
quantization pipelines is a broader challenge for the community rather than our uniqueness. In particular, our training analysis makes
assumptions about the latent mean and spatial noise. We provide an empirical
assessment of the key assumptions in Section
\ref{app:empirical-assumptions}, which offers support for the
theoretical setting. Our quantization analysis adopts a simplified GPTQ model: we use randomized
dither rounding rather than deterministic nearest rounding, which has well-established statistical properties and has
been used in quantization theory \citep{823976, chen2023quantized, dirksen2025subspace}; we also use a global scale
rather than group-wise scales and consider only a two-stage quantization and
compensation process. Such simplifications are common and widely accepted in
the community \citep{banner2018aciq, nagel2020up, chee2023quip}.
Future work will further develop the theory toward more practical quantization
settings.

\begin{figure}
  \centering
  \includegraphics[width=0.98\textwidth]{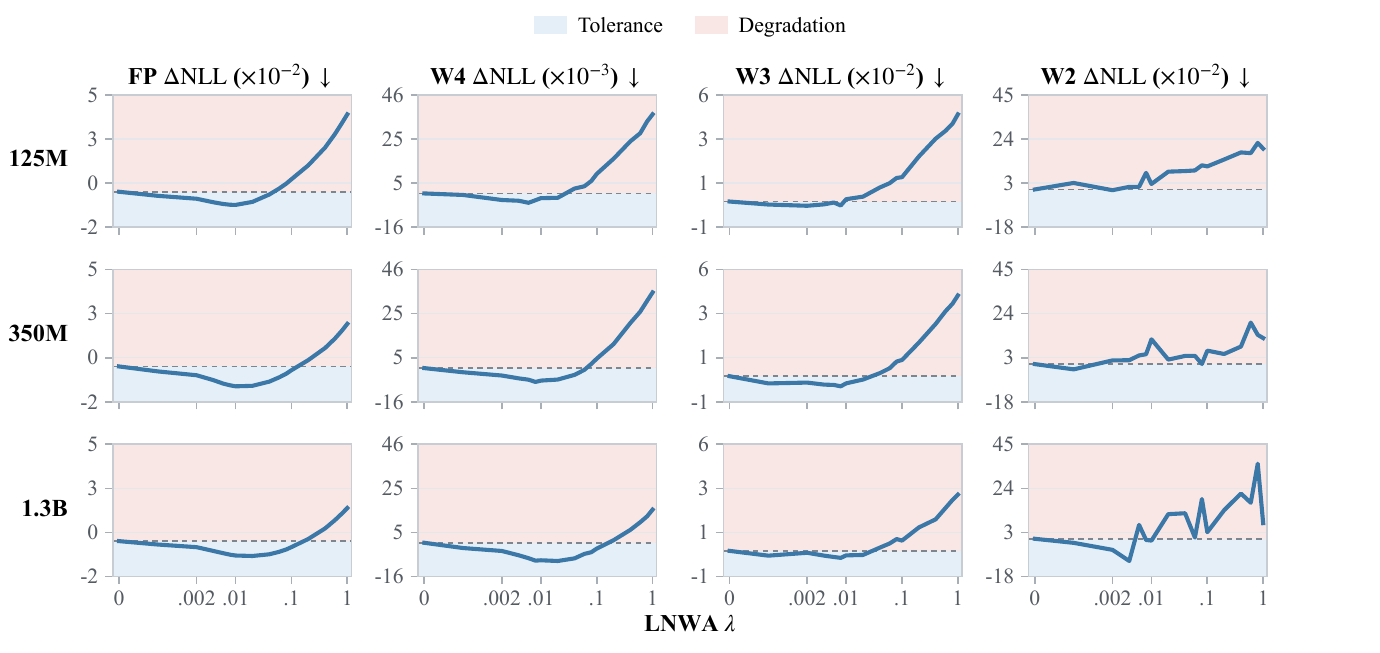}
  \caption{Comparison of $\Delta$NLL across \LNWA{} coefficients for OPT-125M, OPT-350M, and OPT-1.3B under GPTQ at 4-bit (W4), 3-bit (W3), 2-bit (W2), and full precision (FP).}
  \label{fig:lnwa-turnover-opt-full}
\end{figure}

\section{Additional Experimental Results}
\label{app:experiments}
\subsection{Experimental Settings}
\label{app:experimental-settings}

\paragraph{Pretraining settings.}
We pretrain OPT-series models \citep{zhang2022opt} (125M, 350M, and 1.3B) from scratch on FineWeb-Edu \citep{penedo2024fineweb}, with the model-specific configurations summarized in Table~\ref{tab:pretraining-settings}. Each training example is a sequence of length 512, and the global batch size denotes the number of such sequences processed per optimizer update. We use AdamW \citep{loshchilov2017decoupled} with $(\beta_1,\beta_2)=(0.9,0.95)$ and $\epsilon=10^{-8}$ throughout training. The LR is linearly warmed up to its peak value during the first $1\%$ of training and kept constant for the following $89\%$. Over the final $10\%$, we either continue training at the peak LR or transition to a cosine decay schedule, yielding the stable-LR and LR decay baselines, respectively. For weight averaging, we save the last $M=10$ weights from each stable-LR trajectory at intervals of 500 optimizer steps \citep{catalan2026training}. All experiments are run on 4 NVIDIA RTX PRO 6000 GPUs.

\begin{table}[!htbp]
  \centering
  \caption{Pretraining settings used in the experiments. Global batch size is measured in
  sequences per optimizer update, and gradient clipping is by global norm.}
  \label{tab:pretraining-settings}
  \scriptsize
  \setlength{\tabcolsep}{2.7pt}
  \renewcommand{\arraystretch}{1.12}
  \resizebox{0.6\textwidth}{!}{%
  \begin{tabular}{lccc}
    \toprule
    Model size & 125M & 350M & 1.3B \\
    \midrule
    Training tokens
      & 2B & 4B & 20B \\
    Batch size
      & 256 & 512 & 2560 \\
    Sequence length
      & 512 & 512 & 512 \\
    Peak learning rate
      & $1.0{\times}10^{-3}$ & $1.0{\times}10^{-3}$
      & $1.0{\times}10^{-3}$ \\
    LR warmup ratio
      & $1\%$ & $1\%$ & $1\%$ \\
    LR decay ratio
      & $10\%$ & $10\%$ & $10\%$ \\
    Weight decay
      & $0.1$ & $0.1$ & $0.1$ \\
    Gradient clipping
      & $0.5$ & $0.5$ & $0.5$ \\
    \bottomrule
  \end{tabular}%
  }
\end{table}

\paragraph{Post-training quantization settings.}
For both GPTQ \citep{frantar2022gptq} and AWQ \citep{lin2024awq}, we use the same calibration set, consisting of 128 length-512 sequences deterministically sampled from a fixed English C4 training split \citep{raffel2020exploring}. Both methods use a group size of 128 and are evaluated at 4-bit (W4), 3-bit (W3), and 2-bit (W2) precision. For GPTQ, we use $1\%$ Hessian damping, disable activation ordering, and set the column block size to 128. For AWQ, we use asymmetric affine quantization with zero points, with
scaling exponent $\alpha=0.5$ and clipping ratio $c=0.8$.
We evaluate every quantized endpoint on 1,024 blocks from a fixed held-out RefinedWeb split. All remaining quantization settings are kept identical across model sizes.

\begin{figure}
  \centering
  \includegraphics[width=0.98\textwidth]{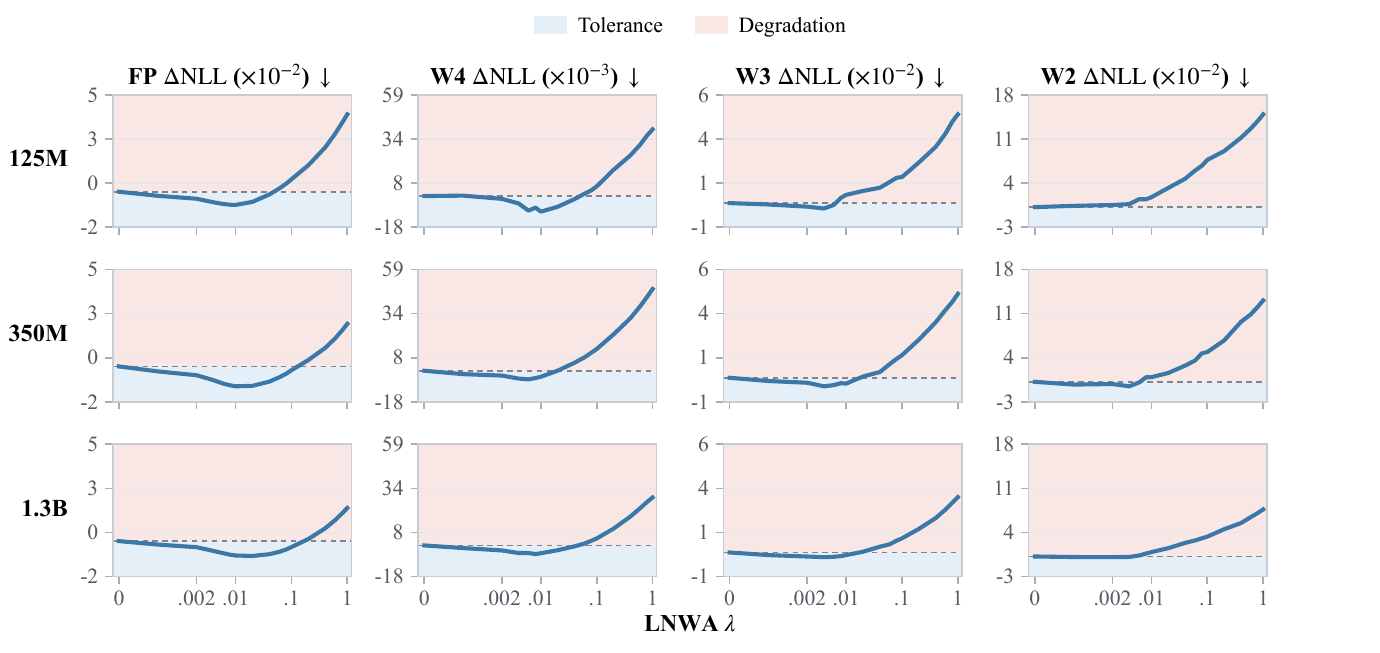}
  \caption{Comparison of $\Delta$NLL across \LNWA{} coefficients for OPT-125M, OPT-350M, and OPT-1.3B under AWQ at 4-bit (W4), 3-bit (W3), 2-bit (W2), and full precision (FP).}
  \label{fig:lnwa-turnover-opt-awq-full}
\end{figure}

\paragraph{Averaging kernels and coefficient grids.}
\label{app:averaging-settings}
The tolerance-to-degradation experiment evaluates the same coefficients $\lambda\in[0, 1]$ at every model size and precision. The precision-preference comparison uses
\LAWA{}, \WMA{}, and \LNWA{} with
$\lambda\in\{0.002, 0.004, 0.006, 0.008\}$. The
\LNWA-equivalent coefficients of the fixed baselines are $0$ for
\LAWA{} and approximately $0.006061$ for \WMA{}. Appendix~\ref{app:classical} derives these
relations.

\begin{figure}
  \centering
  \includegraphics[width=0.98\textwidth]{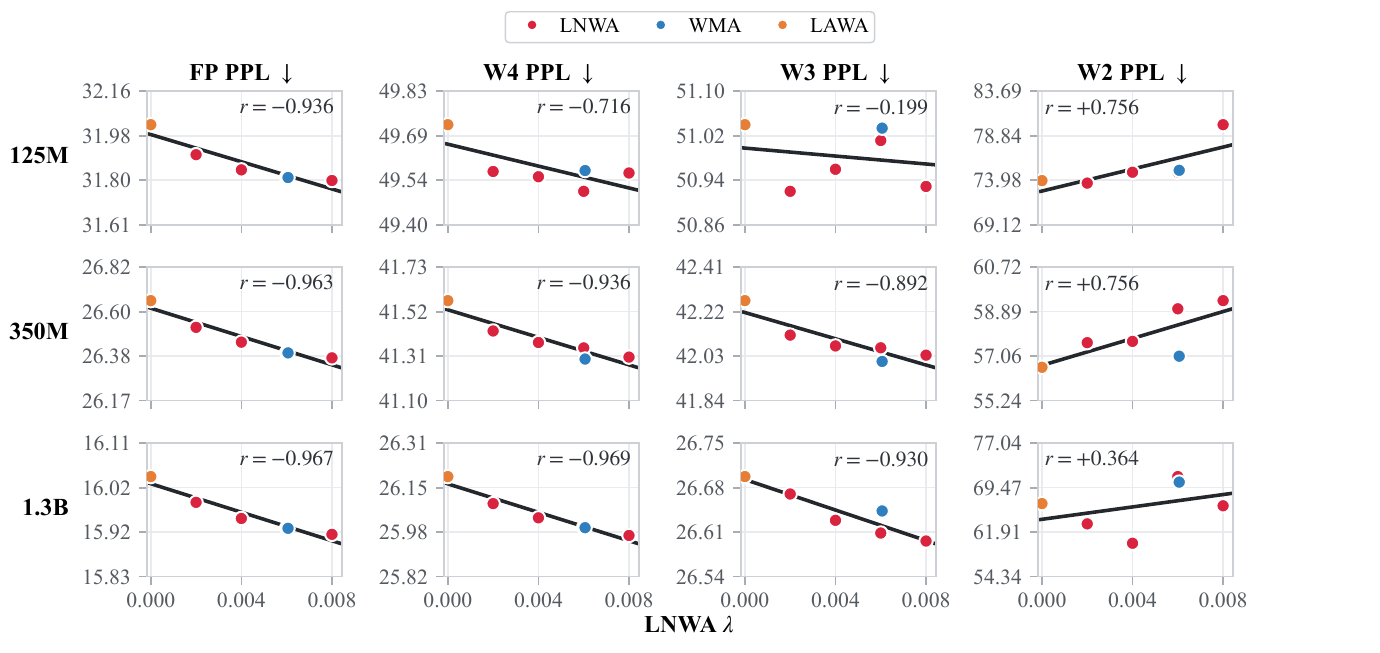}
  \caption{Comparison of PPL across \LNWA{}-equivalent averaging coefficients for OPT-125M, OPT-350M, and OPT-1.3B under GPTQ at 4-bit (W4), 3-bit (W3), 2-bit (W2), and full precision (FP).}
  \label{fig:lambda-regression-opt-full}
\end{figure}

For each model--precision panel, we form $n=6$ pairs $(x_i,y_i)$ from
\LAWA{}, \WMA{}, and \LNWA{} with
$\lambda\in\{0.002, 0.004, 0.006, 0.008\}$, where $x_i$ is the endpoint's
\LNWA-equivalent coefficient and $y_i$ is its measured PPL.  The black line
is the ordinary least-squares fit
\begin{equation*}
  (\hat a,\hat b)
  =\arg\min_{a,b}\sum_{i=1}^{n}\bigl[y_i-(a+b x_i)\bigr]^2,
\end{equation*}
and each panel reports the Pearson correlation
\begin{equation*}
  r=\frac{\sum_{i=1}^{n}(x_i-\bar x)(y_i-\bar y)}
  {\sqrt{\sum_{i=1}^{n}(x_i-\bar x)^2}
   \sqrt{\sum_{i=1}^{n}(y_i-\bar y)^2}}\in[-1,1].
\end{equation*}
Here $\bar x$ and $\bar y$ are the sample means of the $n$ coefficient
and PPL values, respectively. Lower PPL indicates better performance.
A negative correlation associates
larger $\lambda$, and hence less lag, with lower PPL. A positive
correlation associates smaller $\lambda$, and hence less noise, with lower
PPL. The magnitude $|r|$ measures the strength of the linear association
across the six endpoints.

\begin{figure}
  \centering
  \includegraphics[width=0.98\textwidth]{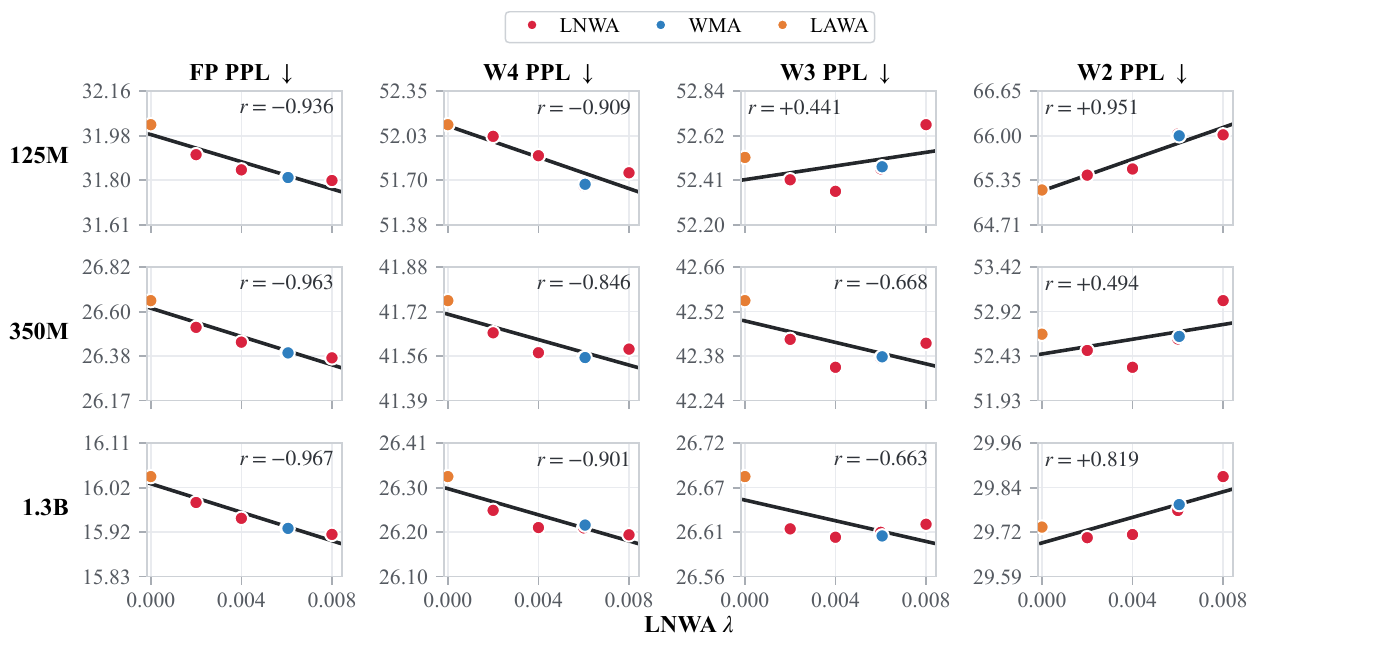}
  \caption{Comparison of PPL across \LNWA{}-equivalent averaging coefficients for OPT-125M, OPT-350M, and OPT-1.3B under AWQ at 4-bit (W4), 3-bit (W3), 2-bit (W2), and full precision (FP).}
  \label{fig:lambda-regression-opt-awq-full}
\end{figure}
\begin{figure}
  \centering
  \includegraphics[width=0.72\textwidth]{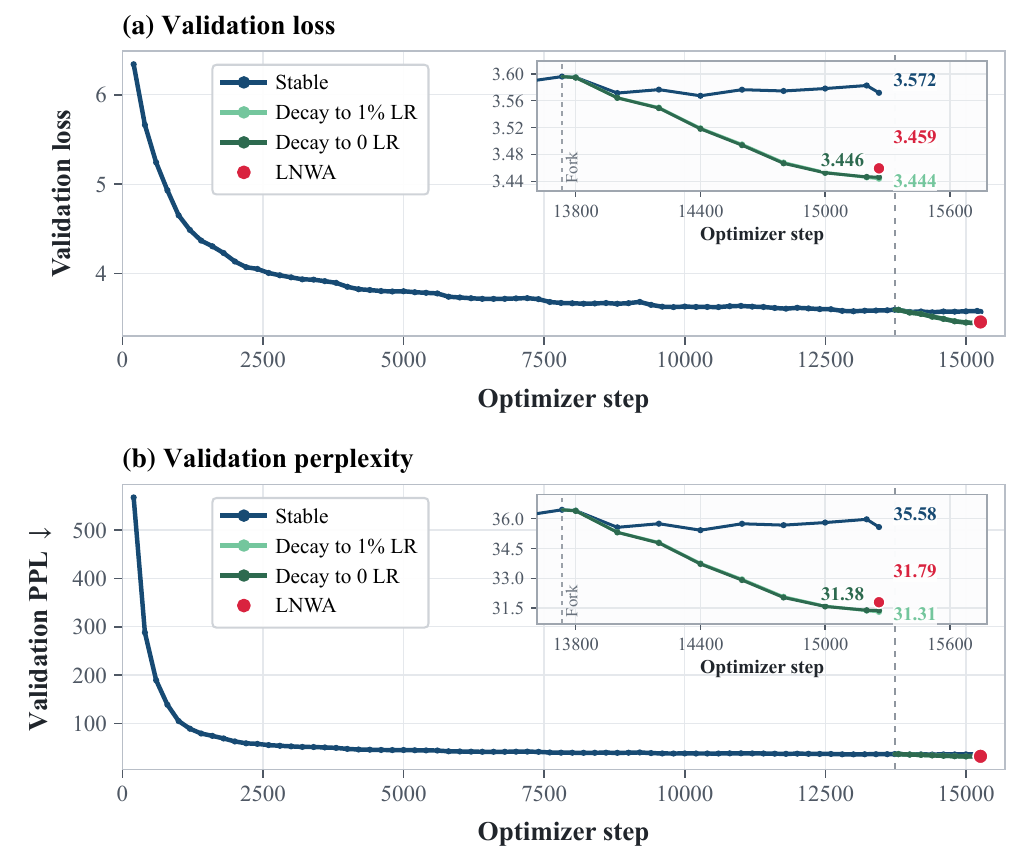}
  \caption{Comparison of validation loss and perplexity under Stable LR, LR decay, and \LNWA{} for OPT-125M.}
  \label{fig:wa-decay-validation}
\end{figure}
\subsection{Experimental Results}
\label{app:complete-results}

\paragraph{From noise tolerance to degradation.}
Figures~\ref{fig:lnwa-turnover-opt-full} and
\ref{fig:lnwa-turnover-opt-awq-full} extend
Figure~\ref{fig:lnwa-turnover-main} to OPT-125M, OPT-350M, and OPT-1.3B,
including full precision (FP) and 4-bit (W4), 3-bit (W3), and 2-bit (W2)
quantization under GPTQ and AWQ, respectively.
FP is evaluated on the original
FineWeb-Edu validation set, whereas W4/W3/W2 are evaluated on the fixed
held-out RefinedWeb validation set.

At small $\lambda$, NLL remains close to or below the \LAWA{} baseline;
at larger $\lambda$, it rises above the baseline. Reducing temporal lag
can initially offset or outweigh the cost of retaining additional
spatial noise. Further movement along the frontier retains more noise
and degrades performance. These results across model scales and
quantization methods further support the tolerance-to-degradation
transition predicted by Theorem~\ref{thm:ptq-turnover}.

\begin{figure}
  \centering
  \includegraphics[width=0.72\textwidth]{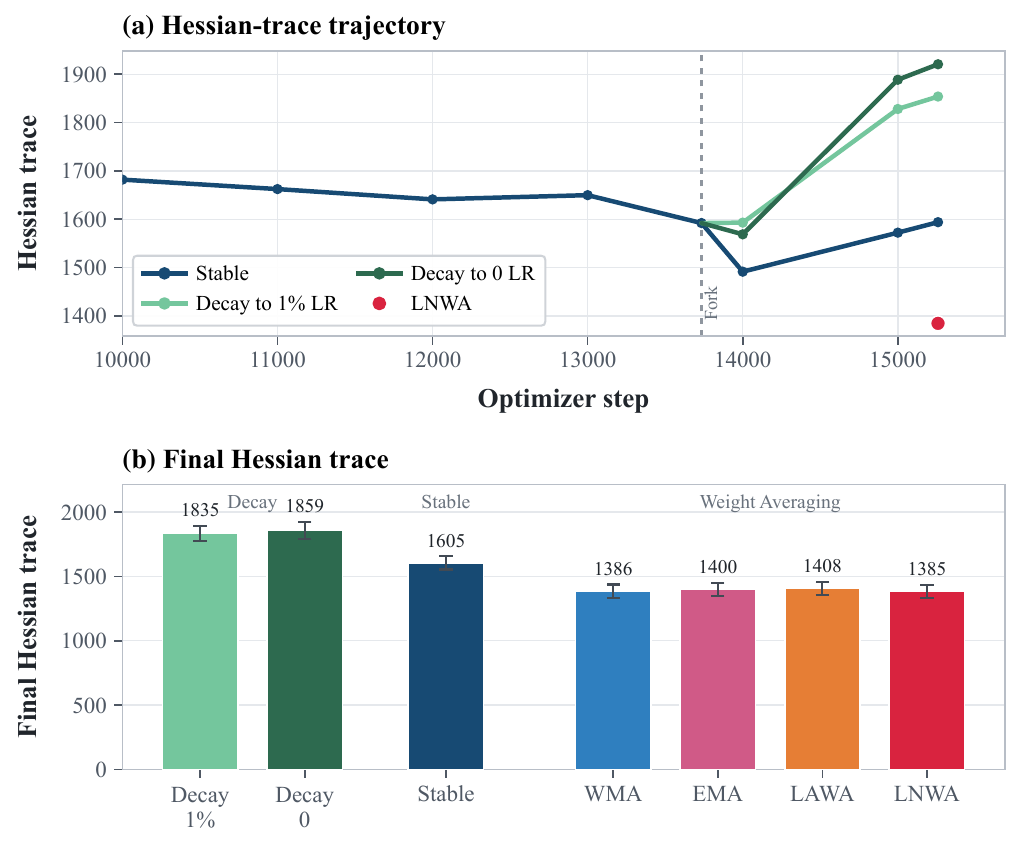}
  \caption{Comparison of Hessian-trace trajectories and final Hessian trace between LR decay, Stable LR, and weight averaging for OPT-125M.}
  \label{fig:wa-decay-hessian}
\end{figure}

\paragraph{Precision changes the preferred averaging kernel.}
Figures~\ref{fig:lambda-regression-opt-full} and
\ref{fig:lambda-regression-opt-awq-full} extend
Figure~\ref{fig:lambda-precision-reversal} to OPT-125M, OPT-350M, and OPT-1.3B,
including FP, W4, W3, and W2, with GPTQ and AWQ, respectively.
At every model scale, PPL is positively correlated with the averaging
coefficient at W2 and negatively correlated at W4 under both methods.
Since smaller coefficients retain less spatial noise and larger
coefficients reduce temporal lag, the reversal indicates a preference
for less noise at lower precision and less lag at higher precision.
This further supports the precision-dependent preference predicted by
Theorem~\ref{thm:bit-ordering}.

\paragraph{Comparison with learning-rate decay.}
We compare continued stable-learning-rate training, cosine decay to zero,
cosine decay to $1\%$ of the maximum learning rate, and \LNWA{} on
OPT-125M. 
Validation loss and perplexity are recorded every 200 optimizer steps,
including the shared starting weight and the final weights.
The \LNWA{} endpoint uses $\lambda=0.008$ and is shown at the
final step.

Figure~\ref{fig:wa-decay-validation} shows that \LNWA{} reaches a
validation PPL of 31.79, compared with 35.58 for continued stable-LR
training and 31.31 and 31.38 for decay to $1\%$ LR and zero,
respectively.

Figure~\ref{fig:wa-decay-hessian} shows that the final
Hessian trace is 1,385 for \LNWA{}, compared with 1,605 for stable LR
and 1,835 and 1,859 for the two decay schedules. Weight averaging
therefore achieves validation performance close to LR decay while
maintaining lower local curvature.

\end{document}

%% file: math_commands.tex
\usepackage{amsmath,amsfonts,bm}

\def\eqref#1{equation~\ref{#1}}

\def\1{\bm{1}}

\DeclareMathAlphabet{\mathsfit}{\encodingdefault}{\sfdefault}{m}{sl}
\SetMathAlphabet{\mathsfit}{bold}{\encodingdefault}{\sfdefault}{bx}{n}

\newcommand{\E}{\mathbb{E}}

\newcommand{\Cov}{\mathrm{Cov}}


%% file: iclr2027_conference.bib
@article{zhang2022opt,
  title={Opt: Open pre-trained transformer language models},
  author={Zhang, Susan and Roller, Stephen and Goyal, Naman and Artetxe, Mikel and Chen, Moya and Chen, Shuohui and Dewan, Christopher and Diab, Mona and Li, Xian and Lin, Xi Victoria and others},
  journal={arXiv preprint arXiv:2205.01068},
  year={2022}
}

@article{workshop2022bloom,
  title={Bloom: A 176b-parameter open-access multilingual language model},
  author={Workshop, BigScience and Scao, Teven Le and Fan, Angela and Akiki, Christopher and Pavlick, Ellie and Ili{\'c}, Suzana and Hesslow, Daniel and Castagn{\'e}, Roman and Luccioni, Alexandra Sasha and Yvon, Fran{\c{c}}ois and others},
  journal={arXiv preprint arXiv:2211.05100},
  year={2022}
}

@article{team2025kimi,
  title={Kimi k2: Open agentic intelligence},
  author={Team, Kimi and Bai, Yifan and Bao, Yiping and Charles, Y and Chen, Cheng and Chen, Guanduo and Chen, Haiting and Chen, Huarong and Chen, Jiahao and Chen, Ningxin and others},
  journal={arXiv preprint arXiv:2507.20534},
  year={2025}
}

@article{team2025gemma,
  title={Gemma 3 technical report},
  author={Team, Gemma and Kamath, Aishwarya and Ferret, Johan and Pathak, Shreya and Vieillard, Nino and Merhej, Ramona and Perrin, Sarah and Matejovicova, Tatiana and Ram{\'e}, Alexandre and Rivi{\`e}re, Morgane and others},
  journal={arXiv preprint arXiv:2503.19786},
  year={2025}
}

@inproceedings{xiao2023smoothquant,
  title={Smoothquant: Accurate and efficient post-training quantization for large language models},
  author={Xiao, Guangxuan and Lin, Ji and Seznec, Mickael and Wu, Hao and Demouth, Julien and Han, Song},
  booktitle={International Conference on Machine Learning},
  pages={38087--38099},
  year={2023},
  organization={PMLR}
}

@article{sun2024flatquant,
  title={Flatquant: Flatness matters for LLM quantization},
  author={Sun, Yuxuan and Liu, Ruikang and Bai, Haoli and Bao, Han and Zhao, Kang and Li, Yuening and Hu, Jiaxin and Yu, Xianzhi and Hou, Lu and Yuan, Chun and others},
  journal={arXiv preprint arXiv:2410.09426},
  year={2024}
}

@inproceedings{liu2025spinquant,
  title={Spinquant: LLM quantization with learned rotations},
  author={Liu, Zechun and Zhao, Changsheng and Fedorov, Igor and Soran, Bilge and Choudhary, Dhruv and Krishnamoorthi, Raghuraman and Chandra, Vikas and Tian, Yuandong and Blankevoort, Tijmen},
  booktitle={International Conference on Learning Representations},
  volume={2025},
  pages={92009--92032},
  year={2025}
}

@inproceedings{shao2024omniquant,
  title={Omniquant: Omnidirectionally calibrated quantization for large language models},
  author={Shao, Wenqi and Chen, Mengzhao and Zhang, Zhaoyang and Xu, Peng and Zhao, Lirui and Li, Zhiqian and Zhang, Kaipeng and Peng, Gao and Qiao, Yu and Luo, Ping},
  booktitle={International Conference on Learning Representations},
  volume={2024},
  pages={45472--45496},
  year={2024}
}

@article{tseng2024quip,
  title={Quip\#: Even better llm quantization with hadamard incoherence and lattice codebooks},
  author={Tseng, Albert and Chee, Jerry and Sun, Qingyao and Kuleshov, Volodymyr and De Sa, Christopher},
  journal={Proceedings of Machine Learning Research},
  volume={235},
  pages={48630},
  year={2024}
}

@article{zhang2024lqer,
  title={Lqer: Low-rank quantization error reconstruction for LLMs},
  author={Zhang, Cheng and Cheng, Jianyi and Constantinides, George A and Zhao, Yiren},
  journal={arXiv preprint arXiv:2402.02446},
  year={2024}
}

@article{shao2026dartquant,
  title={Dartquant: efficient rotational distribution calibration for LLM quantization},
  author={Shao, Yuantian and Chen, Yuanteng and Wang, Peisong and Yu, Jianlin and Lin, Jing and Wei, Zhihui and Cheng, Jian and others},
  journal={Advances in Neural Information Processing Systems},
  volume={38},
  pages={143936--143970},
  year={2026}
}

@article{li2025gptaq,
  title={Gptaq: Efficient finetuning-free quantization for asymmetric calibration},
  author={Li, Yuhang and Yin, Ruokai and Lee, Donghyun and Xiao, Shiting and Panda, Priyadarshini},
  journal={arXiv preprint arXiv:2504.02692},
  year={2025}
}

@article{frantar2022gptq,
  title= {GPTQ: Accurate post-training quantization for generative pre-trained transformers},
  author={Frantar, Elias and Ashkboos, Saleh and Hoefler, Torsten and Alistarh, Dan},
  journal={arXiv preprint arXiv:2210.17323},
  year={2022}
}

@article{lin2024awq,
  title={Awq: Activation-aware weight quantization for on-device LLM compression and acceleration},
  author={Lin, Ji and Tang, Jiaming and Tang, Haotian and Yang, Shang and Chen, Wei-Ming and Wang, Wei-Chen and Xiao, Guangxuan and Dang, Xingyu and Gan, Chuang and Han, Song},
  journal={Proceedings of Machine Learning and Systems},
  volume={6},
  pages={87--100},
  year={2024}
}

@inproceedings{ouyang2025low,
  title={Low-bit quantization favors undertrained LLMs},
  author={Ouyang, Xu and Ge, Tao and Hartvigsen, Thomas and Zhang, Zhisong and Mi, Haitao and Yu, Dong},
  booktitle={Proceedings of the 63rd Annual Meeting of the Association for Computational Linguistics (Volume 1: Long Papers)},
  pages={32338--32348},
  year={2025}
}

@inproceedings{kumar2025scaling,
  title={Scaling laws for precision},
  author={Kumar, Tanishq and Ankner, Zachary and Spector, Benjamin and Bordelon, Blake and Muennighoff, Niklas and Paul, Mansheej and Pehlevan, Cengiz and R{\'e}, Christopher and Raghunathan, Aditi},
  booktitle={International Conference on Learning Representations},
  volume={2025},
  pages={71833--71864},
  year={2025}
}

@inproceedings{catalan2026training,
  title={Training dynamics impact post-training quantization robustness},
  author={Catalan-Tatjer, Albert and Ajroldi, Niccolo and Geiping, Jonas},
  booktitle={International Conference on Learning Representations},
  volume={2026},
  pages={64294--64319},
  year={2026}
}

@article{yano2026pre,
  title={Pre-training LLM without learning rate decay enhances supervised fine-tuning},
  author={Yano, Kazuki and Kiyono, Shun and Kobayashi, Sosuke and Takase, Sho and Suzuki, Jun},
  journal={arXiv preprint arXiv:2603.16127},
  year={2026}
}

@misc{microsoft-ai-team-mai-thinking-1-building-a-hill-climbing-machine-2024,
  doi = {},
  url = {},
  author = {Microsoft AI Team, The},
  title = {MAI-Thinking-1: Building a Hill-Climbing Machine},
  publisher = {},
  year = {2026}
}

@article{watts2026sharpness,
  title={Sharpness-aware pretraining mitigates catastrophic forgetting},
  author={Watts, Ishaan and Li, Catherine and Goyal, Sachin and Springer, Jacob Mitchell and Raghunathan, Aditi},
  journal={arXiv preprint arXiv:2605.02105},
  year={2026}
}

@article{wang2025learning,
  title={Learning dynamics in continual pre-training for large language models},
  author={Wang, Xingjin and Tissue, Howe and Wang, Lu and Li, Linjing and Zeng, Daniel Dajun},
  journal={arXiv preprint arXiv:2505.07796},
  year={2025}
}

@article{li2026model,
  title={Model merging in pre-training of large language models},
  author={Li, Yunshui and Ma, Yiyuan and Yan, Shen and Zhang, Chaoyi and Liu, Jing and Lu, Jianqiao and Xu, Ziwen and Chen, Mengzhao and Wang, Minrui and Zhan, Shiyi and others},
  journal={Advances in Neural Information Processing Systems},
  volume={38},
  pages={133668--133691},
  year={2026}
}

@article{sanyal2023early,
  title={Early weight averaging meets high learning rates for llm pre-training},
  author={Sanyal, Sunny and Neerkaje, Atula and Kaddour, Jean and Kumar, Abhishek and Sanghavi, Sujay},
  journal={arXiv preprint arXiv:2306.03241},
  year={2023}
}

@inproceedings{tian2026wsm,
  title={WSM: decay-free learning rate schedule via checkpoint merging for LLM pre-training},
  author={Tian, Changxin and Wang, Jiapeng and Zhao, Qian and Chen, Kunlong and Liu, Jia and Liu, Ziqi and Mao, Jiaxin and Zhao, Xin and Zhang, Zhiqiang and Zhou, Jun},
  booktitle={International Conference on Learning Representations},
  volume={2026},
  pages={134767--134795},
  year={2026}
}

@article{morales2024exponential,
  title={Exponential moving average of weights in deep learning: Dynamics and benefits},
  author={Morales-Brotons, Daniel and Vogels, Thijs and Hendrikx, Hadrien},
  journal={arXiv preprint arXiv:2411.18704},
  year={2024}
}

@article{ajroldi2025and,
  title={When, where and why to average weights?},
  author={Ajroldi, Niccol{\`o} and Orvieto, Antonio and Geiping, Jonas},
  journal={arXiv preprint arXiv:2502.06761},
  year={2025}
}

@article{sandler2023training,
  title={Training trajectories, mini-batch losses and the curious role of the learning rate},
  author={Sandler, Mark and Zhmoginov, Andrey and Vladymyrov, Max and Miller, Nolan},
  journal={arXiv preprint arXiv:2301.02312},
  year={2023}
}

@article{block2025ema,
  title={EMA Without the Lag: Bias-Corrected Iterate Averaging Schemes},
  author={Block, Adam and Zhang, Cyril},
  journal={arXiv preprint arXiv:2508.00180},
  year={2025}
}

@inproceedings{luo2026learning,
  title={How learning rate decay wastes your best data in curriculum-based llm pretraining},
  author={Luo, Kairong and Sun, Zhenbo and Wen, Haodong and Shi, Xinyu and Cui, Jiarui and Dang, Chenyi and Lyu, Kaifeng and Chen, Wenguang},
  booktitle={International Conference on Learning Representations},
  volume={2026},
  pages={147266--147297},
  year={2026}
}

@inproceedings{chen2026coverage,
  title={The coverage principle: How pre-training enables post-training},
  author={Chen, Fan and Huang, Audrey and Golowich, Noah and Malladi, Sadhika and Block, Adam and Ash, Jordan and Krishnamurthy, Akshay and Foster, Dylan},
  booktitle={International Conference on Learning Representations},
  volume={2026},
  pages={112225--112293},
  year={2026}
}

@article{hu2024minicpm,
  title={Minicpm: Unveiling the potential of small language models with scalable training strategies},
  author={Hu, Shengding and Tu, Yuge and Han, Xu and He, Chaoqun and Cui, Ganqu and Long, Xiang and Zheng, Zhi and Fang, Yewei and Huang, Yuxiang and Zhao, Weilin and others},
  journal={arXiv preprint arXiv:2404.06395},
  year={2024}
}

@article{hagele2024scaling,
  title={Scaling laws and compute-optimal training beyond fixed training durations},
  author={H{\"a}gele, Alexander and Bakouch, Elie and Kosson, Atli and Allal, Loubna B and Von Werra, Leandro and Jaggi, Martin},
  journal={Advances in Neural Information Processing Systems},
  volume={37},
  pages={76232--76264},
  year={2024}
}

@inproceedings{wen2025understanding,
  title={Understanding warmup-stable-decay learning rates: A river valley loss landscape view},
  author={Wen, Kaiyue and Li, Zhiyuan and Wang, Jason and Hall, David and Liang, Percy and Ma, Tengyu},
  booktitle={International Conference on Learning Representations},
  volume={2025},
  pages={42840--42885},
  year={2025}
}

@article{li2026functional,
  title={Functional scaling laws in kernel regression: Loss dynamics and learning rate schedules},
  author={Li, Binghui and Chen, Fengling and Huang, Zixun and Wang, Lean and Wu, Lei},
  journal={Advances in Neural Information Processing Systems},
  volume={38},
  pages={101211--101269},
  year={2026}
}

@article{li2026optimal,
  title={Optimal learning-rate schedules under functional scaling laws: Power decay and warmup-stable-decay},
  author={Li, Binghui and Wang, Zilin and Chen, Fengling and Zhao, Shiyang and Zheng, Ruiheng and Wu, Lei},
  journal={arXiv preprint arXiv:2602.06797},
  year={2026}
}

@article{liu2026effective,
  title={Effective Learning Rate Governs Loss Dynamics in Language Model Pretraining},
  author={Liu, Zihan and Zheng, Ruiheng and Zhang, Shaobo and Tian, Changxin and Chen, Kunlong and Zhang, Zhiqiang and Wu, Lei},
  journal={arXiv preprint arXiv:2608.24814},
  year={2026}
}

@article{alvarez2025introducing,
  title={Introducing NVFP4 for efficient and accurate low-precision inference},
  author={Alvarez, Eduardo and Almog, Omri and Chung, Eric and Layton, Simon and Stosic, Dusan and Krashinsky, Ronny and Aubrey, Kyle},
  journal={NVIDIA Technical Blog},
  volume={24},
  year={2025}
}

@inproceedings{yao2024exploring,
  title={Exploring post-training quantization in llms from comprehensive study to low rank compensation},
  author={Yao, Zhewei and Wu, Xiaoxia and Li, Cheng and Youn, Stephen and He, Yuxiong},
  booktitle={Proceedings of the AAAI Conference on Artificial Intelligence},
  volume={38},
  number={17},
  pages={19377--19385},
  year={2024}
}

@inproceedings{zhang2025qera,
  title={Qera: an analytical framework for quantization error reconstruction},
  author={Zhang, Cheng and Wong, Jeffrey TH and Xiao, Can and Constantinides, George and Zhao, Yiren},
  booktitle={International Conference on Learning Representations},
  volume={2025},
  pages={12531--12560},
  year={2025}
}

@article{izmailov2018averaging,
  title={Averaging weights leads to wider optima and better generalization},
  author={Izmailov, Pavel and Podoprikhin, Dmitrii and Garipov, Timur and Vetrov, Dmitry and Wilson, Andrew Gordon},
  journal={arXiv preprint arXiv:1803.05407},
  year={2018}
}

@article{polyak1992acceleration,
  title={Acceleration of stochastic approximation by averaging},
  author={Polyak, Boris T and Juditsky, Anatoli B},
  journal={SIAM Journal on Control and Optimization},
  volume={30},
  number={4},
  pages={838--855},
  year={1992},
  publisher={SIAM}
}

@article{penedo2024fineweb,
  title={The fineweb datasets: Decanting the web for the finest text data at scale},
  author={Penedo, Guilherme and Kydl{\'\i}{\v{c}}ek, Hynek and Lozhkov, Anton and Mitchell, Margaret and Raffel, Colin and Von Werra, Leandro and Wolf, Thomas and others},
  journal={Advances in Neural Information Processing Systems},
  volume={37},
  pages={30811--30849},
  year={2024}
}

@article{penedo2023refinedweb,
  title={The refinedweb dataset for falcon llm: Outperforming curated corpora with web data only},
  author={Penedo, Guilherme and Malartic, Quentin and Hesslow, Daniel and Cojocaru, Ruxandra and Alobeidli, Hamza and Cappelli, Alessandro and Pannier, Baptiste and Almazrouei, Ebtesam and Launay, Julien},
  journal={Advances in Neural Information Processing Systems},
  volume={36},
  pages={79155--79172},
  year={2023}
}

@article{loshchilov2017decoupled,
  title={Decoupled weight decay regularization},
  author={Loshchilov, Ilya and Hutter, Frank},
  journal={arXiv preprint arXiv:1711.05101},
  year={2017}
}

@article{raffel2020exploring,
  title={Exploring the limits of transfer learning with a unified text-to-text transformer},
  author={Raffel, Colin and Shazeer, Noam and Roberts, Adam and Lee, Katherine and Narang, Sharan and Matena, Michael and Zhou, Yanqi and Li, Wei and Liu, Peter J},
  journal={Journal of machine learning research},
  volume={21},
  number={140},
  pages={1--67},
  year={2020}
}

@article{ahmadian2023intriguing,
  title={Intriguing properties of quantization at scale},
  author={Ahmadian, Arash and Dash, Saurabh and Chen, Hongyu and Venkitesh, Bharat and Gou, Zhen Stephen and Blunsom, Phil and {\"U}st{\"u}n, Ahmet and Hooker, Sara},
  journal={Advances in Neural Information Processing Systems},
  volume={36},
  pages={34278--34294},
  year={2023}
}

@ARTICLE{823976,
  author={Wannamaker, R.A. and Lipshitz, S.P. and Vanderkooy, J. and Wright, J.N.},
  journal={IEEE Transactions on Signal Processing}, 
  title={A theory of nonsubtractive dither}, 
  year={2000},
  volume={48},
  number={2},
  pages={499-516},
  doi={10.1109/78.823976}}

@inproceedings{nagel2020up,
  title={Up or down? adaptive rounding for post-training quantization},
  author={Nagel, Markus and Amjad, Rana Ali and Van Baalen, Mart and Louizos, Christos and Blankevoort, Tijmen},
  booktitle={International conference on machine learning},
  pages={7197--7206},
  year={2020},
  organization={PMLR}
}

@article{chee2023quip,
  title={Quip: 2-bit quantization of large language models with guarantees},
  author={Chee, Jerry and Cai, Yaohui and Kuleshov, Volodymyr and De Sa, Christopher M},
  journal={Advances in neural information processing systems},
  volume={36},
  pages={4396--4429},
  year={2023}
}

@article{zhang2026spfq,
  title={Spfq: A stochastic algorithm and its error analysis for neural network quantization},
  author={Zhang, Jinjie and Saab, Rayan},
  journal={Information and Inference: A Journal of the IMA},
  volume={15},
  number={2},
  pages={iaag011},
  year={2026},
  publisher={Oxford University Press}
}

@article{ashkboos2024quarot,
  title={Quarot: Outlier-free 4-bit inference in rotated llms},
  author={Ashkboos, Saleh and Mohtashami, Amirkeivan and Croci, Maximilian L and Li, Bo and Cameron, Pashmina and Jaggi, Martin and Alistarh, Dan and Hoefler, Torsten and Hensman, James},
  journal={Advances in Neural Information Processing Systems},
  volume={37},
  pages={100213--100240},
  year={2024}
}

@article{dettmers2022gpt3,
  title={Gpt3. int8 (): 8-bit matrix multiplication for transformers at scale},
  author={Dettmers, Tim and Lewis, Mike and Belkada, Younes and Zettlemoyer, Luke},
  journal={Advances in neural information processing systems},
  volume={35},
  pages={30318--30332},
  year={2022}
}

@article{egiazarian2024extreme,
  title={Extreme compression of large language models via additive quantization},
  author={Egiazarian, Vage and Panferov, Andrei and Kuznedelev, Denis and Frantar, Elias and Babenko, Artem and Alistarh, Dan},
  journal={arXiv preprint arXiv:2401.06118},
  year={2024}
}

@article{chen2023quantized,
  title={Quantized low-rank multivariate regression with random dithering},
  author={Chen, Junren and Wang, Yueqi and Ng, Michael K},
  journal={IEEE Transactions on Signal Processing},
  volume={71},
  pages={3913--3928},
  year={2023},
  publisher={IEEE}
}

@article{dirksen2025subspace,
  title={Subspace and DOA estimation under coarse quantization},
  author={Dirksen, Sjoerd and Li, Weilin and Maly, Johannes},
  journal={IEEE Transactions on Information Theory},
  year={2025},
  publisher={IEEE}
}

@article{banner2018aciq,
  title={Aciq: Analytical clipping for integer quantization of neural networks},
  author={Banner, Ron and Nahshan, Yury and Hoffer, Elad and Soudry, Daniel},
  year={2018}
}
